\documentclass{article}

\usepackage{microtype}
\usepackage{graphicx}
\usepackage{subcaption}
\usepackage{booktabs}

\usepackage{hyperref}

\usepackage[accepted]{icml2026}

\usepackage{amsmath}
\usepackage{amssymb}
\usepackage{mathtools}
\usepackage{amsthm}
\usepackage{mathrsfs}
\usepackage{tabularx}
\usepackage{enumitem}

\usepackage[capitalize,noabbrev]{cleveref}

\newcommand{\R}[0]{\mathbb{R}}
\newcommand{\Z}[0]{\mathbb{Z}}

\newcommand{\J}[0]{\mathcal{J}}

\newcommand{\inv}[1]{#1^{-1}}

\DeclareMathOperator{\tr}{tr}

\theoremstyle{plain}
\newtheorem{theorem}{Theorem}[section]
\newtheorem{proposition}[theorem]{Proposition}

\theoremstyle{definition}
\newtheorem{definition}[theorem]{Definition}

\theoremstyle{remark}
\newtheorem{remark}[theorem]{Remark}
\newtheorem{example}[theorem]{Example}

\usepackage[textsize=tiny]{todonotes}

\icmltitlerunning{Multiplication Beyond Groups: Stratified Fourier Mechanisms in Transformer Circuits}

\begin{document}

\twocolumn[
  \icmltitle{Multiplication Beyond Groups: \\ Stratified Fourier Mechanisms in Transformer Circuits}

  \icmlsetsymbol{equal}{*}

  \begin{icmlauthorlist}
    \icmlauthor{Zitong Andrew Chen}{univ,equal}
    \icmlauthor{Junaid Hasan}{univ,equal}
    \icmlauthor{Akhil Srinivasan}{univ,equal}
    \icmlauthor{Hemkesh Bandi}{univ}
    \icmlauthor{Jarod Alper}{univ}
  \end{icmlauthorlist}

  \icmlaffiliation{univ}{UW Math AI Lab, University of Washington}

  \icmlcorrespondingauthor{Zitong A. Chen}{zchen66@uw.edu}

  \icmlkeywords{Machine Learning, ICML}

  \vskip 0.3in
]

\printAffiliationsAndNotice{\icmlEqualContribution}

\begin{abstract}
  Transformers have demonstrated a remarkable ability to learn algorithmic reasoning, yet mechanistic analyses have mostly focused on globally invertible operations such as cyclic addition and group composition. In this work, we investigate how small transformers learn modular integer multiplication over composite moduli, a fundamentally non-invertible operation due to the presence of zero-divisors. We propose the \emph{monoid extension}: a localized generalization of Group Composition via Representation (GCR) that suggests the learned computation does not rely on a single global representation space. Instead, the model partitions the input space into local hierarchical algebraic regions, where group-like structure survives and Fourier mechanisms can be applied. In transformers trained on square-free modular multiplication, we find that embeddings organize around these regions, attention exhibits class-sensitive routing and low-rank write directions, and local character features explain a large fraction of the model's output logits. Our results suggest that representation-theoretic mechanisms previously identified for group operations can extend beyond groups to more general structures.
\end{abstract}

\section{Introduction}
Neural networks can learn to perform structured mathematical tasks, often generalizing far beyond the examples seen during training. A central goal of \emph{mechanistic interpretability} is to reverse-engineer these learned algorithms to be able to deeply understand the internal computations that produce the correct answer.

Early work on modular addition and related arithmetic tasks showed that small transformers often learn rigid mathematical structure rather than relying on rote memorization \citep{power2022grokking,nanda2023progress,gromov2023grokking,zhong2023clock}. These analytical works often involved concepts from the mathematical fields of group theory, Fourier analysis, and representation and character theory, which bridge the gap between the theoretical abstract operations and the physical linear algebra as computed by the networks.

In particular, prior work on learning group operations has found that models represent the group elements using Fourier features and compute group operations through representation-theoretic mechanisms. These results are generalized by Group Composition via Representation (GCR) as introduced by \citet{chughtai2023toy}, which explains how models can compute group operations by embedding elements into representation space, compose the representations internally, and use the unembedding to score candidates via character-like identity tests.

However, this line of work largely assumes that the underlying algebraic structure is a group, which ensures global invertibility for all elements. An extension to arbitrary associative multiplication tables breaks this assumption, where the underlying structures now involve non-invertible elements. This raises a natural question: what happens to GCR-style mechanisms when global invertibility is no longer available?

This is exactly the question posed by \citet{chughtai2023toy}. In this work, we address this question by exploring the non-invertible algebraic setting through analyzing how models learn modular multiplication over composite moduli. We show that, in a transformer trained on modular multiplication over $\Z_n$ for select values of $n$, \textbf{GCR-style computation appears to localize to disjoint algebraic strata} that partitions the input space, where \emph{local} Fourier characters and \emph{local} inverses explain a substantial fraction of the model’s learned representations and logits. Our key contributions are: \begin{itemize}[topsep=0pt, itemsep=2pt]
    \item We leverage local representation theory to formulate square-free modular multiplication as a non-group extension point for GCR-style mechanistic explanations.
    
    \item We provide circuit-level evidence for routing with respect to algebraic structure, where attention and OV circuits move information into low-dimensional local subspaces.

    \item We show that local Fourier characters and local inverses explain a large fraction of the model's output logits.
\end{itemize}
The remainder of this paper is structured as follows: Section \ref{sec:related-work} outlines previous adjacent literature. Section \ref{sec:setup} explains the problem set-up, and formalizes locality by defining relevant algebraic concepts. Section \ref{sec:hypothesis} outlines our exact hypothesis and how we generalized the GCR algorithm to the non-invertible case. Section \ref{sec:interp} dissects and interprets the entire transformer model layer-by-layer, matching relevant architecture to corresponding hypotheses.

\section{Related Work}
\label{sec:related-work}
\textbf{Mechanistic Interpretability of Algorithmic Tasks.} Over the past several years, advances in mechanistic interpretability have successfully reverse-engineered how transformers perform simple arithmetic tasks \citep{power2022grokking,nanda2023progress,gromov2023grokking,zhong2023clock}. Early investigations of modular multiplication by \citet{gromov2023grokking} empirically showcased periodic structures via preactivation visualizations. Works have leveraged empirical techniques such as Principal Component Analysis \citep{jolliffe2002principal}, Discrete Fourier Transform, Canonical Correlation Analysis \citep{raghu2017svcca,morcos2018insights}, and a variety of ways to assess Fraction of Variance Explained in different metrics, to discover that models are able to learn rigid algorithms rather than rote memorization of data.

\textbf{Group-Theoretic Explanations of Networks.} Extending beyond simple arithmetic tasks, recent works have proposed generalized algorithms for abstract operations learned by networks \citep{chughtai2023toy,stander2024cosets}, oftentimes used as toy model explorations of broader hypotheses in mechanistic interpretability. However, these mechanistic analyses have predominantly focused on classical group theory where the underlying operations admit clean, invertible geometric structure. This leaves open how mechanistic explanations should extend to more complex, non-invertible algebraic structures, where operations can collapse information. Our work addresses this gap by studying modular multiplication on composite moduli over complete $n \times n$ multiplication tables, thereby introducing non-invertible elements.

\textbf{Transformer Circuitry and Routing.} A parallel line of work studies networks by decomposing them into smaller individual circuits for interpretability \citep{olah2020zoom,cammarata2020circuits,elhage2021framework}. These mechanisms aim to distinguish where information is learned, and track how that information moves throughout the network. In many algorithmic settings, routing is relatively simple: the network mainly needs to move information between positions while preserving a single global representation space. However, by introducing non-invertibility, we may require networks to form multiple algebraic strata which lead to more subtle and nuanced routing mechanisms.
\section{Setup and Background}
\label{sec:setup}
\subsection{Main Task}
\label{sec:task}
We train our model (see Section \ref{sec:model}) on a simple task: integer modular multiplication. Given two integers $a, b \in \Z_n$, we train a 1-layer transformer to predict the element $c = a \cdot b \pmod n$.

We note that previous work \citep{chughtai2023toy,stander2024cosets} has explored toy models' relative competence in learning basic finite group operations. Given this, we aim to explain the model's complete mastery over the \emph{complete} $n \times n$ multiplication table, thereby introducing non-invertible information-collapsing elements within our inputs.

To this end, we have explored the following list of moduli shown in Table \ref{tab:moduli}. For the remainder of this paper, we focus our discussion on $n = 165$. The plots and analyses for all aforementioned moduli can be found in the Appendix.

\begin{table}[t]
    \centering
    \footnotesize
    \setlength{\tabcolsep}{4pt}
    \caption{\textbf{List of moduli we have explored and their corresponding properties.} $\omega(n)$ counts the number of unique prime factors of $n$. ``Square-free'' refers to a positive integer that is not divisible by any perfect square other than 1; such integers comprise approximately $60.8\%$ of all positive integers (Remark~\ref{rem:square-free}).}
    \begin{tabular}{cccc}
        \toprule
        \textbf{Moduli ($n$)} & \textbf{Prime Decomp} & $\omega(n)$ & \textbf{Square-free} \\
        \midrule
        $113$ & $113$ & $1$ & True \\
        $143$ & $11 \times 13$ & $2$ & True \\
        $154$ & $2 \times 7 \times 11$ & $3$ & True \\
        $\mathbf{165^*}$ & $\mathbf{3 \times 5 \times 11}$ & $\mathbf{3}$ & \textbf{True}\\
        \bottomrule
    \end{tabular}
    \label{tab:moduli}
\end{table}
\subsection{Model}
\label{sec:model}

We study a single-layer decoder-only transformer \citep{vaswani2017attention} trained to perform modular multiplication over $\mathbb{Z}_n$, to predict $c = a \cdot b \mod n$.

Each example is a three-token sequence $(a, b, =)$, where the final token is a special token whose residual stream is used for prediction. Tokens are first one-hot encoded and embedded into dimension $d_{\text{model}} = 128$, which forms the initial residual stream. We then apply multi-head causal self attention to the embedded tokens and add them to the residual stream.

The residual stream is then updated via a feedforward MLP network, which projects up to $d_\text{hidden} = 512$, and applies a ReLU nonlinearity, followed by a down projection back to $d_{\text{model}}$. Finally, we have an unembedding layer $W_U$ that produces logits over $\R^n$. The model is trained with cross-entropy loss using the AdamW optimizer \citep{loshchilov2019decoupled} for 25,000 epochs. Unless otherwise specified, all results are reported on seed 1. Refer to Appendix \ref{app:model} for a detailed mathematical overview of the architecture.

\subsection{Mathematics Background}
\label{sec:monoids}

As we are analyzing the transformer's mastery over the entire $n \times n$ multiplication table, we will naturally come across non-invertible elements. As such, our analysis is done over a broader algebraic structure: monoids. We briefly introduce key definitions and results in this section, with relevant motivations also provided in Appendix \cref{app:monoids-primer}. Further details and proofs may be found in \citet{howie1995semigroup,steinberg2016representation}.

\begin{definition}
A \textbf{monoid} is a set equipped with an associative binary operation, denoted multiplicatively, together with an identity element 1. Invertible elements within the set are called \textbf{units}. Conversely, in integer monoids under multiplication, the non-units are \textbf{zero-divisors}. Multiplying by zero-divisors can collapse distinct residues to the same value.
\end{definition}
\textbf{Example. }The set of integers modulo $n$ forms a finite commutative monoid under multiplication, and we denote it as such with $(\Z_n, \cdot)$, or sometimes simply $\Z_n$. We refer to $n$ as the modulus of the multiplication operation. Crucially, this is not to be confused with the multiplicative \emph{group} of integers modulo $n$, denoted $(\Z / n\Z)^\times$. 

Take the modulus $n = 6$ for example:
\begin{align*}
    &(\Z / 6\Z)^\times = \{1, 5\}\\
    &(\Z_6, \cdot) = \{0, 1, 2, 3, 4, 5\}
\end{align*}
For readers less familiar with the mathematical background, we highly recommend skimming Appendix \ref{app:monoids-primer} (specifically examples \ref{ex:group}, \ref{ex:monoid}).

\begin{definition}
In the commutative monoid $\Z_n$, we define \textbf{$\J$-classes} by grouping elements with the same greatest common divisor with $n$: \[
  J_d = \{a \in \Z_n : \gcd(a, n) = d\}
\]
These classes form a disjoint partition of $\mathbb Z_n$.
\end{definition}
\begin{remark}
    For square-free \(n\), every $\mathcal J$-class is \textbf{regular}. The precise definition of a regular $\mathcal J$-class is provided in the Appendix, Definition \ref{def:regular-j}.
\end{remark}

\begin{theorem}
\label{thm:cong}
In the multiplication monoid $(\mathbb Z_n, \cdot)$, every regular $\mathcal J$-class $J_d$ forms a group with local identity $e_{J_d}$ called the \emph{idempotent}, and
\[
    J_d \cong \left(\mathbb Z\Big/\left(\frac{n}{d}\right)\mathbb Z\right)^\times.
\]
For square-free moduli such as $n = 165$, \textbf{all $\mathcal J$-classes are regular}. We give the full statement and proof in Appendix \cref{thm:jclass-group}.
\end{theorem}
\begin{remark} \label{rem:j_class_z_n}
The $\J$-class $J_1$ of any monoid $\Z_n$ is exactly the group of units $(\Z / n\Z)^\times$.
\end{remark}
\textbf{Example. }The monoid $(\Z_6, \cdot)$ can be partitioned into \[
    \Z_{6} = J_1 \sqcup J_2 \sqcup J_3 \sqcup J_6 = \{1, 5\} \sqcup \{2, 4\} \sqcup \{3\} \sqcup \{0\}
\] 
where, crucially, \emph{every $\J$-class is regular}. Notice $J_1 = (\Z / 6\Z)^\times$ as mentioned in Remark \ref{rem:j_class_z_n}.

See Appendix \ref{app:monoids-primer} for a breakdown of the $\J$-classes and related isomorphisms of $(\Z_{165}, \cdot)$, which we will be working over for our analysis later on.
\begin{definition}
For a finite group \(G\), a real \textbf{representation} is a homomorphism $\rho: G \to GL(\mathbb R^d)$. The \textbf{character} of \(\rho\) is the trace function $\chi_\rho(g)=\mathrm{tr}(\rho(g))$.

In the orthogonal/Fourier representations used by GCR-style mechanisms, the character acts as an identity detector: after normalization, \(\chi_\rho(g)\) is maximized when \(\rho(g)\) is the identity matrix. For faithful representations, this occurs precisely when \(g\) is the identity element of \(G\).
\end{definition}

\section{An Extension to the GCR Algorithm}
\label{sec:hypothesis}
\subsection{Review: The Group Composition (GCR) Algorithm}
\label{sec:review}
To formalize the network's behavior, we build upon the Group Composition via Representation (GCR) algorithm introduced by \citet{chughtai2023toy}. This section serves as a brief review, with more details found in the original work. For an arbitrary finite \emph{group} $G$ with a representation $\rho$, GCR posits that small networks learn group operations via the following procedure:

 (1) The embedding layer translates $a, b \overset{W_E}{\longmapsto} \rho(a), \rho(b)$.\\
 (2) The MLP performs group operations through matrix multiplication on the representations\[
 \rho(a), \rho(b) \overset{MLP}{\longmapsto} \rho(a)\rho(b) = \rho(ab)
 \]
 (3) The unembedding layer computes output logit $c$ by taking trace: $\text{Logit}(c) \propto \tr(\rho(ab)\rho(c^{-1})) = \chi_\rho(ab\inv c)$.

Crucially, GCR formalizes the \textbf{key frequencies} first observed by \citet{nanda2023progress} in modular addition. For cyclic groups, each Fourier frequency \(k\) gives a 2D irreducible representation, which appears architecturally as (superpositioned) pairs of entries in the residual-stream, storing
\[
    \cos\left(\frac{2\pi ka}{n}\right) \quad \text{and} \quad \sin\left(\frac{2\pi ka}{n}\right)
\]
for each token \(a\). Empirically, models use only a sparse set of such frequency planes, and compute the group operation by composing these planes through the network.

\subsection{The Failure of Global Invertibility}

The GCR mechanism depends crucially on invertibility. In the group setting, every candidate output $c \in G$ has an inverse $c^{-1}$, allowing the unembedding to assign a logit score for $c$ by testing whether $ab c^{-1} = 1$. 

However, this decoding rule no longer applies to the full multiplication monoid $\mathbb Z_n$. When $n$ is composite, many elements are non-invertible. For instance, $5 \in \mathbb Z_{165}$ has no multiplicative inverse, so there is no element $5^{-1}$ for the unembedding to represent. Thus, while GCR can directly explain computation inside the unit group $J_1 = (\mathbb Z/n\mathbb Z)^\times$, it does not by itself explain how the model scores outputs lying in the zero-divisor classes, implying the model cannot simply run the same fixed-dimensional inverse-lookup algorithm everywhere in $\mathbb Z_n$.

\textbf{Example.}
In \(\mathbb Z_{15}\), multiplication by the zero-divisor \(3\) collapses distinct units: $4\cdot 3 \equiv 14\cdot 3 \equiv 12 \pmod{15}$. Since \(12\) has no multiplicative inverse, the GCR decoding rule has no global inverse representation $\rho^{-1}(12)$ to use, thus breaking the standard GCR assumption.

\subsection{The Monoid Extension}
\label{sec:monoid-ext}
Despite the absence of global invertibility in finite monoids, we claim that the core mechanisms of GCR can generalize beyond the group setting. We formalize this through the \textbf{Monoid Extension}.

\textbf{Monoid partitioning with $\J$-classes.}
Because each element has a unique gcd with \(n\), this stratification partitions the monoid into disjoint \(\mathcal J\)-classes. This gives a natural target for a routing mechanism: instead of representing $\mathbb Z_n$ as one undifferentiated multiplication table, the model can first identify the $\J$-class of the product and then perform class-specific computation within the corresponding subspace.

As shown in Section~\ref{sec:monoids} and Appendix~\ref{app:monoids-primer}, $\mathbb Z_{165}$ partitions into eight disjoint $\mathcal J$-classes. Standard GCR explains computation within the unit group $J_1=(\mathbb Z/165\mathbb Z)^\times$, but not compositions involving zero-divisors. The Monoid Extension generalizes this picture by modeling computation both within regular $\mathcal J$-classes and across classes, where multiplication by a zero-divisor collapses the output into a non-invertible class.

\textbf{While global invertibility fails, local invertibility survives.} The key insight to resolving zero-divisor compositions lies in \emph{local invertibility},  a structural property guaranteed in the local groups associated with regular \(\mathcal J\)-classes. In our modular integer setting, this gives the group structure described by Theorem \ref{thm:cong}. There are two immediate consequences:
\begin{itemize}[itemsep=0pt, topsep=-2.5pt]
    \item The idempotent $e \in J_d$ maps to the identity $1 \in \left(\mathbb Z/\left(\frac{n}{d}\right)\mathbb Z\right)^\times$. We refer to $e$ as the ``local identity'' of $J_d$ and denote it $e_{J_d}$.
    \item Every element $c \in J_d$ has a ``local inverse'' $c^\sharp \in J_d$, where $c c^\sharp = e_{J_d} \pmod n$.
\end{itemize}
Given this result, we hypothesize that for each candidate $c \in J_d$, the unembedding layer $W_U$ learns (in weights) the representation $\rho_d(c^\sharp)$ of the local inverse $c^\sharp \in J_d$. Here, $\rho_d$ denotes the representation with respect to the group $J_d$. 

\textbf{Routing before scoring.}
The local inverse test only makes sense after the model has identified the $\mathcal J$-class of the product. We therefore hypothesize that the computation decomposes into two phases. First, the model routes the represented product $ab$ into the subspace associated with its class $J_d$. Second, conditioned on this class, the unembedding scores candidates $c \in J_d$ by testing whether $ab c^\sharp = e_{J_d}$.

Assuming the network first executes this routing phase, the monoid extension is a \emph{local generalization} of the procedure outlined in standard GCR: $e_{J_1}$ is exactly $1 \in (\Z / n\Z)^\times$, and each $c^\sharp$ corresponds to $\inv c$ whenever $c \in J_1$. 

We claim then the unembedding layer's logit computation can be universally rephrased as:
\begin{equation}
\label{eq:logit-formula}
    \text{Logit}(c) \propto \tr(\rho_d(ab)\rho_d(c^\sharp)) = \chi_{\rho_d}(abc^\sharp)
\end{equation}
which is exactly maximized for the correct token $c$ where $abc^\sharp = e$ for the corresponding local identity $e \in \Z_n$ (Proposition~\ref{prop:logit-max}).

\textbf{Example.} Recall that \(4\) is a unit and \(3\) is a zero-divisor in \(\mathbb Z_{15}\). Their product lands in the zero-divisor class $4 \cdot 3 \equiv 12 \in J_3$. By Theorem~\ref{thm:cong}, this class carries a local group structure (since $J_3 \cong (\Z / 5 \Z)^\times$) with generator \(3\) and local identity \(e_{J_3}=6\).

To compute the logit for the correct output candidate $c = 12$, the unembedding matrix $W_U$ retrieves its local inverse $12^\sharp = 3$, since $12 \cdot 3 \equiv 6 \pmod{15}$. The network evaluates the character:
\[
\text{Logit}(12) \propto \tr(\rho_3(4 \cdot 3)\rho_3(3)) = \chi_{\rho_3}(4 \cdot 3 \cdot 3) = \chi_{\rho_3}(6)
\]
Because the product simplifies exactly to the local identity, $\chi_{\rho_3}(abc^\sharp)$ is maximized. This gives the correct candidate $c=12$ the maximal local identity score, demonstrating how the Monoid Extension mathematically resolves a composition that would otherwise trigger structural collapse under the standard GCR assumption.

\textbf{Takeaway.} The network handles zero-divisor multiplication by projecting the product into the appropriate subspace, where it calculates and applies a local inverse which maximizes the logits for the correct output.

\section{Interpreting Monoid Compositions}
\label{sec:interp}
We adopt a reverse-engineering methodology similar to those outlined by
\citet{chughtai2023toy} and \citet{nanda2023progress}. Specifically, we employ several classic mechanistic interpretability techniques to provide evidence for a local extension of GCR-style mechanisms to a non-group setting, as described in Section \ref{sec:monoid-ext}. As previously mentioned, our analytical work will be done on the transformer architecture (see Section \ref{sec:model}) trained on the finite commutative monoid $(\Z_{165}, \cdot)$. 

There are three primary hypotheses as presented in the monoid extension, and we seek to address each hypothesis with an analysis of the corresponding layer of the network. First, the monoid extension predicts that, prior to mapping inputs to their representations as described in standard GCR, the \textbf{embedding layer} first partitions the set of all inputs into disjoint $\J$-classes, where elements of the same $\J$-class would occupy the same corresponding subspace of the total 128-dimensional space of the model. 

The most pivotal step of the monoid extension is the routing mechanism, where inputs are projected to the appropriate subspace corresponding to the $\J$-class of the correct output. We explore the broader subspace projection capabilities of the network's \textbf{multi-head attention} while ensuring the nuanced representations of the inputs are maintained, later to be used for logit computations. 

Finally, the monoid extension postulates that \textbf{logits} are computed using the local inverses existing within the correct output's corresponding subspace. To verify this theoretical alignment, we construct synthetic logits using the explicit group characters $\chi_\rho(abc^\sharp)$ and demonstrate a direct geometric correspondence with the model's empirical outputs.
\subsection{Embeddings}
\label{sec:embed}
We verify the $\J$-class partitioning claim by first directly looking at the model embedding weights. We additionally seek to permute the weights of elements by their $\J$-classes, such that the internal Fourier geometry representations (Section \ref{sec:review}) can be exposed.

\textbf{Fourier Features Analysis.} We begin by partitioning the embedding matrix $W_E$ into sub-matrices $W_E^{J_d}$, each corresponding to a $\J$-class of $\Z_{165}$. Because each class forms a finite abelian group (Theorem \ref{thm:cong}), we can cleanly permute the rows of each sub-matrix lexicographically by their minimal generating set to yield $\widetilde{W}_E^{J_d}$.

\begin{figure*}[t!]
    \centering
    \includegraphics[width=0.9\linewidth]{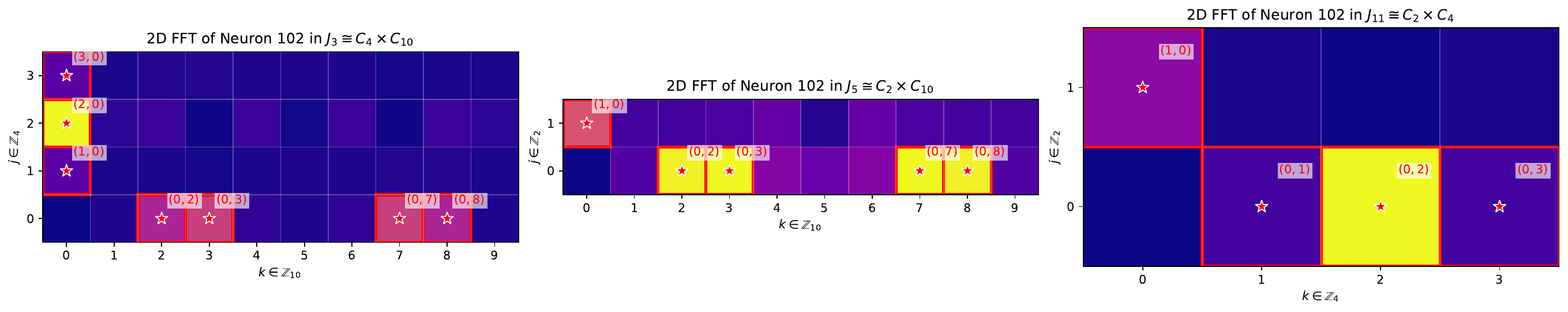} 
    \caption{\textbf{Compositional factorization of Fourier features within Neuron 102.} A 2D DFT of the permuted embeddings isolates the neuron's active frequencies (red stars). The network constructs non-cyclic representations by reusing frequencies from shared ``atomic'' cyclic factors. For instance, $J_3$ and $J_5$ utilize the exact same frequencies along their shared $k \in \Z_{10}$ axis. This provides visual evidence that the learned representations reflect the monoid’s direct product structure. Color intensity denotes \emph{relative} amplitude.}
    \label{fig:hierarchy}
\end{figure*}

\begin{table}[t]
    \centering
    \caption{\textbf{Fraction of variance explained (FVE) by key frequencies across non-trivial $\mathcal{J}$-classes.} For each $\mathcal{J}$-class, a small subset of the total available frequencies (Key Freq) accounts for the vast majority of the spectral energy in the permuted embedding space. Note that $J_{33}$, $J_{55}$, and $J_{165}$ are excluded; as these groups become sufficiently small and trivial, all frequencies become similarly crucial.}
    \label{tab:key-freqs}
    \small
    \begin{tabular}{cccc}
        \toprule
        \textsc{$\J$-class} & \textsc{Key Freq} & \textsc{Total Freq} & FVE \\
        \midrule
        $J_1$ & $8$ & $79$ & $97.0\%$\\
        $J_3$ & $7$ & $39$ & $96.8\%$\\
        $J_5$ & $5$ & $19$ & $95.0\%$\\
        $J_{11}$ & $4$ & $7$ & $99.1\%$\\
        $J_{15}$ & $4$ & $9$ & $94.3\%$\\
        \bottomrule
    \end{tabular}
\end{table}

We then extract element representations by applying a discrete Fourier transform (DFT) to each $\widetilde{W}_E^{J_d}$. As seen in Table \ref{tab:key-freqs}, a smaller fraction of the available frequencies account for over 94\% of the spectral energy variance across all non-trivial $\J$-classes.

This is consistent with prior observations that learned arithmetic representations often concentrate on sparse Fourier modes (\citet{nanda2023progress,gromov2023grokking,zhong2023clock}; Section \ref{sec:review}), and suggests that analogous structure appears in the non-cyclic local groups arising from $\J$-classes.

More importantly, this supports our hypothesis that representations of \emph{all} elements, including zero-divisors, \textbf{are subordinately structured by the $\J$-class partitions} applied by the network's embeddings.

Furthermore, we observe a clear compositional structure within these representations. Because non-cyclic $\J$-classes are isomorphic to direct products of smaller, cyclic $\J$-classes (which we call the ``atomic factors'' of the monoid), their corresponding Fourier features directly reflect this factorization. Rather than learning entirely unique representations for composite classes like $J_3$ or $J_5$, it appears the network chooses to build their representations by systematically reusing the key frequencies of their underlying atomic components. We hereon refer to this phenomenon along our interpretation as the \emph{atomic factorization hypothesis}.

As visually supported by Figure \ref{fig:hierarchy}, applying a 2D DFT to individual neurons reveals that composite classes (e.g., $J_3$ and $J_5$) explicitly construct their representations by reusing the active frequency coordinates of their shared atomic factors. For more details, see Appendix \ref{app:embed-hierarchy}.

\textbf{PCA on Embedding.} We provide further evidence that $\J$-classes are isolated into specific subspaces of the available $d_{\text{model}} = 128$ dimensions by applying principal component analysis (PCA; \citealt{jolliffe2002principal}) to each $W_E^{J_d}$. We found that fewer than 20\% of the principal components are needed to explain over 95\% of the variance in the embedding weights across all $\J$-classes.

Crucially, to ensure the model is not simply learning a generally low-dimensional space due to the small size of the $\J$-classes themselves, we compared the number of components needed to explain 95\% of the variance for each $\J$-class against random subsets of the vocabulary of the same size (Figure \ref{fig:pca-dim}). The results show that the network compresses elements of the same $\J$-class together much more tightly than expected by random chance.

\begin{figure}[t!]
    \centering
    \includegraphics[width=0.8\linewidth]{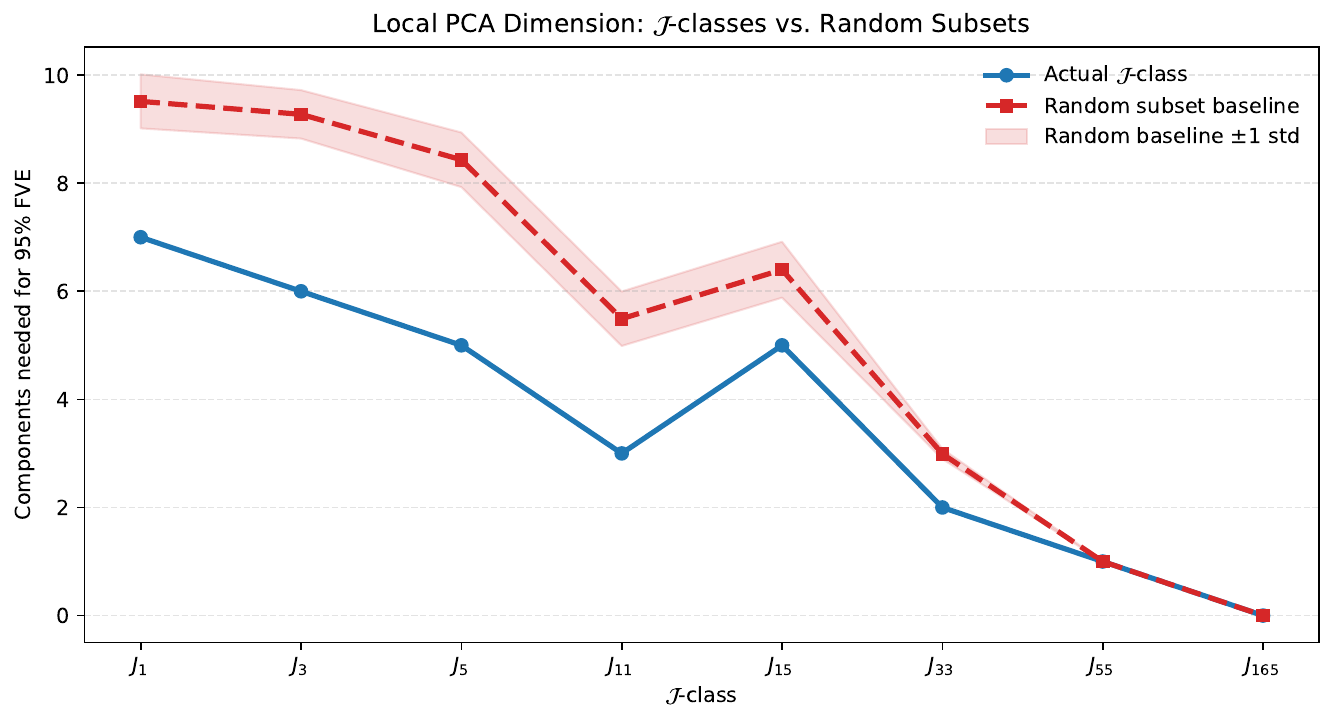} 
    \caption{\textbf{Principal components required to explain 95\% of embedding variance.} $\J$-classes (solid) require substantially fewer dimensions than equivalently sized random subsets of the vocabulary (dashed). This dimensional collapse provides evidence that the network explicitly compresses elements of the same class into tight, localized subspaces.}
    \label{fig:pca-dim}
\end{figure}
Comprehensive 2D and 3D PCA projections of the embeddings, which synthesizes the highly structured Fourier geometry with the hierarchical subspace of distinct $\J$-classes, are provided in Appendix \ref{app:embed-visual}.

\subsection{Multi-Head Attention}
\label{sec:attn}
\begin{figure*}[ht]
    \centering
    \includegraphics[width=0.95\linewidth]{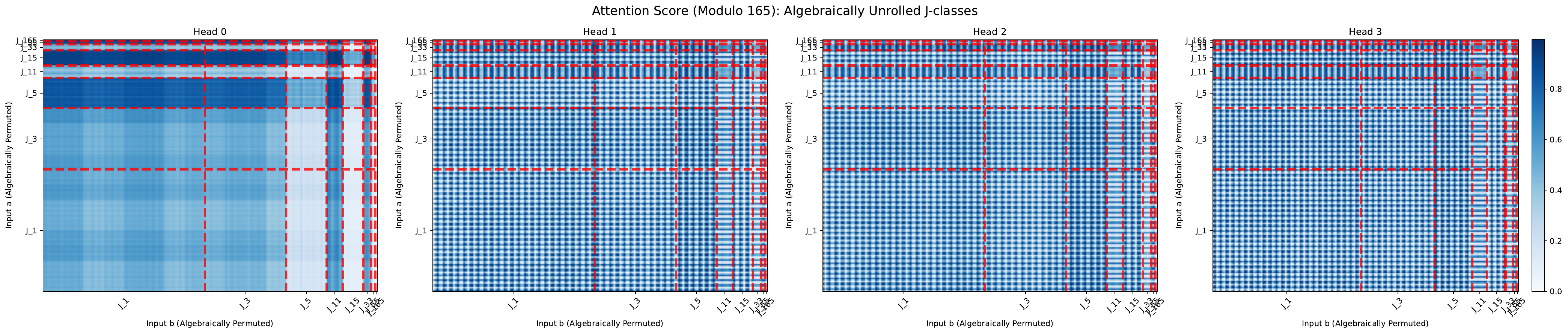} 
    \caption{\textbf{Attention scores $A^{(h)}$ computed by the QK circuit for the destination token `='.} Head 0 (left) exhibits a block-diagonal structure aligned with $\mathcal{J}$-class boundaries(red dashed lines), functioning as a macroscopic router. In contrast, Heads 1-3 display high-frequency checkerboard patterns consistent with finer within-class phase information.}
    \label{fig:attn-scores}
\end{figure*}
The monoid extension hypothesizes that the network executes a routing phase to project elements into the appropriate output subspace. We now present evidence that the attention heads participate in this routing phase.

We reverse-engineer our one-layer multi-head attention using interpretability
methods as outlined in \citet{elhage2021framework}. For a given attention head $h$, the output written to the residual stream is a weight sum 
\begin{equation}
\text{head}^{(h)} = \sum_{m \in \{a, b\}} A^{(h)}_{=, m} \rho(m) W_{OV}^{(h)}\label{eq:attn-head}
\end{equation}
Here $W_{OV}$ represents the \textbf{OV circuit}, dictating \emph{how} the operand representations are projected onto the output subspace. Further, the attention score, defined \[
 A^{(h)}_{=, m} = \text{softmax}\left( \frac{\rho(=)W_{QK}^{(h)}\rho(m)^T}{d_k}\right)
\]
represents the output of the \textbf{QK circuit} $W_{QK}$, dictating \emph{how much} information is moved from each operand. We now procedurally explore each circuit.

\textbf{Factorized Routing in QK Circuitry.} To investigate the QK circuit, we look at the softmax attention score for head 0 from the token `=' to `$a$', as a function of inputs $a, b$. This is similar to techniques applied in \citet{nanda2023progress}.

However, rather than looking at these matrices naively, we apply the same partitioning and permutation as described in Section \ref{sec:embed} to all pairs of inputs $a, b$, segmenting the attention scores by the $\J$-class partitions, as seen in Figure \ref{fig:attn-scores}. 

Immediately, we notice that Head 0 exhibits a block structure consistent with coarse J-class routing (red dashed lines), while Heads 1-3 show higher-frequency patterns consistent with finer within-class information.

This separation of roles is important because the QK circuit controls only the \emph{amount} of information transferred from each operand, not the content of the information itself. Thus, the block structure in Head 0 suggests that the model first identifies the broad algebraic structure relevant to the output, while the periodic structure in the remaining heads suggests that the Fourier-phase representation information is still available to distinguish elements within the same class. 

Consistent with prior literature, this precise periodicity suggests that the Fourier features extracted by the embedding layer are not localized artifacts of the lookup table. Rather, they appear to be preserved into the attention computation, where they influence the movement of operand information into the final-token residual stream.

\begin{figure}[t!]
    \centering
    \includegraphics[width=\linewidth]{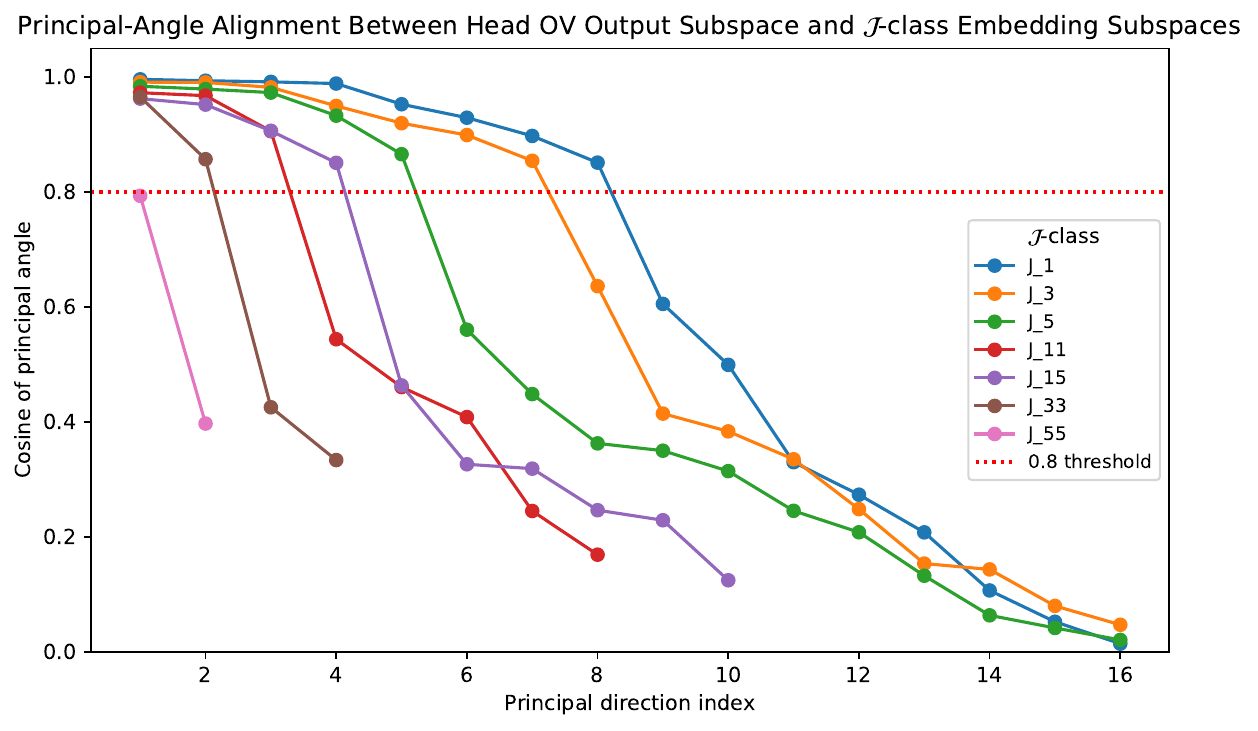} 
    \caption{\textbf{Principal angle alignment between the OV circuit and $\mathcal{J}$-class embeddings.} Alignment remains near perfect before a sharp collapse. Crucially, the rank of this drop-off approximates the number of key Fourier frequencies (e.g., $J_1$ drops after 8 dimensions, $J_3$ after 7). While trivially small classes (e.g., $J_{55}$) lack the group size to resolve stable PCA subspaces, the overall trend provides evidence that $W_{OV}^{(0)}$ acts as a precise, low-rank projection filter isolating the target atomic frequencies.}
    \label{fig:cosine-align}
\end{figure}

\textbf{Idempotent Projections in OV Circuitry.} While Figure \ref{fig:attn-scores} suggests that the QK circuit, and specifically Head 0, extracts identifying information of the output's broader algebraic structure, Equation \eqref{eq:attn-head} implies $W_{OV}^{(0)}$ should have projection-like behavior in order to correctly map the operands into the target subspace of this structure.

To verify this geometric behavior, we first apply PCA on $W_{OV}^{(0)}$. Through this, we found that of the 32 singular values available, only 4 were needed to explain 95\% of the variance in weights of the circuit. Additionally, the top 8 singular values altogether explains more than 99.9\% of the total variance. This result clearly demonstrates the circuit's projection down to subspaces of similar dimensions as the $\J$-class embeddings.

From here, in order to ensure the resulting subspaces are aligning exactly with those of each $\J$-class, we compute the principal-angle alignments \citep{bjork1973numerical} between the OV circuit and the corresponding $\J$-class embeddings. As shown in Figure \ref{fig:cosine-align}, the alignment remains relatively high (cosine $\geq 0.8$) for a distinct number of principal directions before experiencing a sharp drop-off. Crucially, the rank at which this collapse occurs for each $\J$-class closely corresponds to the number of key frequencies (irreducible representations) established in Table \ref{tab:key-freqs}. This provides strong evidence that \textbf{$W_{OV}^{(0)}$ acts as a targeted low-rank projection matrix}, either preserving or restricting operand representations to ensure they reside strictly within the output subspace dictated by their key frequencies.

Furthermore, while we previously established that distinct $\J$-class embeddings occupy localized subspaces within the 128-dimensional model, Figure \ref{fig:cosine-align} reveals a more detailed context. Since the embedding subspaces align closely with the $W_{OV}^{(0)}$ projection space regardless of class size, the subspaces of smaller $\J$-classes must geometrically overlap with those of larger $\J$-classes. This structural overlap is highly consistent with the ``atomic factorization'' hypothesis introduced in Section \ref{sec:embed}.

\subsection{MLP Neurons}
\label{sec:mlp}

While attention provides routing and low-rank subspace projection, the model still requires a nonlinear component to compose the routed operand representations into a representation of the product. Both GCR and the Monoid Extension predict that this compositional step occurs in the MLP.

\textbf{Clustering persists in neuron activations.} To probe this, we visualize hidden-layer MLP activations over input pairs $(a,b)$. We find that activations remain organized by $\mathcal J$-class and exhibit frequency structure matching the key modes identified in the embeddings and attention. This suggests that the MLP contributes to the nonlinear composition step while preserving the same stratified Fourier structure used throughout the network. These plots, along with further empirical analysis, can be found in \ref{app:mlp_heatmaps}. A complete neuron-level reconstruction of the MLP composition step is left to future work.

\subsection{Unembedding and Logit Computation}
To validate the monoid extension hypothesis, we directly compare the model's empirical output logits against theoretical predictions derived from Equation \eqref{eq:logit-formula}.

\textbf{Character Fitting.} We define an input \emph{prompt} as a tuple $(a, b)$ where $a \in J_{d_1}$ and $b \in J_{d_2}$. A prompt evaluates to a target class $J_d$ if and only if $\text{lcm}(d_1, d_2) = d$ (i.e., $ab \in J_d$).

Because our vocabulary is finite ($n=165$), we exhaustively evaluate all valid prompts for each target class. For every valid prompt, we construct a synthetic logit vector to compare the actual model output against. Specifically, this synthetic logit vector has dimension $|J_d|$, where the entry for the candidate token $c$ is computed explicitly via the local group character, $\chi_\rho(abc^\sharp)$.  Crucially, these characters are evaluated strictly with respect to the irreducible representations corresponding to the key frequencies $k$ isolated in our embedding analysis (Section \ref{sec:embed}). 

We perform a forward pass for all valid prompts to extract the model's empirical logits (the direct output of the unembedding layer $W_U$), then plot these empirical logits against our synthetic theoretical logits across all valid prompts. We then fit a linear regression and compute the coefficient of determination ($R^2$). This metric quantifies the extent to which the network's final output variance is explained by the theoretical character trace computation.

As seen in Table~\ref{tab:character-fit}, the theoretical logits computed explicitly from Equation~\eqref{eq:logit-formula} account for a large fraction of the output variance. This provides quantitative evidence that Fourier-trace features capture a substantial component of the final logit computation, consistent with the Monoid Extension hypothesis. At the same time, the lower fits for larger $\J$-classes such as $J_1$ and $J_3$ suggest that additional mechanisms: including nonlinear scaling, harmonic distortion, or sparse class-specific corrections may also contribute to the model's logits.

\begin{table}[t]
    \centering
    \caption{\textbf{Fraction of centered logit variance explained by local character features.} For each target $\J$-class, we fit a linear model using only the theoretically specified character features $\chi_\rho(abc^\sharp)$ at the key frequencies identified in the embedding spectrum. These features explain a large fraction of the model’s empirical logit variance, supporting a local character-based readout.}
    \label{tab:character-fit}
    \small
    \begin{tabular}{ccc}
        \toprule
        \textsc{$\J$-class} & \textsc{Input Pairs} & \textsc{FVE} $(R^2)$ \\
        \midrule
        $J_{55}$ & $J_5 \times J_{11}$ & $99.4\%$\\
        $J_{33}$ & $J_3 \times J_{11}$ & $80.2\%$\\
        $J_{15}$ & $J_3 \times J_5$ & $95.2\%$\\
        $J_{11}$ & $J_1 \times J_{11}$ & $86.8\%$\\
        $J_{5}$ & $J_1 \times J_5$ & $86.9\%$\\
        $J_{3}$ & $J_1 \times J_3$ & $76.5\%$\\
        $J_{1}$ & $J_1 \times J_1$ & $71.2\%$\\
        \bottomrule
    \end{tabular}
\end{table}

\section{Conclusion / Future Work}
In this work, we take a step toward understanding how neural networks implement algebraic computation beyond groups. We use mechanistic interpretability to show that previous representation-theoretic algorithms for interpreting how small networks perform group operations \citep{chughtai2023toy,stander2024cosets} can be extended beyond globally invertible settings. We refer to this localized generalization as the \emph{Monoid Extension}. Specifically, we study modular integer multiplication, where zero divisors prevent a single global group-based explanation. In the square-free case, we find that the learned computation organizes around regular $\J$-classes, within which local inverses and local Fourier characters recover a group-like explanation of the model's logits.

Moreover, rather than relying on shallow memorization or surface-level heuristics, the network sub-divides the task into phases executed by different parts:

\begin{itemize}[topsep=0pt, itemsep=2pt]
    \item The \emph{embedding layer} organizes inputs according to algebraically meaningful structure. In particular, embeddings cluster into subspaces aligned with the $\J$-class decomposition of the multiplication monoid.
    \item The \emph{multi-head attention layer} appears to route operand information into the residual stream in a class-sensitive way, with different heads separating coarse $\J$-class information from finer within-class phase information.
    \item The \emph{MLP} performs the nonlinear composition step, combining the routed operand representations into a representation of the product.
    \item The \emph{unembedding layer} converts this product representation into logits by implementing a local character-style readout, using local inverses $c^\sharp$ within each regular $\J$-class in place of the global inverses used in group-based GCR.
\end{itemize}
However, it is important to note some limitations of our work. Crucially, at this time, \textbf{our evidence is only correlational}, meaning causation cannot be directly inferred. In future work, we plan to apply causal intervention techniques, such as ablations and activation patching, to provide stronger evidence for these claims.

Below we discuss some additional areas of future work.
\label{sec:future-work}

\textbf{Investigation on non-square-free moduli.} Our monoid extension framework analyzes square-free moduli (which have a natural density of $6/\pi^2 \approx 0.608$; Remark~\ref{rem:square-free}), where $\J$-classes are guaranteed to be regular and therefore locally invertible. However, non-square-free moduli introduce non-regular $\J$-classes that contain \emph{nilpotent} elements (elements where $x^n = 0$), which breaks the reduction to local group characters \citep{howie1995semigroup,steinberg2016representation}. Future work must investigate how networks handle these algebraic ``one-way'' collapses, leading to a complete analysis of all integer moduli.

\textbf{Progress measures of ``atomic'' structure build-up.} Consistent with similar-sized networks trained on modular arithmetic, our model experiences grokking \citep{power2022grokking,nanda2023progress}. Building on our ``atomic factoring'' hypothesis from Section \ref{sec:embed}, we suspect the network's stratification of inputs is learned directly during this grokking phase. A comprehensive understanding of these dynamics requires progress measures that track whether larger $\J$-classes are sequentially composed from previously learned smaller counterparts.

\textbf{Larger models and realistic tasks.} In this work, we studied the behavior of small models on non-invertible compositions. However, we did not explore whether our results apply to larger, more practical models. Core real-world LLM operations, such as token sequence concatenation, are mathematically analogous to non-invertible monoid compositions. Future work should investigate whether production-scale models employ similar structural stratification to hierarchically route and process complex textual inputs.

\section*{Acknowledgements}
We would like to thank Jarod Alper, Vasily Ilin, Michael Ruofan Zeng and the entirety of the Math AI Lab at the University of Washington for their belief, guidance, resources, and most importantly, for bringing our team together. 

We are also grateful for Claire Xu, Ivonne Zhang, and Nina Tharamal for their contributions to the project in its early stages. In addition, we are grateful for Jarod Alper for providing much needed belief and valuable feedback on our manuscript.

Z.A. Chen would like to thank his family, David, Sophie, and Jack for their unwavering support, without whom this journey would not have been possible.

Of course, none of this work would have been possible without the groundwork of modular addition set by Neel Nanda, and the generalization to all groups via GCR by Bilal Chughtai.

We trained our models using PyTorch \cite{paszke2019pytorch} and performed our interpretability analysis using NumPy \cite{harris2020numpy}, scikit-learn \cite{pedregosa2011scikit}, and SymPy \cite{meurer2017sympy}. Our figures were made using Matplotlib \cite{hunter2007matplotlib} and Plotly \cite{plotly2015collaborative}. Computational resources were provided by Google Colaboratory.

\section*{Impact Statement}
This paper presents work whose goal is to advance the field of Machine Learning.
More specifically, it studies mechanistic interpretability methods for
understanding learned algorithmic structure in neural networks. We do not
anticipate direct societal risks from the specific toy models studied here,
beyond the general consequences of progress in interpretability and machine
learning research.

\bibliography{main}

@article{power2022grokking,
  title={Grokking: Generalization Beyond Overfitting on Small Algorithmic Datasets},
  author={Power, Alethea and Burda, Yuri and Edwards, Harri and Babuschkin, Igor and Misra, Vedant},
  journal={arXiv preprint arXiv:2201.02177},
  year={2022},
  url={https://arxiv.org/abs/2201.02177}
}

@inproceedings{nanda2023progress,
  title={Progress Measures for Grokking via Mechanistic Interpretability},
  author={Nanda, Neel and Chan, Lawrence and Lieberum, Tom and Smith, Jess and Steinhardt, Jacob},
  booktitle={International Conference on Learning Representations},
  year={2023},
  url={https://openreview.net/forum?id=9XFSbDPmdW}
}

@article{gromov2023grokking,
  title={Grokking Modular Arithmetic},
  author={Gromov, Andrey},
  journal={arXiv preprint arXiv:2301.02679},
  year={2023},
  url={https://arxiv.org/abs/2301.02679}
}

@inproceedings{zhong2023clock,
  title={The Clock and the Pizza: Two Stories in Mechanistic Explanation of Neural Networks},
  author={Zhong, Ziqian and Liu, Ziming and Tegmark, Max and Andreas, Jacob},
  booktitle={Advances in Neural Information Processing Systems},
  volume={36},
  pages={27223--27250},
  year={2023},
  url={https://openreview.net/forum?id=S5wmbQc1We}
}

@inproceedings{chughtai2023toy,
  title={A Toy Model of Universality: Reverse Engineering How Networks Learn Group Operations},
  author={Chughtai, Bilal and Chan, Lawrence and Nanda, Neel},
  booktitle={Proceedings of the 40th International Conference on Machine Learning},
  series={Proceedings of Machine Learning Research},
  volume={202},
  pages={6243--6267},
  year={2023},
  publisher={PMLR},
  url={https://proceedings.mlr.press/v202/chughtai23a.html}
}

@inproceedings{stander2024cosets,
  title={Grokking Group Multiplication with Cosets},
  author={Stander, Dashiell and Yu, Qinan and Fan, Honglu and Biderman, Stella},
  booktitle={Proceedings of the 41st International Conference on Machine Learning},
  series={Proceedings of Machine Learning Research},
  volume={235},
  pages={46441--46467},
  year={2024},
  publisher={PMLR},
  url={https://proceedings.mlr.press/v235/stander24a.html}
}

@article{olah2020zoom,
  title={Zoom In: An Introduction to Circuits},
  author={Olah, Chris and Cammarata, Nick and Schubert, Ludwig and Goh, Gabriel and Petrov, Michael and Carter, Shan},
  journal={Distill},
  year={2020},
  doi={10.23915/distill.00024.001},
  url={https://distill.pub/2020/circuits/zoom-in}
}

@article{cammarata2020circuits,
  title={Thread: Circuits},
  author={Cammarata, Nick and Carter, Shan and Goh, Gabriel and Olah, Chris and Petrov, Michael and Schubert, Ludwig},
  journal={Distill},
  year={2020},
  url={https://distill.pub/2020/circuits}
}

@article{elhage2021framework,
  title={A Mathematical Framework for Transformer Circuits},
  author={Elhage, Nelson and Nanda, Neel and Olsson, Catherine and Henighan, Tom and Joseph, Nicholas and Mann, Ben and Askell, Amanda and Bai, Yuntao and Chen, Anna and Conerly, Tom and DasSarma, Nova and Drain, Dawn and Ganguli, Deep and Hatfield-Dodds, Zac and Hernandez, Danny and Jones, Andy and Kernion, Jackson and Lovitt, Liane and Ndousse, Kamal and Amodei, Dario and Brown, Tom and Clark, Jack and Kaplan, Jared and McCandlish, Sam and Olah, Chris},
  journal={Transformer Circuits Thread},
  year={2021},
  url={https://transformer-circuits.pub/2021/framework/index.html}
}

@inproceedings{raghu2017svcca,
  title={{SVCCA}: Singular Vector Canonical Correlation Analysis for Deep Learning Dynamics and Interpretability},
  author={Raghu, Maithra and Gilmer, Justin and Yosinski, Jason and Sohl-Dickstein, Jascha},
  booktitle={Advances in Neural Information Processing Systems},
  volume={30},
  year={2017},
  url={https://arxiv.org/abs/1706.05806}
}

@inproceedings{morcos2018insights,
  title={Insights on Representational Similarity in Neural Networks with Canonical Correlation},
  author={Morcos, Ari S. and Raghu, Maithra and Bengio, Samy},
  booktitle={Advances in Neural Information Processing Systems},
  volume={31},
  year={2018},
  url={https://arxiv.org/abs/1806.05759}
}

@book{jolliffe2002principal,
  title={Principal Component Analysis},
  author={Jolliffe, Ian T.},
  edition={2},
  series={Springer Series in Statistics},
  publisher={Springer},
  year={2002},
  doi={10.1007/b98835}
}

@inproceedings{vaswani2017attention,
  title={Attention Is All You Need},
  author={Vaswani, Ashish and Shazeer, Noam and Parmar, Niki and Uszkoreit, Jakob and Jones, Llion and Gomez, Aidan N. and Kaiser, Lukasz and Polosukhin, Illia},
  booktitle={Advances in Neural Information Processing Systems},
  volume={30},
  year={2017},
  url={https://arxiv.org/abs/1706.03762}
}

@inproceedings{loshchilov2019decoupled,
  title={Decoupled Weight Decay Regularization},
  author={Loshchilov, Ilya and Hutter, Frank},
  booktitle={International Conference on Learning Representations},
  year={2019},
  url={https://openreview.net/forum?id=Bkg6RiCqY7}
}

@book{howie1995semigroup,
  title={Fundamentals of Semigroup Theory},
  author={Howie, John M.},
  series={London Mathematical Society Monographs},
  publisher={Oxford University Press},
  year={1995}
}

@book{steinberg2016representation,
  title={Representation Theory of Finite Monoids},
  author={Steinberg, Benjamin},
  series={Universitext},
  publisher={Springer},
  year={2016},
  doi={10.1007/978-3-319-43932-7}
}

@article{bjork1973numerical,
  title={Numerical Methods for Computing Angles Between Linear Subspaces},
  author={Bj{\"o}rck, {\AA}ke and Golub, Gene H.},
  journal={Mathematics of Computation},
  volume={27},
  number={123},
  pages={579--594},
  year={1973},
  doi={10.1090/S0025-5718-1973-0348991-3}
}

@book{apostol1976analytic,
  title={Introduction to Analytic Number Theory},
  author={Apostol, Tom M.},
  publisher={Springer},
  year={1976}
}

@inproceedings{paszke2019pytorch,
  title={PyTorch: An Imperative Style, High-Performance Deep Learning Library},
  author={Paszke, Adam and others},
  booktitle={Advances in Neural Information Processing Systems 32},
  year={2019}
}

@article{harris2020numpy,
  title={Array programming with NumPy},
  author={Harris, Charles R. and others},
  journal={Nature},
  volume={585},
  number={7825},
  pages={357--362},
  year={2020},
  publisher={Nature Publishing Group}
}

@article{pedregosa2011scikit,
  title={Scikit-learn: Machine Learning in Python},
  author={Pedregosa, Fabian and others},
  journal={Journal of Machine Learning Research},
  volume={12},
  pages={2825--2830},
  year={2011}
}

@article{meurer2017sympy,
  title={SymPy: symbolic computing in Python},
  author={Meurer, Aaron and others},
  journal={PeerJ Computer Science},
  volume={3},
  pages={e103},
  year={2017},
  publisher={PeerJ Inc.}
}

@article{hunter2007matplotlib,
  title={Matplotlib: A 2D graphics environment},
  author={Hunter, John D.},
  journal={Computing in Science \& Engineering},
  volume={9},
  number={3},
  pages={90--95},
  year={2007},
  publisher={IEEE}
}

@misc{plotly2015collaborative,
  title={Collaborative data science},
  author={{Plotly Technologies Inc.}},
  year={2015},
  address={Montreal, QC},
  url={https://plot.ly}
}
\bibliographystyle{icml2026}

\newpage
\appendix
\onecolumn
\section{Supplemental Material}
Interactive versions of figures, as well as the code to reproduce the main body results, are available at \texttt{https://github.com/uw-math-ai/interpreting-monoids}

\section{Model}
\label{app:model}

Our model is trained on 30\% of all $n^2$ entries in the multiplication table. We use full batch gradient descent. We utilize the AdamW optimizer, with weight decay $1$, learning rate $10^{-3}$, and betas $\beta_1 = 0.9, \beta_2 = 0.98$. We train each model for 25,000 epochs. We also withhold another 30\% of the $n^2$ for validation. We use this to compute our validation loss every 100 epochs. We finally compute our model's final accuracy on all $n^2$ inputs.

\begin{figure}[!h]
    \centering
    \includegraphics[width=0.9\linewidth]{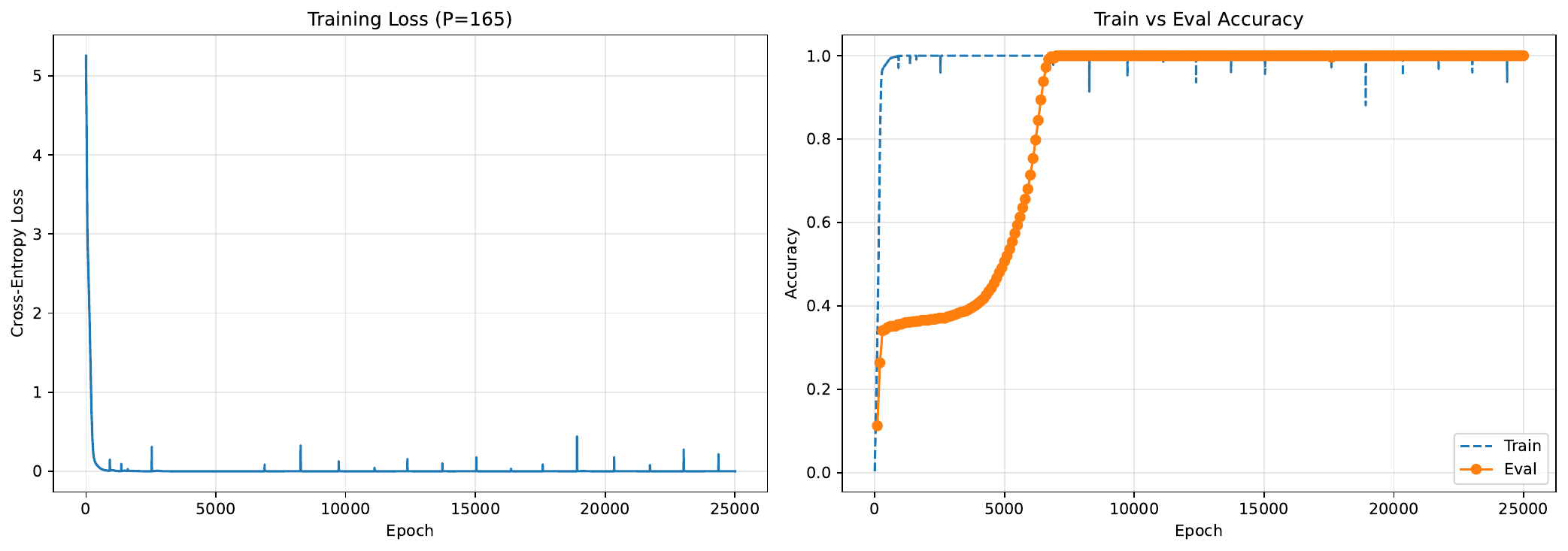} 
    \caption{\textbf{Training plots for the single-layer decoder-only transformer on n=165.} Left: training loss across epochs, which decreases sharply as the model fits the training data. Right: training and validation accuracy across epochs, with the grokking gap before validation accuracy rises to match training accuracy.}
    \label{fig:TranformerTrainingPlots}
\end{figure}

\subsection{Transformer}
We use a decoder-only transformer model identical to the setup used for the experiments in \cite{nanda2023progress}

\begin{figure}[b]
    \centering
    \includegraphics[width=0.9\linewidth]{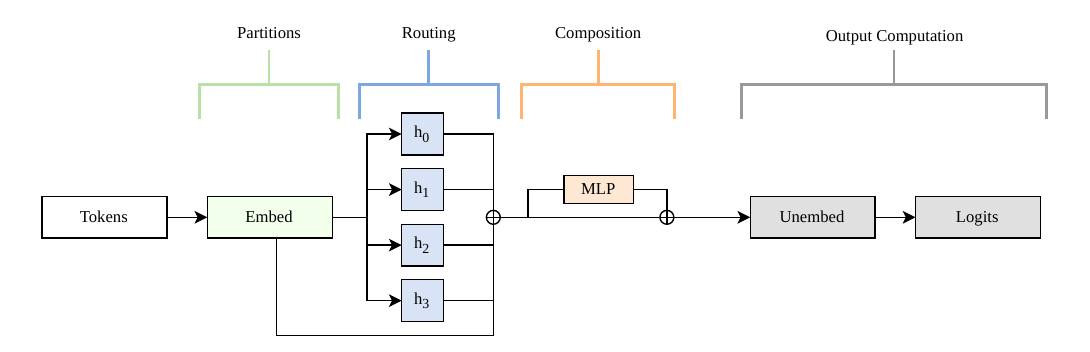} 
    \caption{\textbf{Schematic of the proposed monoid-extension circuit.} Tokens are first embedded into the residual stream, where representations are organized by algebraic partitions. Attention heads route operand information into the appropriate local subspace, the MLP performs the nonlinear composition step, and the unembedding converts the resulting product representation into output logits.}
    \label{fig:model}
\end{figure}

We use \(d_{\mathrm{vocab}}=n+1\), with residues \(0,\dots,n-1\) and a special \(=\) token indexed by \(n\). Inputs $a, b,$ and $=$ are one-hot encoded as $n+1$ dimensional vectors. Each one-hot encoded vector is embedded with $d_\text{model} = 128$, which begins the residual stream.

We then denote the following parameters. $W_E$ is the embedding matrix. $W_Q^i, W_K^i, W_V^i, W_O^i$ correspond to the query, key, value, and output matrix, respectively, of head $i$. We then have $W_\text{in}, b_\text{in}, W_\text{out}, b_\text{out}$ as the input and output matrices for the MLP. Finally, we have $W_U$ for the unembedding layer.

We can let $t_a, t_b, t_=$ represent the one-hot encoded token representation of $a, b,$ and $=$. We also note that loss is computed only on the logits on the final token, corresponding to $=$. After the attention module, all outputs refer to just the final token, because information only moves between tokens during the self-attention step.

After embedding, our initial residual stream on token $i$, which we denote as $x_i^{(0)}$, becomes:
\[
x_i^{(0)} = W_Et_i
\]

We then apply single-layer causal multi-head self-attention. First, we compute the attention score $A^j$:

\[
A^j = \text{softmax}(x^{(0)^T} W_K^j W_Q^j x_=)
\]

We then add the attention outputs back into the residual stream, which we denote as $x^{(1)}$

\[
x^{(1)} = x_i^{(0)} + \sum_jW_O^jW_V^j(x_=^{(0)}\cdot A^j)
\]

We then pass the residual stream after attention through the MLP block, with hidden dimension $d_\text{hidden} = 512$ and a ReLU nonlinearity. We denote the output as $x^{(2)}$:

\[
x^{(2)} = x^{(1)} + W_\text{out}\text{ReLU}(W_\text{in}x^{(1)})
\]

Finally, we pass our outputs through the unembedding matrix, which produces logits over $\R^n$. We then take the argmax of the resulting vector to compute the product $c$.
\[
\text{Logits} = W_Ux^{(2)}
\]

\subsection{MLP Only}
We also train another architecture without the attention module. We use the same training parameters, with $d_\text{model} = 128$, and $d_\text{hidden} = 512$. We can further describe the mathematical structure of this alternate structure.

We take one-hot encoded vectors $a, b$ with dimension $d_\text{vocab} = n$, and pass them through an embedding layer $W_E$, to produce vectors of dimension $d_\text{model}$. We then concatenate the embeddings of $a, b$ to produce a $256$ dimensional model. We can let $t_a, t_b$ represent the one-hot encoded tokens. We define the output of our embedding as $x^{(0)}$

\[
x^{(0)} = [W_E t_a, W_Et_b]
\]

We then pass our embedding through the MLP block. We represent the outputs as $x^{(1)}$:
\[
x^{(1)} = W_\text{out}\text{ReLU}(W_\text{in}x^{(0)})
\]

Unlike the transformer architecture, we have no residual stream, and simply pass this output into the unembedding layer to compute our logits:

\[
\text{Logits} = W_Ux^{(1)}
\]

\subsection{Additional Comments}

For our transformer architecture, we note empirically that the attention paid by the $=$ token to itself is trivial and can be ignored.

We also find that the MLP only architecture generalized faster than the transformer architecture, and performed equally as well. We hypothesize that the MLP is able to route information and perform the computation step without the attention step because of the concatenated embeddings.

\section{Additional Interpretability Analyses}
This section provides supplementary analyses and methodological details supporting the interpretability evidence presented in the main text.
\subsection{Hierarchical Structure of $\J$-class Embeddings} 
\label{app:embed-hierarchy}

In this subsection, we provide additional evidence for the \emph{atomic factorization} hypothesis discussed in Section~\ref{sec:embed}. The main text shows that the embedding matrix contains sparse Fourier structure after being partitioned by $\J$-class. Here, we describe how this structure becomes visible only after applying the algebraic ordering induced by the local group structure of each class.

We first inspect the raw embedding matrix $W_E$ directly, together with a naive one-dimensional FFT taken over the original token order. As shown in Figure~\ref{fig:raw-naive-fft}, this representation is difficult to interpret: while the naive FFT exposes some periodic bands, these bands are not aligned with the algebraic structure of the multiplication monoid. This suggests that the relevant organization is not visible in the default token ordering.

\begin{figure}[H]
    \centering
    \includegraphics[width=0.75\linewidth]{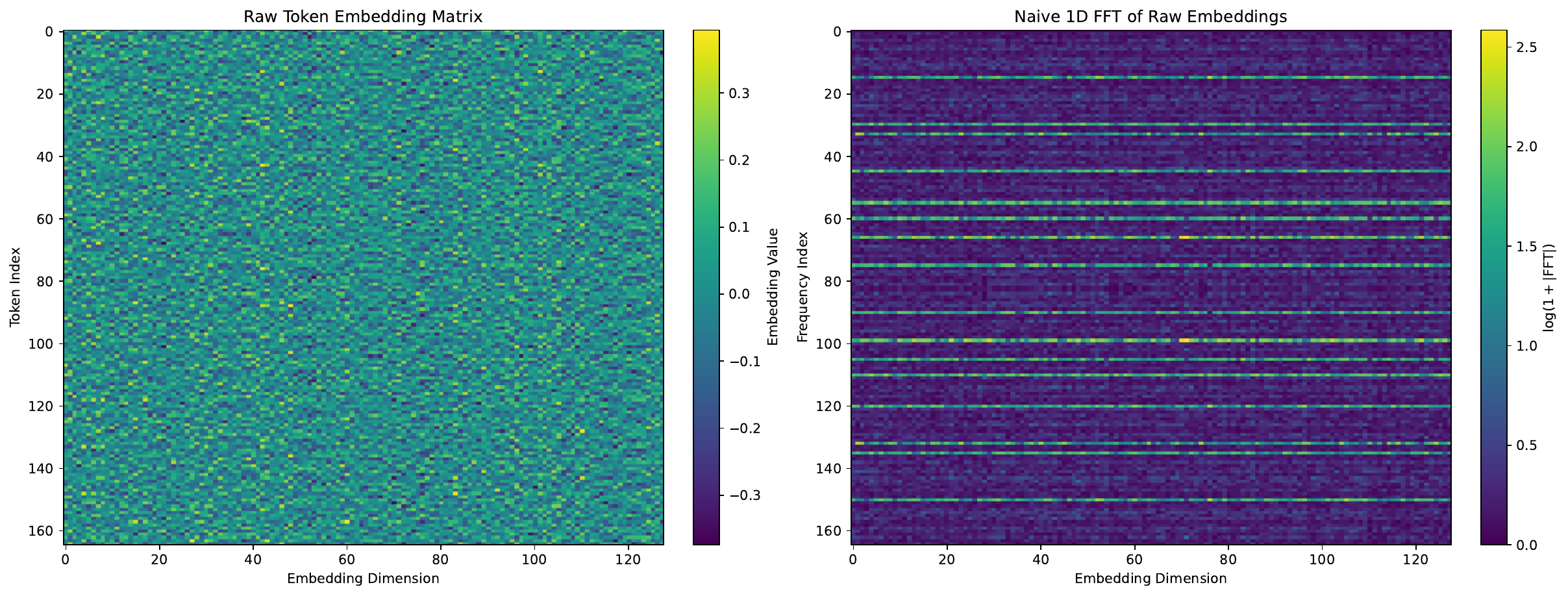} 
    \caption{\textbf{Raw embedding matrix and naive FFT.} Left: the learned token embedding matrix $W_E$ in the original token order. Right: a naive 1D FFT over token index. Although some periodic structure is visible, the raw token order does not respect the algebraic decomposition of $\mathbb Z_{165}$, making the resulting spectrum difficult to interpret.}
    \label{fig:raw-naive-fft}
\end{figure}

Motivated by the $\J$-class decomposition, we then partition $W_E$ into submatrices $W_E^{J_d}$, one for each $\J$-class. Within each class, we further permute rows according to the algebraic coordinates induced by the isomorphism
\[
    J_d \cong \left(\mathbb Z\Big/\left(\frac{165}{d}\right)\mathbb Z\right)^\times.
\]
This produces an algebraically ordered embedding matrix $\widetilde W_E^{J_d}$ for each class. Figure~\ref{fig:jclass-ordered-embedding} compares the embedding matrix after grouping rows only by $\J$-class against the matrix after additionally ordering rows within each class. The latter reveals clearer vertical and periodic structure, suggesting that the model's embedding geometry is organized not merely by class membership, but also by within-class group coordinates.

\begin{figure}[H]
    \centering
    \includegraphics[width=0.75\linewidth]{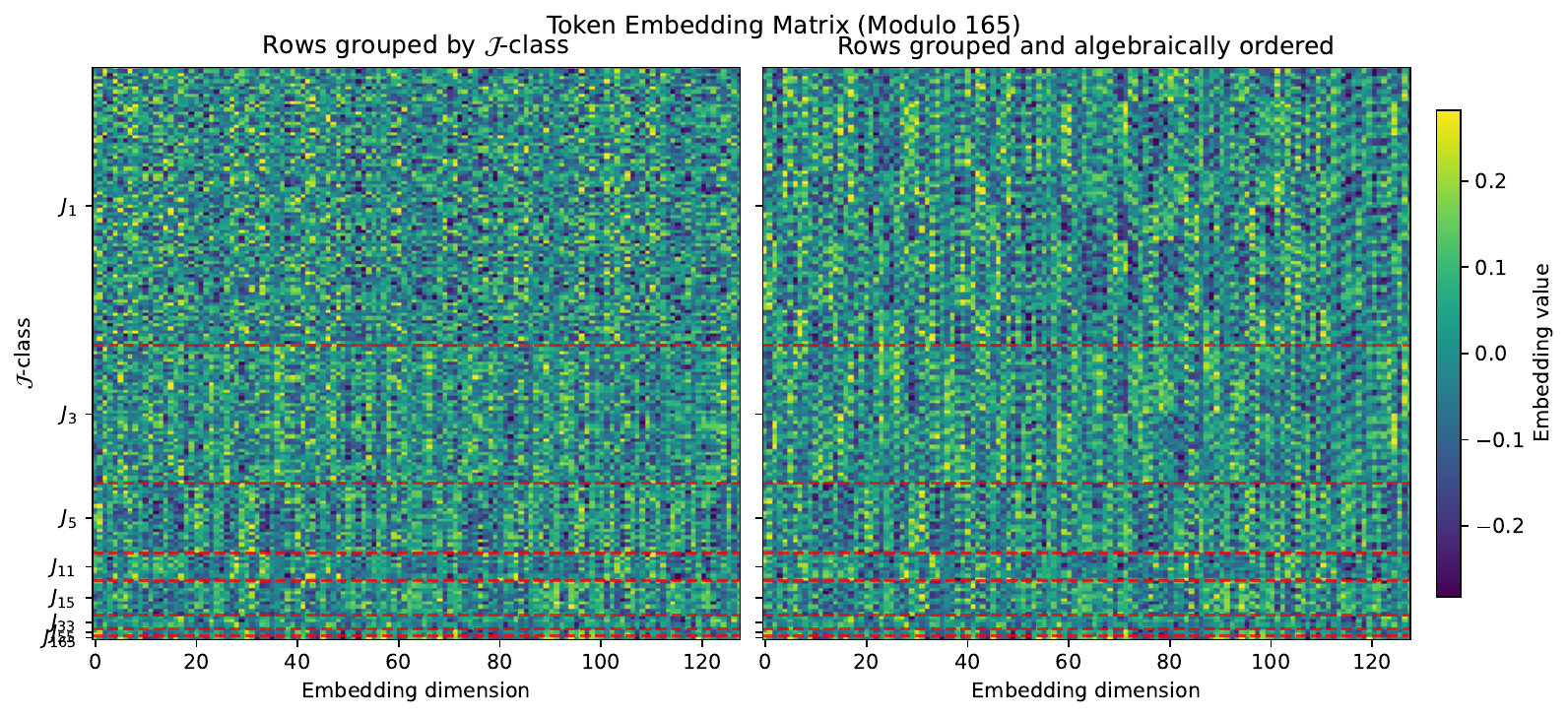} 
    \caption{\textbf{Embedding matrix after $\J$-class grouping and algebraic ordering.} Left: rows grouped by $\J$-class. Right: rows grouped by $\J$-class and then ordered using the local group coordinates inside each class. The algebraic ordering makes periodic structure more visible, indicating that the learned embeddings track within-class group coordinates.}
    \label{fig:jclass-ordered-embedding}
\end{figure}

Finally, we apply a discrete Fourier transform along the algebraic axes of each ordered class $\widetilde W_E^{J_d}$. For cyclic classes this is an ordinary 1D DFT; for product classes, first observe we can use the Chinese Remainder Theorem: 
\[
J_1 \cong C_2 \times C_4 \times C_{10},
\qquad
J_3 \cong C_4 \times C_{10},
\]
we apply a multidimensional DFT  over the corresponding product coordinates and flatten the resulting frequency grid for visualization. Figure~\ref{fig:fft-jclass-embedding} shows the resulting Fourier spectra and total energy at each frequency.

\begin{figure}[!h]
    \centering
    \includegraphics[width=0.9\linewidth]{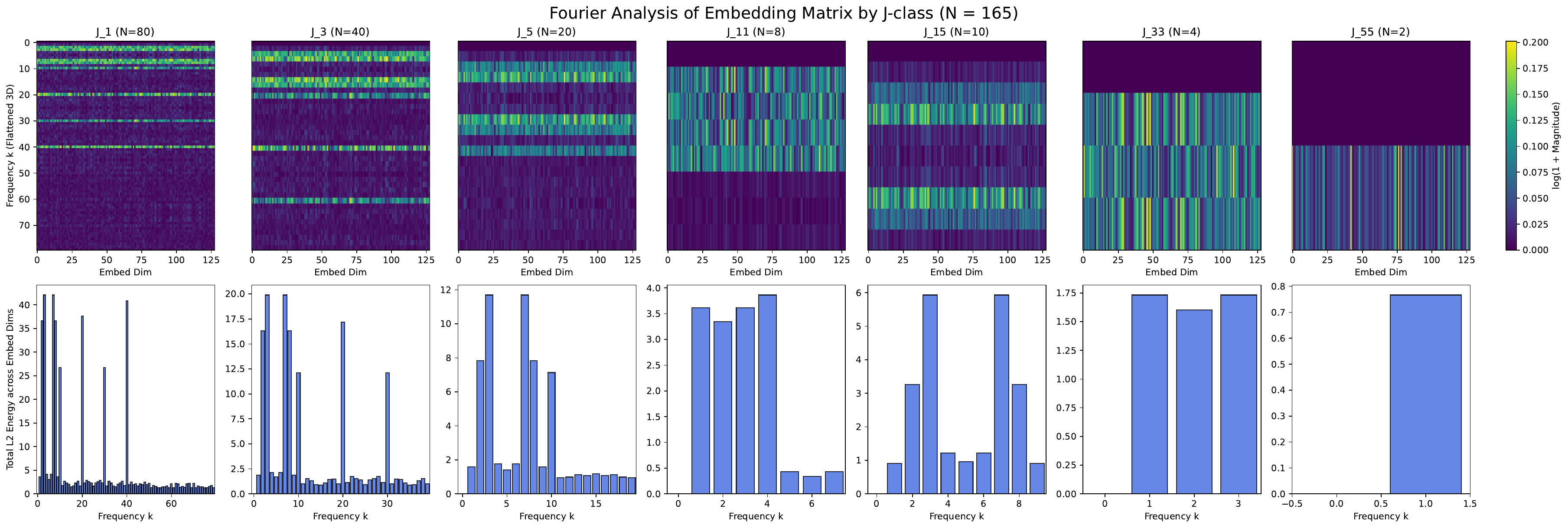} 
    \caption{\textbf{Fourier spectra of $\J$-class embedding blocks.} Top: log-magnitude Fourier spectra of each algebraically ordered embedding block. Bottom: total Fourier energy across embedding dimensions for each frequency. The spectra concentrate on a sparse set of frequencies, matching the key-frequency structure reported in Table~\ref{tab:key-freqs}.}
    \label{fig:fft-jclass-embedding}
\end{figure}

This classwise Fourier analysis reveals two important patterns. First, the embedding geometry is highly sparse in the local Fourier basis: only a small number of frequencies account for most of the spectral energy in each nontrivial $\J$-class. This supports the claim that the model does not treat each element as an unrelated token, but instead represents elements using structured local Fourier coordinates.

Second, the active frequencies appear to be reused across related classes. For example, larger product classes inherit axes corresponding to smaller cyclic factors. This is the sense in which we describe the representation as \emph{atomic}: non-cyclic classes do not appear to require entirely new Fourier structure from scratch, but instead reuse frequency components associated with their cyclic factors. This provides additional evidence for the hypothesis that the model builds composite $\J$-class representations from reusable local group coordinates.

\subsection{Visualizing the Embeddings}
\label{app:embed-visual}

\begin{figure}[H]
    \centering
    \includegraphics[width=0.4\linewidth]{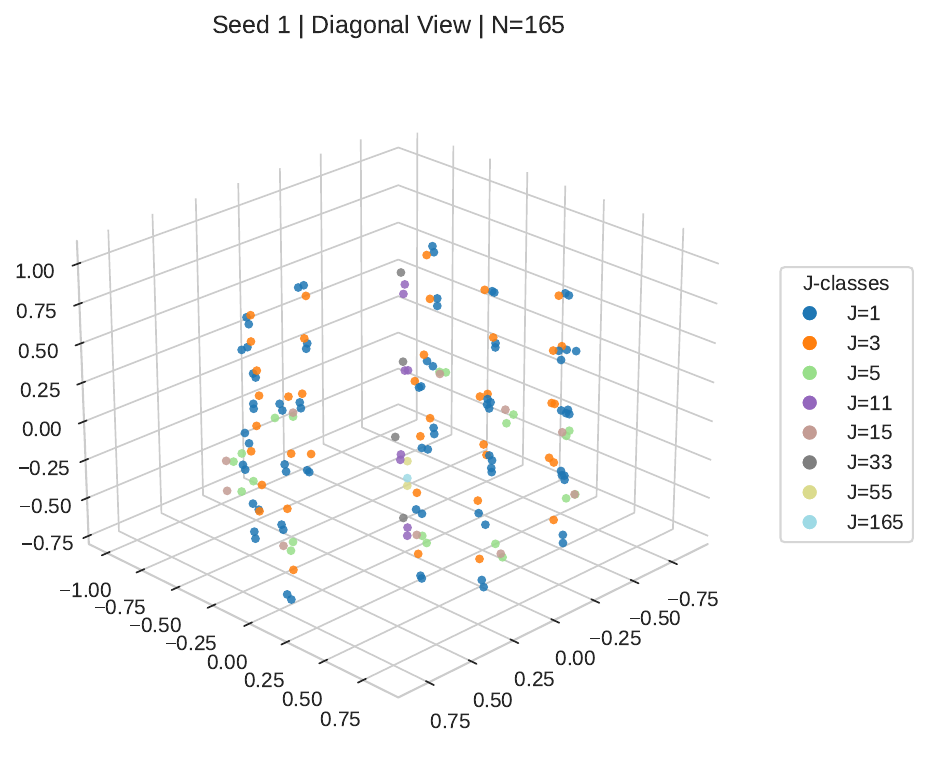} 
    \includegraphics[width=0.4\linewidth]{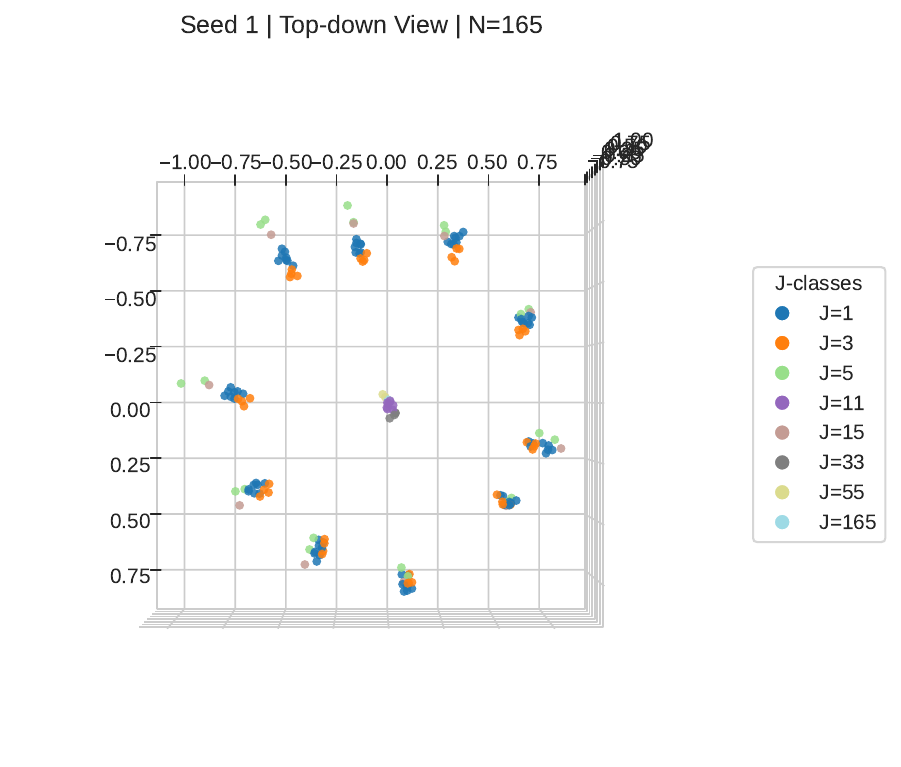} 
    \caption{\textbf{PCA visualization of token embeddings.} We project the learned token embeddings \(W_E\) onto their first three principal components and color tokens by \(\mathcal J\)-class. The diagonal and top-down views reveal structured, class-dependent geometry, with repeated curved patterns consistent with the local Fourier structure observed in the classwise spectral analysis.}
    \label{fig:embed-pca-3d}
\end{figure}

As an additional qualitative check, we visualize the learned token embeddings by projecting the embedding matrix \(W_E\) onto its first three principal components. Figure~\ref{fig:embed-pca-3d} shows two views of the resulting projection, with tokens colored by their \(\mathcal J\)-class.

The projection reveals that the embeddings are not arranged randomly in residual space. Instead, elements from related \(\mathcal J\)-classes occupy structured regions, and the global geometry exhibits repeated curved patterns consistent with the Fourier structure identified in Section~\ref{sec:embed}. In particular, the top-down view makes visible a roughly periodic arrangement, suggesting that the learned representation retains circular or toroidal geometry inherited from the local cyclic factors.

We emphasize that these PCA plots are intended as qualitative visualization rather than primary evidence. The quantitative claims in the main text are supported by the classwise Fourier energy analysis and the comparison of \(\mathcal J\)-class PCA dimension against random subsets.

\subsection{MLP Neuron Heatmaps}
\label{app:mlp_heatmaps}

Because the MLP is the network’s sole nonlinear component, it must be responsible for composing the routed operand representations into the final product representation. To verify whether this composition utilizes the same $\mathcal{J}$-class-local Fourier structure identified in the attention and embedding layers, we analyze the hidden layer activations.

For each hidden neuron $i$, we compute its post-ReLU activation on the final token residual stream, $x^{(1)}(a,b)$, across all input pairs $(a,b) \in \mathbb{Z}_{165}^2$. This yields an activation matrix $H_i \in \mathbb{R}^{165 \times 165}$. To expose latent algebraic structure, we permute the rows and columns by $\mathcal{J}$-class and then by local group coordinates, mirroring our embedding analysis.

As shown in Figure 10, the resulting heatmaps reveal two distinct structural phenomena:
\begin{itemize}
    \item \textbf{Macroscopic Stratification:} Activations exhibit strict block-like boundaries aligned exactly with $\mathcal{J}$-class partitions (red dashed lines), indicating the neurons are highly sensitive to the algebraic strata of the operands.
    \item \textbf{Microscopic Fourier Geometry:} Within these blocks, activations display fine-grained periodic bands and checkerboard textures. This mirrors the local Fourier frequencies previously observed in the embeddings and attention heads.
\end{itemize}
These results provide strong qualitative evidence that the MLP does not simply memorize raw input-output pairs. Instead, it actively preserves and operates upon the stratified, local Fourier coordinate system to compute product representations. While this confirms that algebraic structure persists through the nonlinear composition step, a complete causal mapping of these bilinear interactions via advanced dictionary learning remains an important direction for future work.

\begin{figure}[H]
    \centering
    \includegraphics[width=0.8\linewidth]{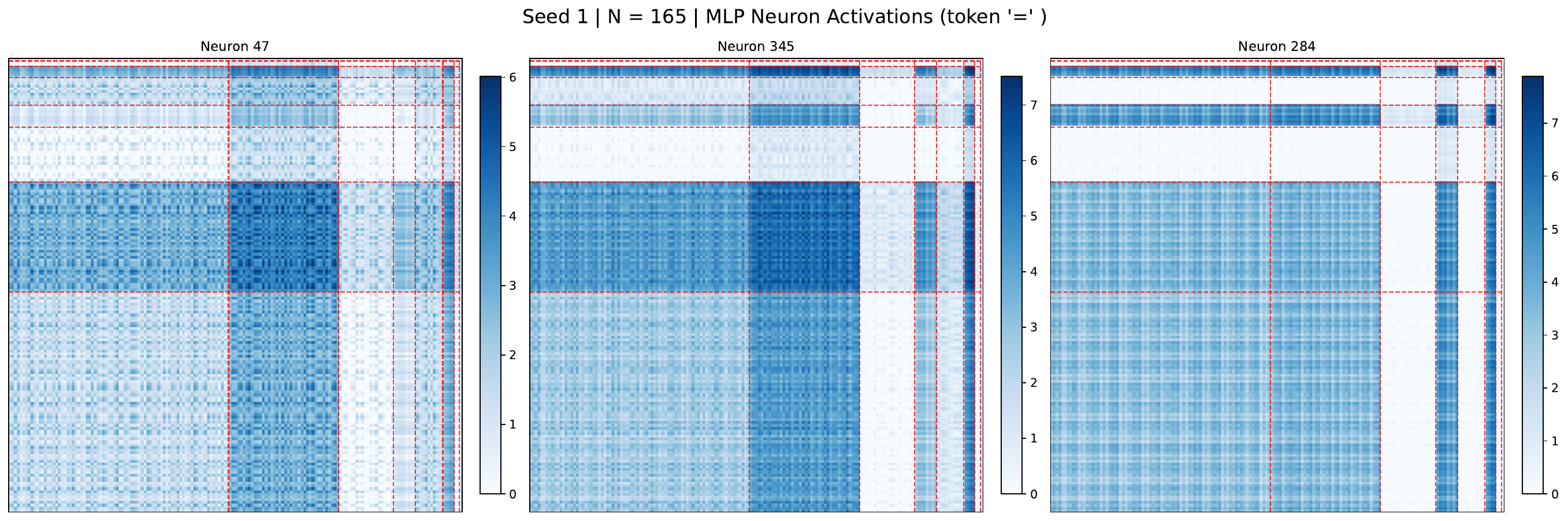} 
    \caption{\textbf{MLP hidden neuron activations exhibit macroscopic and microscopic algebraic structure.} Inputs $(a,b)$ are ordered algebraically by $\mathcal{J}$-class and local group coordinates. The strict block boundaries (red dashed lines) confirm the neurons' sensitivity to broad $\mathcal{J}$-class strata, while the internal periodic textures mirror the local Fourier frequencies observed in earlier layers. This provides qualitative evidence that the MLP actively preserves and operates upon the network's stratified coordinate system during nonlinear composition, rather than relying on rote memorization.}
    \label{fig:mlp-neurons}
\end{figure}

\section{More Mathematical Background} \label{app:monoids-primer}
\subsection{Groups and Representations}
\begin{definition}
    \label{def:group}
    A \textit{finite group} is a finite, non-empty set $G$ equipped with a binary operation $\cdot$ that satisfies the following axioms: 
    \begin{enumerate}[topsep=0pt, itemsep=0pt]
        \item Closure: if $a \in G$ and $b \in G$, then $a \cdot b \in G$. 
        \item Associativity: $a \cdot (b \cdot c) = (a \cdot b) \cdot c$ for all $a, b, c \in G$
        \item There is an element $e \in G$ (called the \textit{identity}) such that $a \cdot e = a = e \cdot a$ for every $a \in G$. 
        \item For each $a \in G$, there exists an element $d \in G$ (called the \textit{inverse} of $a$) such that $a \cdot d = e$ and $d \cdot a = e$\\

        A group is said to be \textit{abelian} if it satisfies the following axiom:
        \item Commutativity: $a \cdot b = b \cdot a$ for all $a, b \in G$. 
    \end{enumerate}  
\end{definition}
\begin{example}
    \label{ex:group}
    Consider the set of integers modulo 6: $\Z_6 := \{0, 1, 2, 3, 4, 5\}$. Under integer multiplication, this is \textit{not} a group as it does not satisfy the inverse element axiom. For instance, there does not exist an $x \in \Z_6$ such that $2 \cdot x \equiv 1 \pmod 6$.

    However, we can ``mold'' this set into a group under multiplication, simply by keeping only the invertible elements (which are also called the \textit{units}) 1 and 5. This new set now forms a group under integer multiplication, which we denote with $(\Z / 6 \Z)^\times = \{1, 5\}$.
\end{example}
\begin{definition}
    \label{def:representation}
    Let $G$ be a finite group and let $V$ be a finite-dimensional real vector space. A \textit{(real) representation} of $G$ on $V$ is a map $\rho: G \to \mathrm{GL}(V)$ satisfying
    \[
        \rho(a \cdot b) = \rho(a)\rho(b) \quad \text{for all } a, b \in G.
    \]
    In other words, $\rho$ assigns to each group element an invertible matrix in a way that respects the group operation. Note that this condition forces $\rho(e) = \mathrm{id}_V$ and $\rho(a^{-1}) = \rho(a)^{-1}$. 
    
    For each group G, there exist a finite set of fundamental \textit{irreducible representations}. 
\end{definition}

\begin{definition}
    \label{def:character}
    The \textit{character} of a representation $\rho: G \to \mathrm{GL}(V)$ is the function $\chi_\rho: G \to \R$ defined by
    \[
        \chi_\rho(a) = \tr(\rho(a)),
    \]
    the trace of the linear map $\rho(a)$. If $\rho$ is irreducible, we call $\chi_\rho$ an \textit{irreducible character}. Remarkably, the character of a representation determines it up to isomorphism, so this single scalar-valued function captures everything about $\rho$ that we care about.
\end{definition}

\subsection{Monoids and $\J$-classes}
\begin{definition}
    A \textit{finite monoid} \(M\) is a finite set equipped with an associative binary operation \(\cdot\) and an identity element.
    \begin{itemize}
        \item Associativity means that \((a \cdot b) \cdot c = a \cdot (b \cdot c)\) for all \(a,b,c \in M\).
        \item An identity element is an element \(1 \in M\) such that \(a \cdot 1 = 1 \cdot a = a\) for all \(a \in M\).
    \end{itemize}
\end{definition}
\begin{example}
    \label{ex:monoid}
    Consider the set of integers modulo 6: $\Z_6 := \{0, 1, 2, 3, 4, 5\}$. Unlike with example \ref{ex:group}, under integer multiplication, this \textit{is} a monoid, as monoids do not adhere to the invertibility axiom. We denote the monoid as $(\Z_6, \cdot)$.

    The association between monoids and multiplication tables should now be apparent:  an $n \times n$ multiplication table is simply the set of pairs $(a, b)$ together with their products $a \cdot b$ in $\Z_n$. We emphasize the definition and properties of monoids precisely because our aim is to explore the structure of the multiplication table as learned by our models.
\end{example}

\begin{definition}
    A monoid is \textit{commutative} if its binary operation is commutative; that is,
    \[
        a \cdot b = b \cdot a
    \]
    for all \(a,b \in M\).
\end{definition}
    
\begin{definition}
    A two-sided ideal $I$ of a monoid $M$ is a subset of $M$ such that $M I M \subseteq I$. In other words, for every $m, m' \in M$ and for every $i \in I$ we have $m \cdot i \cdot m' \in I$.
\end{definition}
\begin{definition}
    Two elements $m$ and $m'$ in $M$ are $\mathcal{J}$-related if they generate the same two-sided ideal; in other words: $$ m \mathcal{J} m' \quad \iff \quad MmM = Mm'M. $$
\end{definition}
\begin{definition}
    A $\mathcal{J}$-class is the equivalence class corresponding to the $\mathcal{J}$-relation.
\end{definition}
\begin{definition}
    An element $e \in M$ is \textit{idempotent} if $e^2 = e$. 
\end{definition}
\begin{definition}
    \label{def:regular-j}
    A $\mathcal{J}$-class is \textbf{regular} if it contains at least one 
    \emph{idempotent} element, i.e.\ some $e \in \mathbb{Z}_n$ with $e^2 \equiv e 
    \pmod{n}$. The $\mathcal{J}$-classes of a monoid form a perfect partition: 
    every element belongs to exactly one class.
\end{definition}
 For a regular $\mathcal{J}$-class $J$ with idempotent $e \in J$, the \textit{maximal subgroup} at $e$ is
    \[
        G_e = \{ m \in eMe \mid \exists\, m' \in eMe,\; mm' = m'm = e \},
    \]
    i.e.\ the group of units of the monoid $eMe$.

\begin{proposition}
\label{prop:squarefree-regular}
If $n$ is square-free, then every $\mathcal{J}$-class of $(\mathbb{Z}_n, \cdot)$ 
is regular.
\end{proposition}
\begin{proof}
Let $d \mid n$ be any divisor, and let $J_d = \{a \in \mathbb{Z}_n \mid \gcd(a,n) = d\}$ 
be the corresponding $\mathcal{J}$-class. We must exhibit an idempotent $e \in J_d$, 
i.e.\ an element satisfying $e^2 \equiv e \pmod{n}$ and $\gcd(e, n) = d$.

Write $n = p_1 p_2 \cdots p_k$ with $p_1, \ldots, p_k$ distinct primes (possible since 
$n$ is square-free), and write $d = \prod_{i \in S} p_i$ for some subset 
$S \subseteq \{1, \ldots, k\}$. By the Chinese Remainder Theorem,
\[
    \mathbb{Z}_n \;\cong\; \mathbb{Z}_{p_1} \times \cdots \times \mathbb{Z}_{p_k}.
\]
Define $e$ to be the unique element of $\mathbb{Z}_n$ whose image under this 
isomorphism is
\[
    e \;\longleftrightarrow\; (e_1, \ldots, e_k), \qquad 
    e_i = \begin{cases} 0 & \text{if } i \in S, \\ 1 & \text{if } i \notin S. \end{cases}
\]
Since $0^2 = 0$ and $1^2 = 1$ in each $\mathbb{Z}_{p_i}$, we have $e_i^2 = e_i$ for 
all $i$, so $e^2 \equiv e \pmod{n}$, confirming $e$ is idempotent.

It remains to verify $\gcd(e, n) = d$. For each prime $p_i$, the $p_i$-component 
of $e$ is $0$ if $i \in S$ and $1$ if $i \notin S$. Therefore $p_i \mid e$ if and 
only if $i \in S$, which gives $\gcd(e, n) = \prod_{i \in S} p_i = d$. Hence 
$e \in J_d$, so $J_d$ is regular.

Since $d$ was an arbitrary divisor of $n$, every $\mathcal{J}$-class is regular.
\end{proof}
\begin{example}
Let $n = 165 = 3 \times 5 \times 11$. Consider the finite set $M = \mathbb{Z}_{165}$ of integers modulo $165$, equipped with the multiplication operation $a \cdot b := a \times b \pmod{165}$. It is clear that this gives $(\mathbb{Z}_{165}, \cdot)$ the structure of a finite commutative monoid with identity $1$.

\textbf{The $\mathcal{J}$-classes of $(\mathbb{Z}_{165},\cdot)$.}
In this monoid, the two-sided ideal generated by an element $a$ is
\[
    \mathbb{Z}_{165}\, a\, \mathbb{Z}_{165} = a \cdot \mathbb{Z}_{165} = \{ a \cdot k \pmod{165} \mid k \in \mathbb{Z}_{165}\}.
\]
Because $\mathbb{Z}_{165}$ is commutative, this ideal is exactly the principal ideal $\langle a \rangle = \{ a \cdot k \bmod 165 \mid k \in \mathbb{Z}_{165} \}$. One can verify directly that
\[
    \langle a \rangle = \langle b \rangle \quad\iff\quad \gcd(a, 165) = \gcd(b, 165).
\]
Therefore, two elements are $\mathcal{J}$-related if and only if they share the same GCD with $165$. Since the divisors of $165$ are $\{1, 3, 5, 11, 15, 33, 55, 165\}$, the monoid decomposes into exactly eight $\mathcal{J}$-classes:
\[
    \mathbb{Z}_{165} \;=\; J_1 \sqcup J_3 \sqcup J_5 \sqcup J_{11} \sqcup J_{15} \sqcup J_{33} \sqcup J_{55} \sqcup J_{165},
\]
where $J_d = \{ a \in \mathbb{Z}_{165} \mid \gcd(a, 165) = d \}$. The sizes, idempotents, and local group structures of each class are recorded in Table~\ref{tab:jclasses-165}.

\begin{table}[h]
    \centering
    \footnotesize
    \setlength{\tabcolsep}{4pt}
    \renewcommand{\arraystretch}{1.1}
    \caption{$\mathcal{J}$-classes of $(\mathbb{Z}_{165}, \cdot)$, with sizes, idempotents, local group isomorphism types, and minimal generating sets.}
    \label{tab:jclasses-165}
    \begin{tabular}{@{}ccccc@{}}
        \toprule
        $J_d$ & $|J_d|$ & $e_{J_d}$ & $G_{e_{J_d}} \cong$ & Min. generators \\
        \midrule
        $J_1$     & $80$ & $1$   & $(\mathbb{Z}/165\mathbb{Z})^\times \cong \mathbb{Z}_2 \times \mathbb{Z}_4 \times \mathbb{Z}_{10}$ 
        & $\langle 56, 67, 46 \rangle$ \\

        $J_3$     & $40$ & $111$ & $(\mathbb{Z}/55\mathbb{Z})^\times  \cong \mathbb{Z}_4 \times \mathbb{Z}_{10}$ 
        & $\langle 12, 156 \rangle$ \\

        $J_5$     & $20$ & $100$ & $(\mathbb{Z}/33\mathbb{Z})^\times  \cong \mathbb{Z}_2 \times \mathbb{Z}_{10}$ 
        & $\langle 155, 145 \rangle$ \\

        $J_{11}$  & $8$  & $121$ & $(\mathbb{Z}/15\mathbb{Z})^\times  \cong \mathbb{Z}_2 \times \mathbb{Z}_4$ 
        & $\langle 11, 22 \rangle$ \\

        $J_{15}$  & $10$ & $45$  & $(\mathbb{Z}/11\mathbb{Z})^\times  \cong \mathbb{Z}_{10}$ 
        & $\langle 90 \rangle$ \\

        $J_{33}$  & $4$  & $66$  & $(\mathbb{Z}/5\mathbb{Z})^\times   \cong \mathbb{Z}_4$ 
        & $\langle 132 \rangle$ \\

        $J_{55}$  & $2$  & $55$  & $(\mathbb{Z}/3\mathbb{Z})^\times   \cong \mathbb{Z}_2$ 
        & $\langle 110 \rangle$ \\

        $J_{165}$ & $1$  & $0$   & trivial 
        & --- \\
        \bottomrule
    \end{tabular}
\end{table}

\noindent Every $\mathcal{J}$-class contains an idempotent (listed in the table), so every class is regular. This is a special property of square-free $n$: when $n$ has no repeated prime factor, every divisor $d \mid n$ yields a regular class $J_d$. The idempotent $e_{J_d} \in J_d$ is the unique element satisfying $e_{J_d}^2 \equiv e_{J_d} \pmod{165}$ with $\gcd(e_{J_d}, 165) = d$. Classes whose local group is cyclic admit a single generator; non-cyclic classes (those involving a direct product) do not.
\end{example}
\begin{remark}[Density of square-free integers] \label{rem:square-free}
\label{rem:squarefree-density}
The assumption that $n$ is square-free is mild: the natural density of square-free positive integers is
\[
    \lim_{N \to \infty} \frac{|\{n \leq N \mid n \text{ is square-free}\}|}{N}
    \;=\; \frac{1}{\zeta(2)} \;=\; \frac{6}{\pi^2} \;\approx\; 0.608,
\]
where $\zeta(s) = \sum_{n=1}^\infty n^{-s}$ is the Riemann zeta function \citep{apostol1976analytic}. In other words, approximately $60.8\%$ of all positive integers are square-free, so our framework applies to the majority of moduli.

\emph{Proof sketch.} An integer $n$ fails to be square-free if and only if $p^2 \mid n$ for some prime $p$. By inclusion-exclusion over primes, the density of square-free integers equals
\[
    \prod_{p \text{ prime}} \left(1 - \frac{1}{p^2}\right)
    \;=\; \frac{1}{\prod_p (1 - p^{-2})^{-1}}
    \;=\; \frac{1}{\zeta(2)}
    \;=\; \frac{6}{\pi^2},
\]
using the Euler product formula $\zeta(s) = \prod_p (1 - p^{-s})^{-1}$ and the classical result $\zeta(2) = \pi^2/6 \approx 60.8\%$ \citep{apostol1976analytic}.
\end{remark}

\subsection{The Clifford--Munn--Ponizovski\u{\i} Theorem and Local Structure}

The key structural result underlying the Monoid Extension is a classical theorem of representation theory, which we now state for the special case of commutative integer monoids.

\begin{theorem}[Clifford--Munn--Ponizovski\u{\i} \citep{steinberg2016representation}]
\label{thm:cmp}
Let $M$ be a finite monoid and $\Bbbk$ a field. There is a bijection between isomorphism classes of simple $\Bbbk M$-modules and isomorphism classes of simple $\Bbbk G_e$-modules, taken one per regular $\mathcal{J}$-class, where $G_e = eMe$ is the maximal subgroup at the idempotent $e \in J$.

In particular, every irreducible representation of $M$ is indexed by a pair $(J, V)$ where $J$ is a regular $\mathcal{J}$-class and $V$ is an irreducible representation of the corresponding maximal subgroup $G_e$.
\end{theorem}

This theorem has a profound consequence for our setting: the representation theory of the full monoid $(\mathbb{Z}_n, \cdot)$ \emph{reduces} to the representation theory of its maximal subgroups, one per regular $\mathcal{J}$-class. In other words, understanding the irreducible representations of $\mathbb{Z}_n$ under multiplication is equivalent to understanding the irreducible representations of each group $G_{e_{J_d}}$. For square-free $n$, all $\mathcal{J}$-classes are regular, so this reduction is total.

\begin{theorem}
\label{thm:jclass-group}
In the multiplication monoid $(\mathbb{Z}_n, \cdot)$, every regular $\mathcal{J}$-class $J_d$ forms a group under the inherited multiplication with local identity $e_{J_d}$, and
\[
    J_d \;\cong\; \left(\mathbb{Z}/(n/d)\mathbb{Z}\right)^\times.
\]
When $n$ is square-free, all $\mathcal{J}$-classes are regular and this isomorphism holds for every divisor $d \mid n$.
\end{theorem}

\begin{proof}
Let $d \mid n$ and write $n = dm$ where $m = n/d$. The map
\[
\phi : J_d \to (\mathbb Z/m\mathbb Z)^\times,
\qquad
a \mapsto a \bmod m.
\]
is well-defined because $a \in J_d$ implies $\gcd(a, n) = d$, so $a = d \cdot a'$ with $\gcd(a', m) = 1$, i.e.\ $a' \in (\mathbb{Z}/m\mathbb{Z})^\times$. One verifies that $\phi$ is a bijection preserving the group operation (multiplication modulo $n$ on $J_d$ corresponds to multiplication modulo $m$ on $(\mathbb{Z}/m\mathbb{Z})^\times$), and that the local idempotent $e_{J_d}$ maps to the identity $1 \in (\mathbb{Z}/m\mathbb{Z})^\times$.

For square-free $n$, the prime factorization $n = p_1 \cdots p_k$ ensures that every divisor $d \mid n$ is also square-free, so $\gcd(d, n/d) \mid d$ divides a product of distinct primes, guaranteeing that the relevant class $J_d$ contains an idempotent (which can be constructed explicitly by the Chinese Remainder Theorem).
\end{proof}

\begin{remark}
By the Chinese Remainder Theorem, for square-free $n = p_1 \cdots p_k$,
\[
    (\mathbb{Z}/m\mathbb{Z})^\times \;\cong\; \prod_{p_i \nmid d} (\mathbb{Z}/p_i\mathbb{Z})^\times \;\cong\; \prod_{p_i \nmid d} \mathbb{Z}_{p_i - 1}.
\]
This direct product structure is what drives the ``atomic factorization'' of Fourier features observed empirically in Section~\ref{sec:embed}: the network learns the cyclic factors $(\mathbb{Z}/p_i\mathbb{Z})^\times$ independently and recombines them to represent composite $\mathcal{J}$-classes.
\end{remark}
\begin{theorem}[Chughtai et al., 2023]
\label{thm:chughtai}
    Let $G$ be a group of finite order and $\rho: G \to GL(\R^d)$ a real representation of dimension $d$. For any $g \in G$, $\chi_{\rho}(g) \leq d$ with equality if and only if $\rho(g) = I_d$.
\end{theorem}

\begin{proposition}
\label{prop:logit-max}
Let $J_d \cong (\mathbb{Z}/(n/d)\mathbb{Z})^\times$ be a regular $\mathcal{J}$-class with local idempotent $e_{J_d}$, and let $\rho$ be a direct sum of \emph{all} irreducible representations of $J_d$. Then for any $x \in J_d$, $\chi_\rho(x) \;\leq\; \chi_\rho(e_{J_d})$ 
with equality if and only if $x = e_{J_d}$.
\end{proposition}
\begin{proof}
By Theorem~\ref{thm:jclass-group}, $J_d$ is a finite abelian group with identity 
$e_{J_d}$. Since the character of a group representation is maximized at the identity element, as a corollary of Theorem~\ref{thm:chughtai}, we can infer that $\chi_\rho(x) \;\leq\; \chi_\rho(e_{J_d})$ with equality if and only if $\rho(x) = I$.

To show that $x=e_{J_d}$ is the unique element for which $\rho(x)=I$, we must show that $\rho$ is faithful, or equivalently that the intersection of the kernels of the frequency representations used in $\rho$ is trivial. By the Fundamental Theorem of Finitely Generated Abelian Groups,
\[
J_d \cong \mathbb{Z}_{q_1}\times \mathbb{Z}_{q_2}\times \cdots \times \mathbb{Z}_{q_m}.
\]
Thus any $x\in J_d$ can be uniquely represented by a coordinate tuple
\[
(p_1(x),\dots,p_m(x)), \qquad p_j(x)\in \{0,1,\dots,q_j-1\},
\]
with the identity $e_{J_d}$ corresponding to the zero vector.

Over $\mathbb{C}$, the irreducible characters of $J_d$ are indexed by frequency vectors
$k=(k_1,\dots,k_m)$ and take the form
\[
\chi_k(x)=\exp\left(2\pi i\sum_{j=1}^m \frac{k_j p_j(x)}{q_j}\right).
\]
Over $\mathbb{R}$, a real-valued order-two character gives a one-dimensional sign representation; this includes the coordinate characters coming from factors with $q_j=2$. A non-real character realifies, together with its complex conjugate, to the two-dimensional rotation block
\[
\rho_k(x)=
\begin{pmatrix}
\cos(\theta_k(x)) & -\sin(\theta_k(x)) \\
\sin(\theta_k(x)) & \cos(\theta_k(x))
\end{pmatrix},
\qquad
\theta_k(x)=2\pi\sum_{j=1}^m \frac{k_j p_j(x)}{q_j}.
\]
If $\rho=\bigoplus_{k\in K}\rho_k$ is the direct sum of the frequency representations used by the readout, then $\rho(x)=I$ if and only if every selected frequency is trivial on $x$, i.e.
\[
x\in \bigcap_{k\in K}\ker(\rho_k).
\]
Therefore $\rho$ is faithful exactly when
\[
\bigcap_{k\in K}\ker(\rho_k)=\{e_{J_d}\}.
\]
Under this faithfulness condition, $x=e_{J_d}$ is the unique group element satisfying $\rho(x)=I$.

Since $\rho = \bigoplus_k\rho_k$ is a direct sum of all such irreducible representations, $\rho(x) = I$ implies that every single diagonal block must be the identity. In other words, for any block $k$, we must have $\theta_{{k}}(x) = 0 \mod{2\pi}$. Let $j \in \{1, ..., m\}$ and let $k$ be the frequency vector where $k_j = 1$ and all other entries are 0. Since the previous identity holds for \emph{all} diagonal blocks $k$, we have now $\theta_{{k}}(x) = 2\pi \frac{p_j(x)}{q_j} = 0 \mod 2\pi$, which directly implies $\frac{p_j(x)}{q_j} \in \Z$; in other words, $p_j(x)$ is a multiple of $q_j$. Recall that $p_j(x) \in \Z_{q_j}$, meaning the only possible solution is $p_j(x) = 0$. 

Because this is true for every coordinate $j$, we have that $x = (0,\ldots, 0)$. This uniquely identifies $x$ as the identity $e_{J_d}$.

\end{proof}
\begin{remark}
    \label{rem:key-freqs}
    Since we only require $\rho$ to be faithful in the proof above, $\rho$ does not need to be a direct sum of \emph{every} irrep for \ref{prop:logit-max} to hold; it only needs to be a direct sum $\rho = \bigoplus_{i \in S} \rho _i$ over any set $S$ where $\bigcap_{i \in S} \ker \rho_i = \{e_{J_d}\}$. 

    This is a crucial fact that factors into how models choose subsets of key frequencies, creating subspaces within the entire input space.
\end{remark}
\textbf{Connection to GCR and the Monoid Extension.} The Clifford--Munn--Ponizovski\u{\i} theorem is precisely the algebraic justification for the Monoid Extension proposed in Section~\ref{sec:monoid-ext}. The GCR algorithm of \citet{chughtai2023toy} operates exclusively within the group of units $J_1 = (\mathbb{Z}/n\mathbb{Z})^\times$, which is a single regular $\mathcal{J}$-class. The CMP theorem tells us that the \emph{entire} representation theory of the monoid decomposes as a direct sum of group representations, one per regular $\mathcal{J}$-class. Our Monoid Extension operationalizes this: the network routes each computation into the appropriate $\mathcal{J}$-class and then applies GCR-style character-based scoring \emph{within} that class using the local group structure guaranteed by Theorem~\ref{thm:jclass-group}.

In the group-only case studied by \citet{chughtai2023toy}, there is a single $\mathcal{J}$-class ($J_1$ itself), and CMP reduces to the classical representation theory of finite abelian groups. The Monoid Extension is therefore a strict generalization: it applies CMP class by class across all regular $\mathcal{J}$-classes of the monoid, replacing the single global inverse $c^{-1}$ of GCR with a local inverse $c^\sharp$ defined within each class. Along with Proposition \ref{prop:logit-max} and Remark \ref{rem:key-freqs}, the model is theoretically able to uniquely single out each element through local inverses and unembedding logic as outlined in GCR.
\section{Broader Mechanistic Interpretability}
To fully understand how neural networks reason, mechanistic interpretability must expand its focus beyond globally invertible operations. While our analysis isolates non-invertible structures within a mathematical toy model, this framework offers critical insights into the fundamentally non-invertible nature of real-world language modeling.

Most mechanistic analyses of algorithmic tasks have focused on invertible or group-like structure, such as modular addition and finite group composition. This makes the representation-theoretic mechanism especially clean: candidate outputs can be scored by comparing against a global inverse. In contrast, many computations in language models are not naturally invertible. Sequence processing, retrieval, and compression into the residual stream often require the model to preserve some information while discarding or routing other information.

A related body of LLM interpretability work studies sequence-level circuits such as induction heads, which detect and continue repeated subsequences. These mechanisms show that transformers learn structured algorithms over token sequences, but they are typically not framed as algebraic monoid computations. Our modular multiplication setting provides a finite testbed for one aspect of this broader problem: how a transformer organizes computation when global inverse-based decoding is unavailable. The resulting $\mathcal J$-class stratification should not be viewed as a direct model of natural language, but as a controlled example of how representation-theoretic mechanisms can localize to algebraic substructures under non-invertible composition.
\section{Other Moduli and Stability Analysis}

We evaluate the stability of the learned Fourier structure and embedding geometry across random initializations and moduli $n \in \{113, 143, 154, 165\}$. For each setting, we select representative seeds to visualize variation in spectral structure and PCA geometry.

This section evaluates whether the structural phenomena identified in Section~\ref{sec:monoid-ext} persist beyond a single trained model, and whether they are stable under changes in initialization.

Specifically, we examine the consistency of four aspects of representation:

(i) \textbf{Embedding geometry}, through Fourier analysis and PCA-based fraction-of-variance-explained (FVE) measurements, to assess whether $\mathcal{J}$-class subspaces and their spectral decompositions are preserved across seeds and moduli.

(ii) \textbf{Torus and CRT structure}, through unrolled embeddings aligned with  $\mathcal{J}$-class orderings, to test whether the implicit cyclic and product structure of $\mathbb{Z}_n$ is consistently recovered.

(iii) \textbf{Attention routing}, by analyzing head-wise attention patterns and their alignment with $\mathcal{J}$-class partitions, to determine whether the routing mechanism exhibits stable block-structure and frequency-sensitive behavior.

(iv) \textbf{MLP feature formation}, by examining neuron-wise Fourier spectra and activation clustering, to verify whether intermediate nonlinear representations consistently encode the same multiplicative substructures across seeds.

Together, these analyses provide a multi-scale stability test of the monoid extension hypothesis, allowing us to distinguish incidental structure from reproducible algorithmic behavior.

\subsection{Embedding Dimensionality and Fourier Structure}
We first look at the PCA dimensionality required to explain 95\% of the variance across various values of $n$, averaged across runs on multiple seeds. We see across variations of our experiment, the relatively low dimensionality of our embedding matrix remains consistent.

\begin{table}[h]
\centering
\caption{
PCA dimensionality required to explain 95\% of variance across 10 random seeds.
Values are reported as mean $\pm$ std over seeds, with min--max range included for stability characterization.
Final FVE is averaged across seeds.
}
\begin{tabular}{c|c|c|c|c}
\hline
\textbf{n} & \textbf{Mean PCs} & \textbf{Std} & \textbf{Min--Max PCs} & \textbf{Mean FVE (\%)} \\
\hline
165 & 10.0 & 0.6 & 9 -- 11 & 96.16 \\
143 & 10.9 & 1.6 & 8 -- 14 & 95.89 \\
113 & 9.6  & 1.7 & 7 -- 13 & 95.90 \\
154 & 10.1 & 0.9 & 9 -- 12 & 95.99 \\
\hline
\end{tabular}
\label{tab:pca_variance_summary}
\end{table}

\newpage 

\begin{figure*}[!t]
    \centering

    \textbf{$N = 165$}\\[0.2cm]

    \begin{subfigure}[t]{0.24\textwidth}
        \centering
        \includegraphics[width=\linewidth]{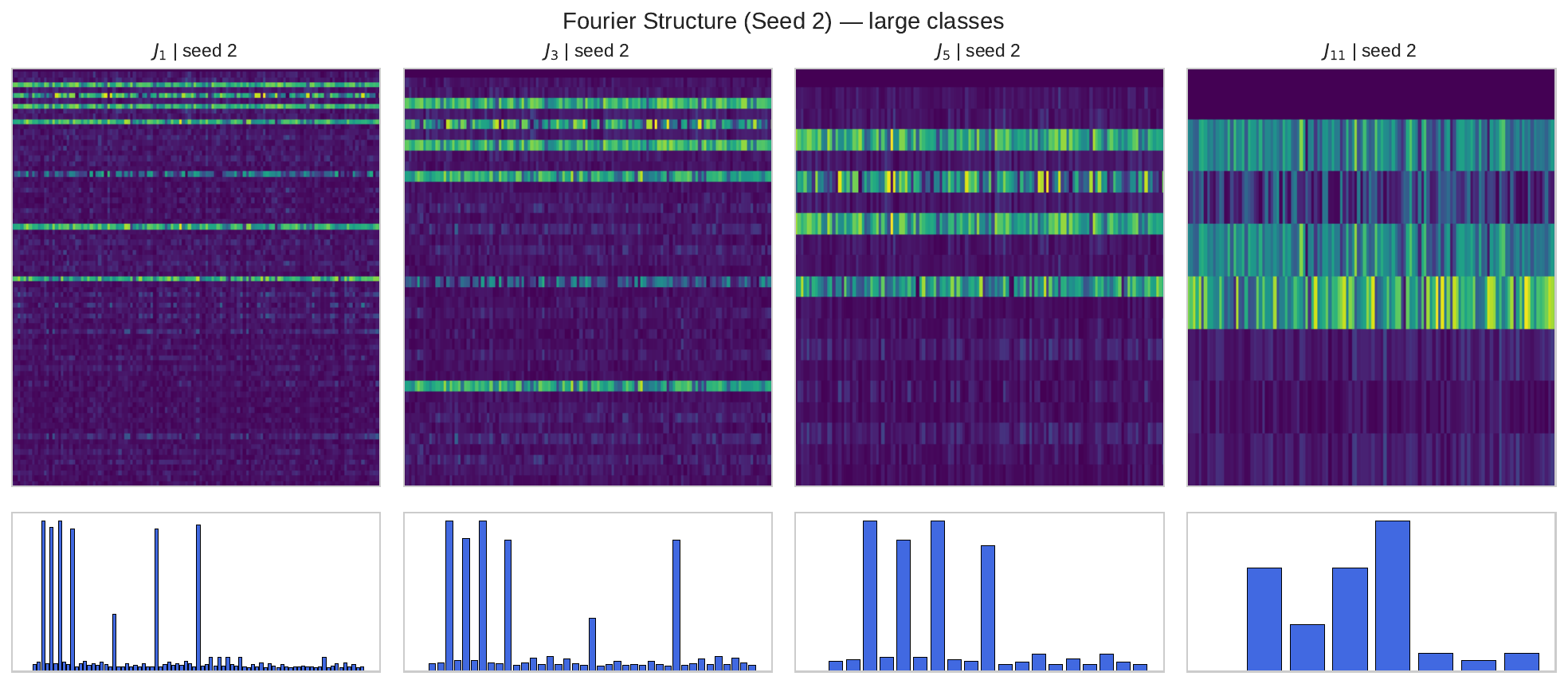}
        \caption{Seed 2}
    \end{subfigure}
    \hfill
    \begin{subfigure}[t]{0.24\textwidth}
        \centering
        \includegraphics[width=\linewidth]{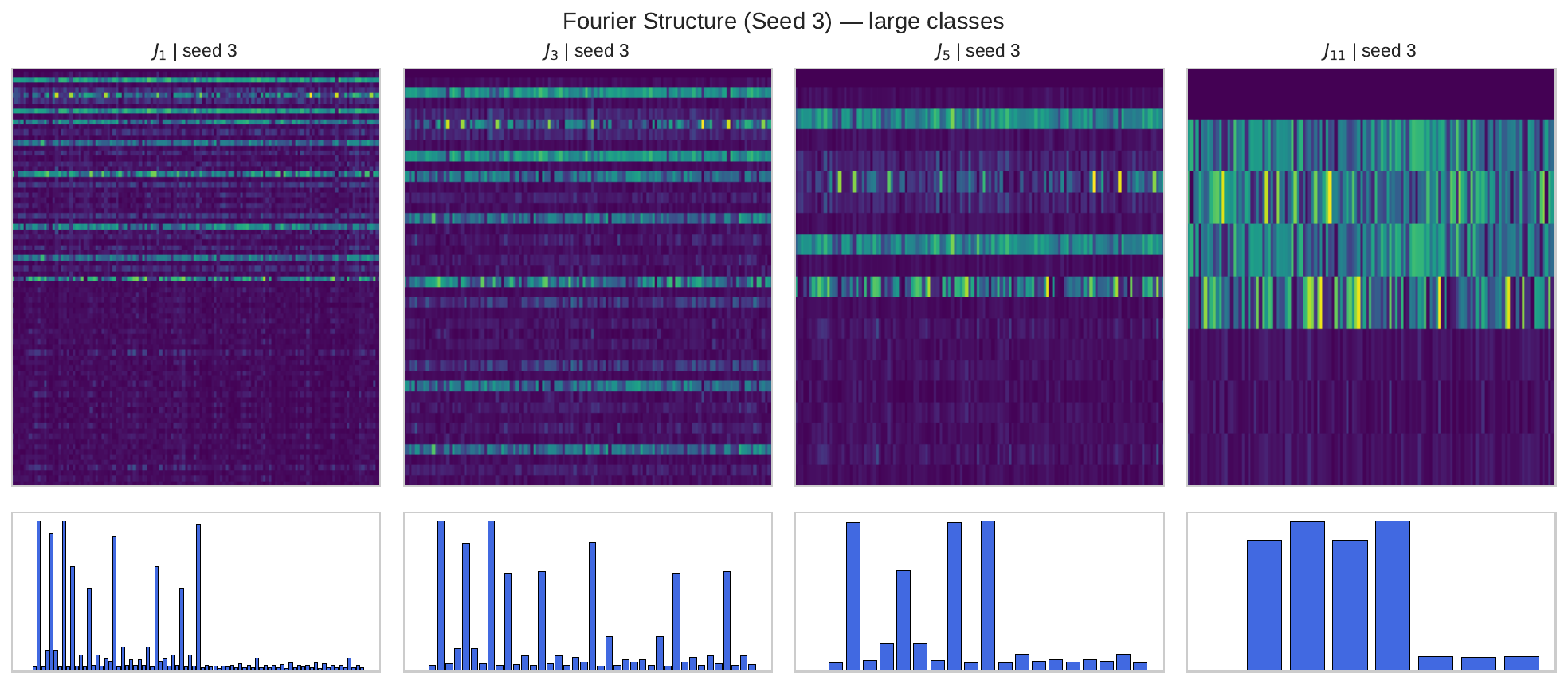}
        \caption{Seed 3}
    \end{subfigure}
    \hfill
    \begin{subfigure}[t]{0.24\textwidth}
        \centering
        \includegraphics[width=\linewidth]{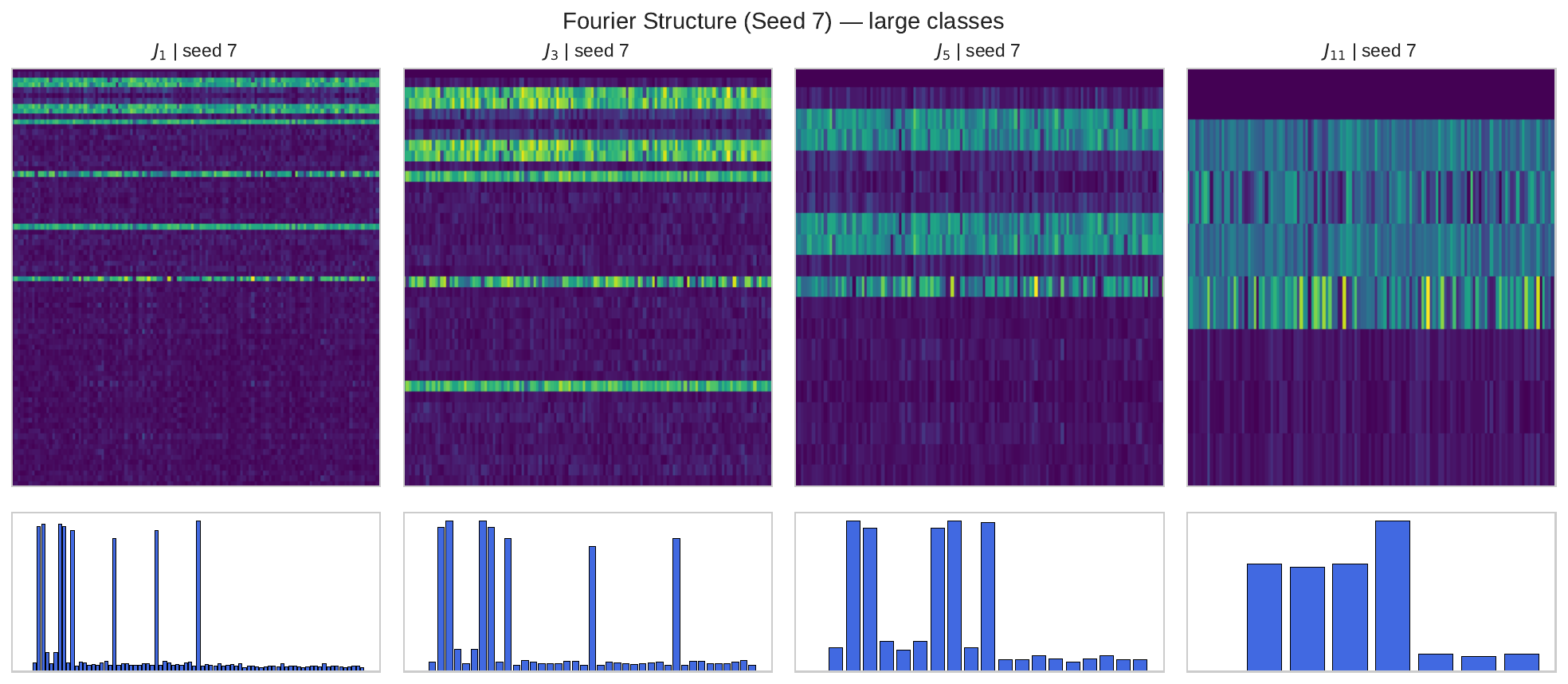}
        \caption{Seed 7}
    \end{subfigure}
    \hfill
    \begin{subfigure}[t]{0.24\textwidth}
        \centering
        \includegraphics[width=\linewidth]{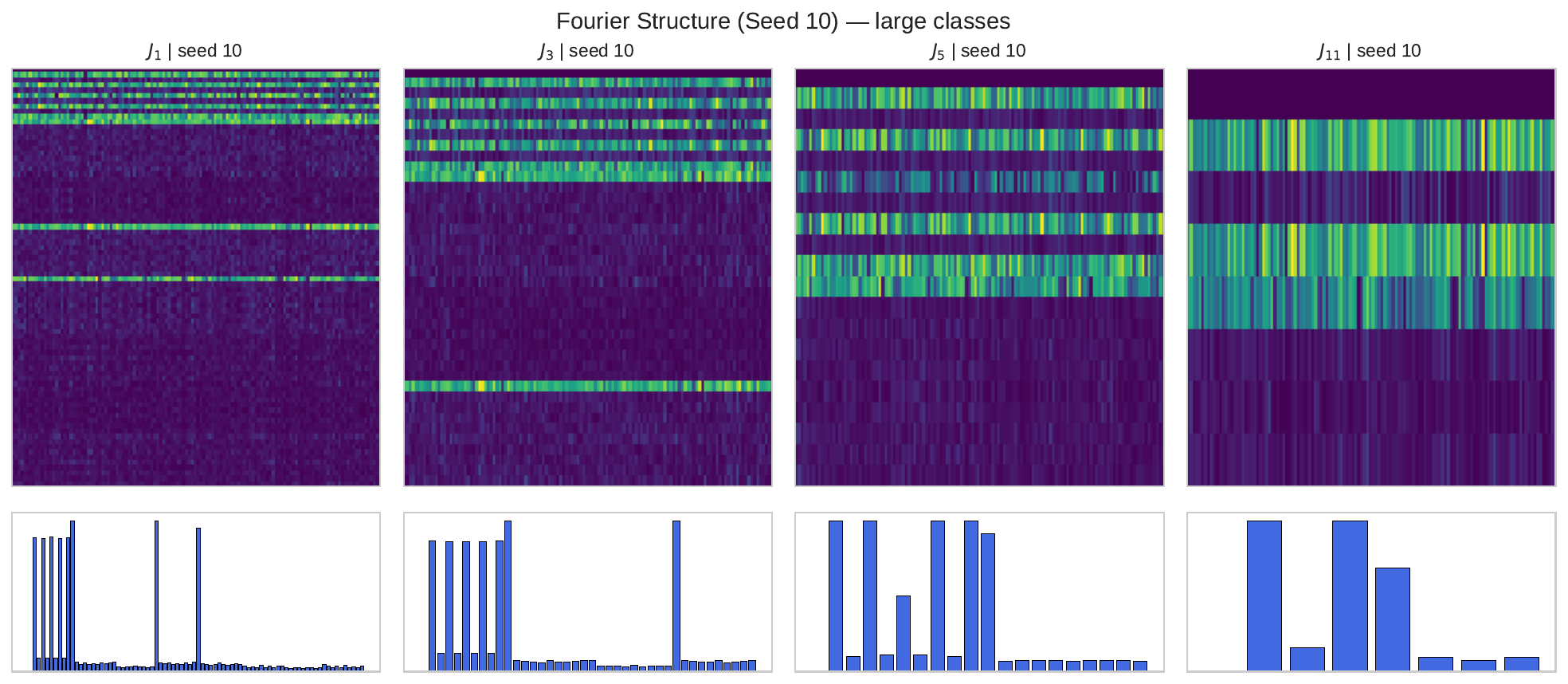}
        \caption{Seed 10}
    \end{subfigure}

    \vspace{0.4cm}

    \textbf{$N = 154$}\\[0.2cm]

    \begin{subfigure}[t]{0.24\textwidth}
        \centering
        \includegraphics[width=\linewidth]{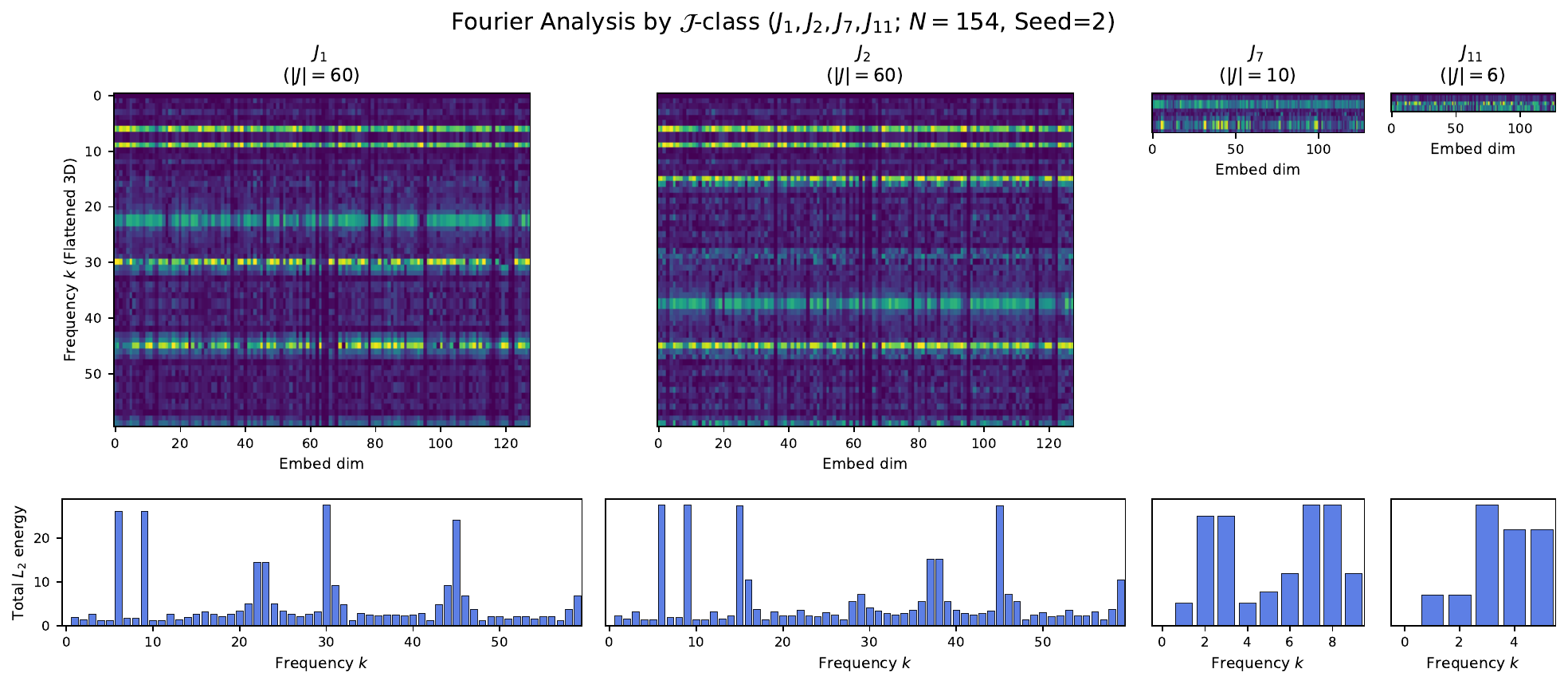}
        \caption{Seed 2}
    \end{subfigure}
    \hfill
    \begin{subfigure}[t]{0.24\textwidth}
        \centering
        \includegraphics[width=\linewidth]{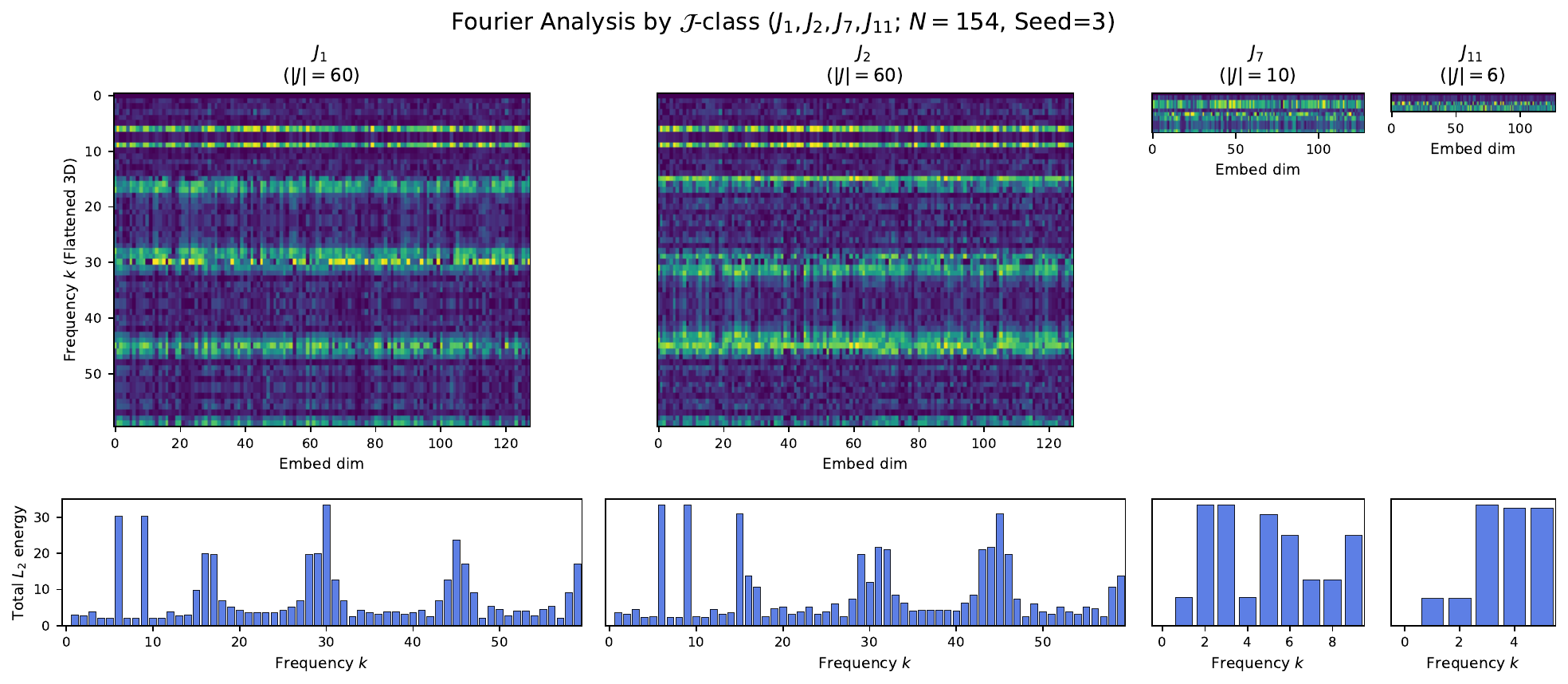}
        \caption{Seed 3}
    \end{subfigure}
    \hfill
    \begin{subfigure}[t]{0.24\textwidth}
        \centering
        \includegraphics[width=\linewidth]{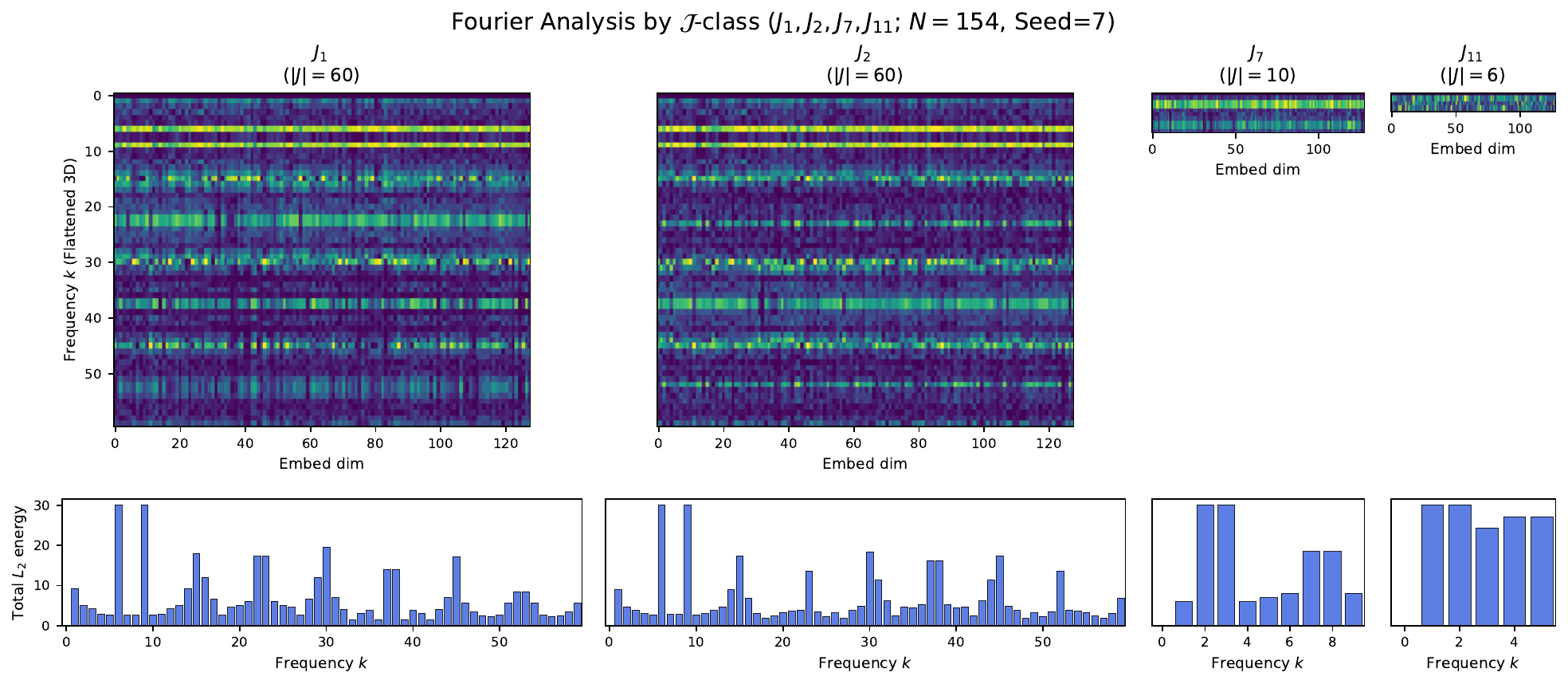}
        \caption{Seed 7}
    \end{subfigure}
    \hfill
    \begin{subfigure}[t]{0.24\textwidth}
        \centering
        \includegraphics[width=\linewidth]{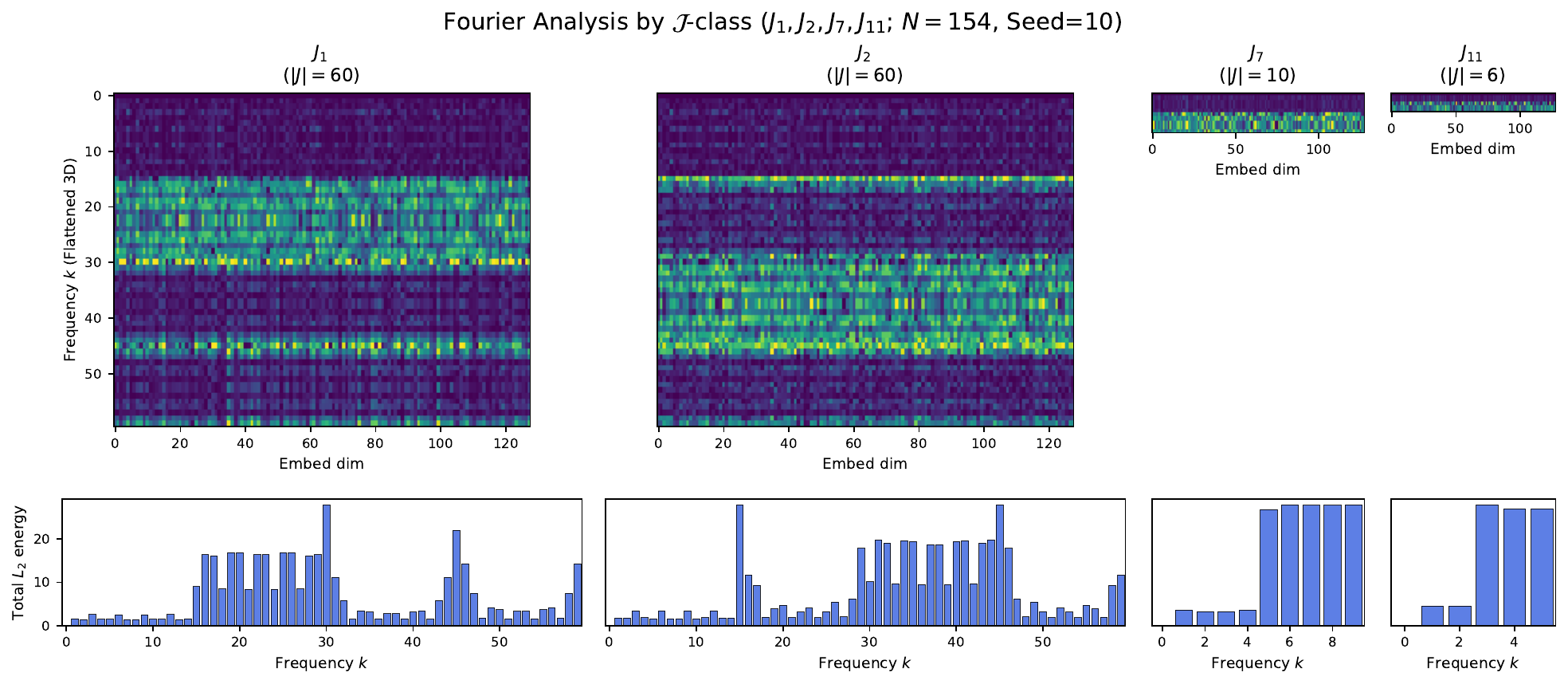}
        \caption{Seed 10}
    \end{subfigure}

    \vspace{0.4cm}

    \textbf{$N = 143$}\\[0.2cm]

    \begin{subfigure}[t]{0.24\textwidth}
        \centering
        \includegraphics[width=\linewidth]{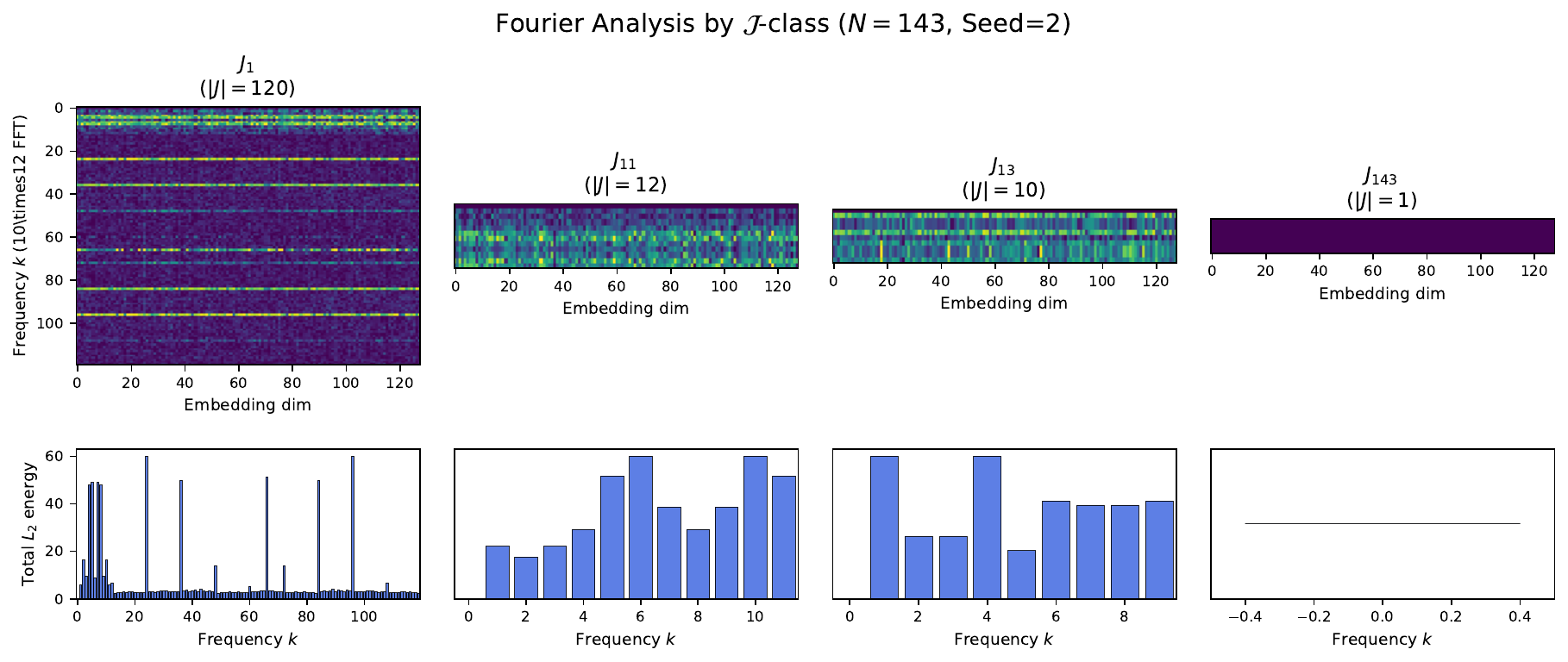}
        \caption{Seed 2}
    \end{subfigure}
    \hfill
    \begin{subfigure}[t]{0.24\textwidth}
        \centering
        \includegraphics[width=\linewidth]{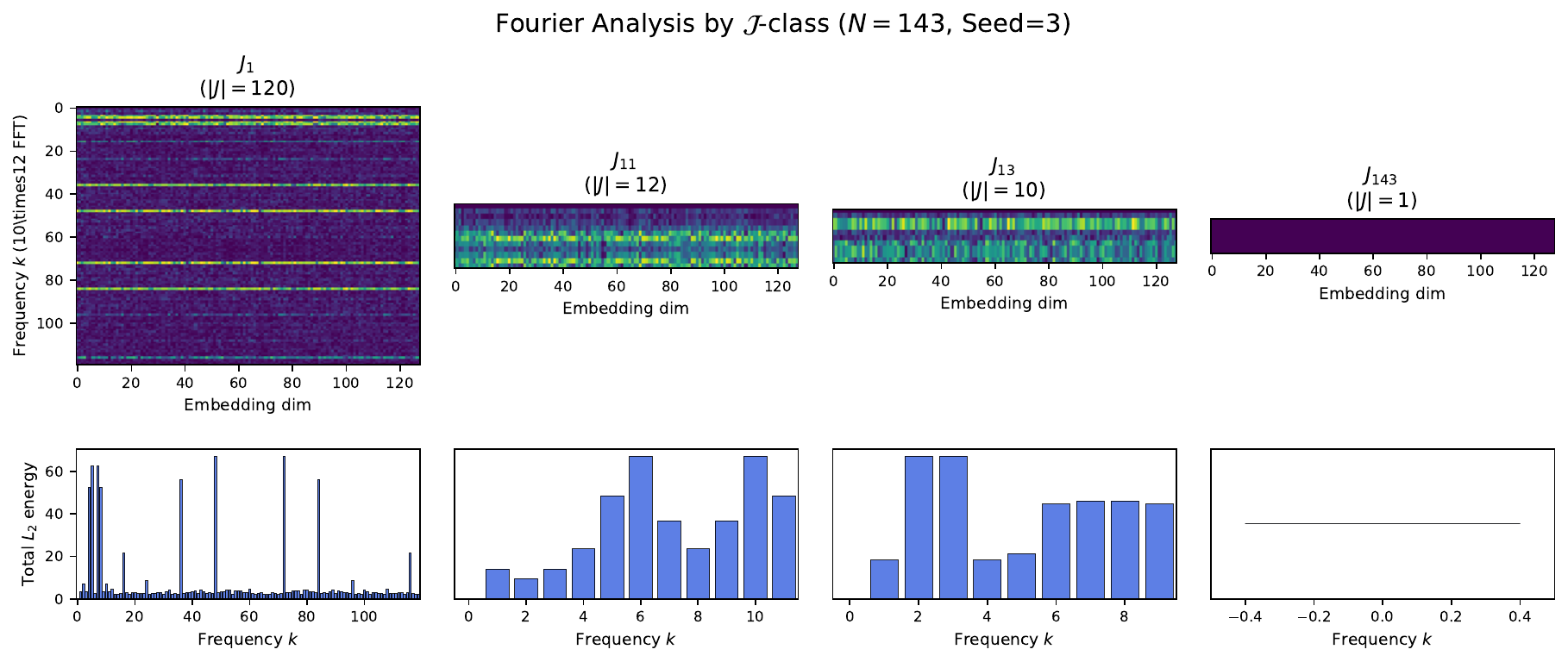}
        \caption{Seed 3}
    \end{subfigure}
    \hfill
    \begin{subfigure}[t]{0.24\textwidth}
        \centering
        \includegraphics[width=\linewidth]{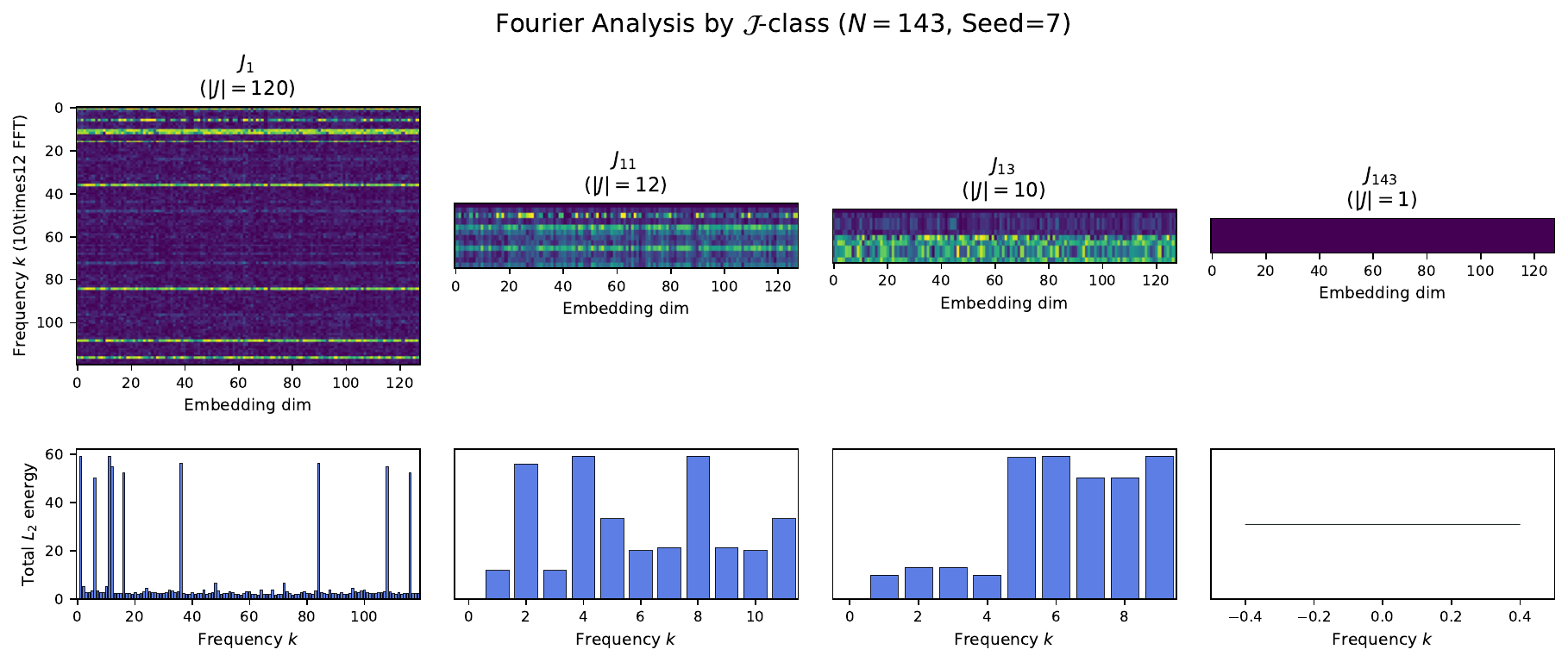}
        \caption{Seed 7}
    \end{subfigure}
    \hfill
    \begin{subfigure}[t]{0.24\textwidth}
        \centering
        \includegraphics[width=\linewidth]{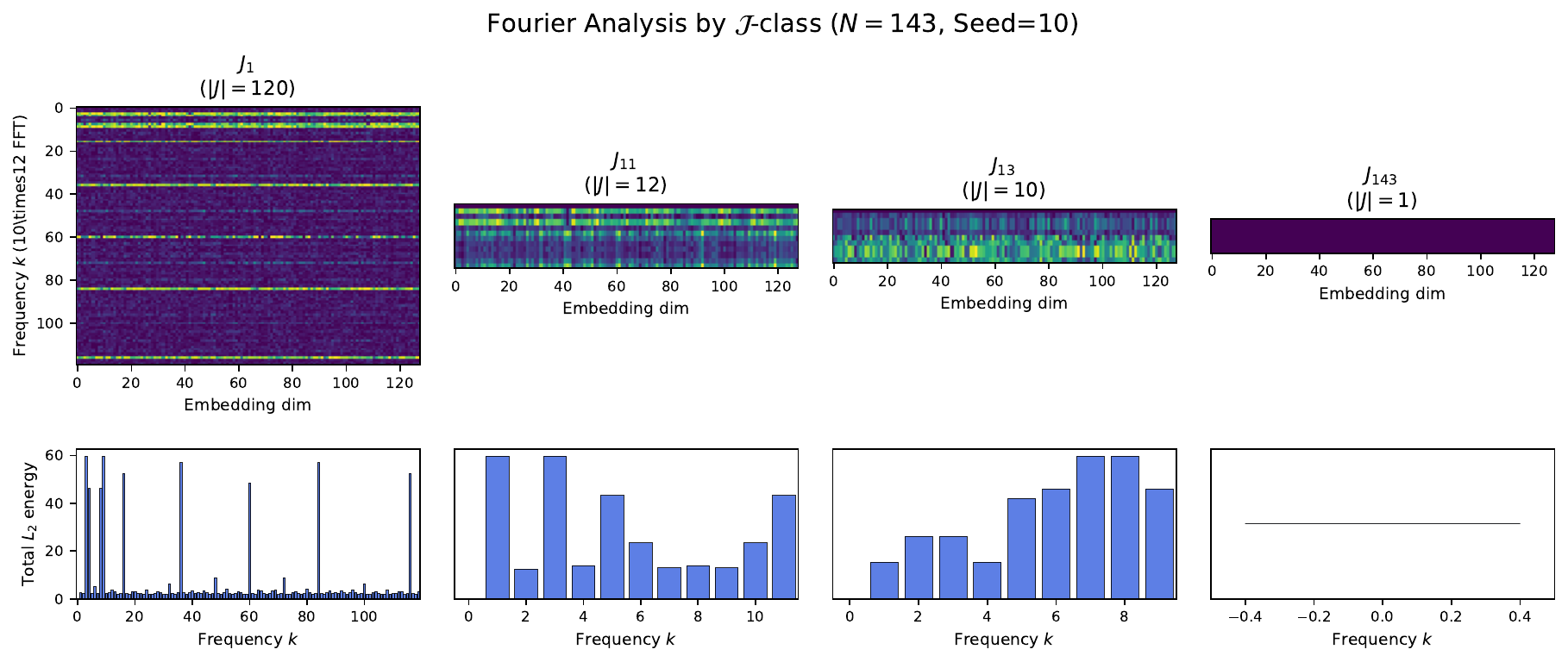}
        \caption{Seed 10}
    \end{subfigure}

    \vspace{0.4cm}

    \textbf{$N = 113$}\\[0.2cm]

    \begin{subfigure}[t]{0.24\textwidth}
        \centering
        \includegraphics[width=\linewidth]{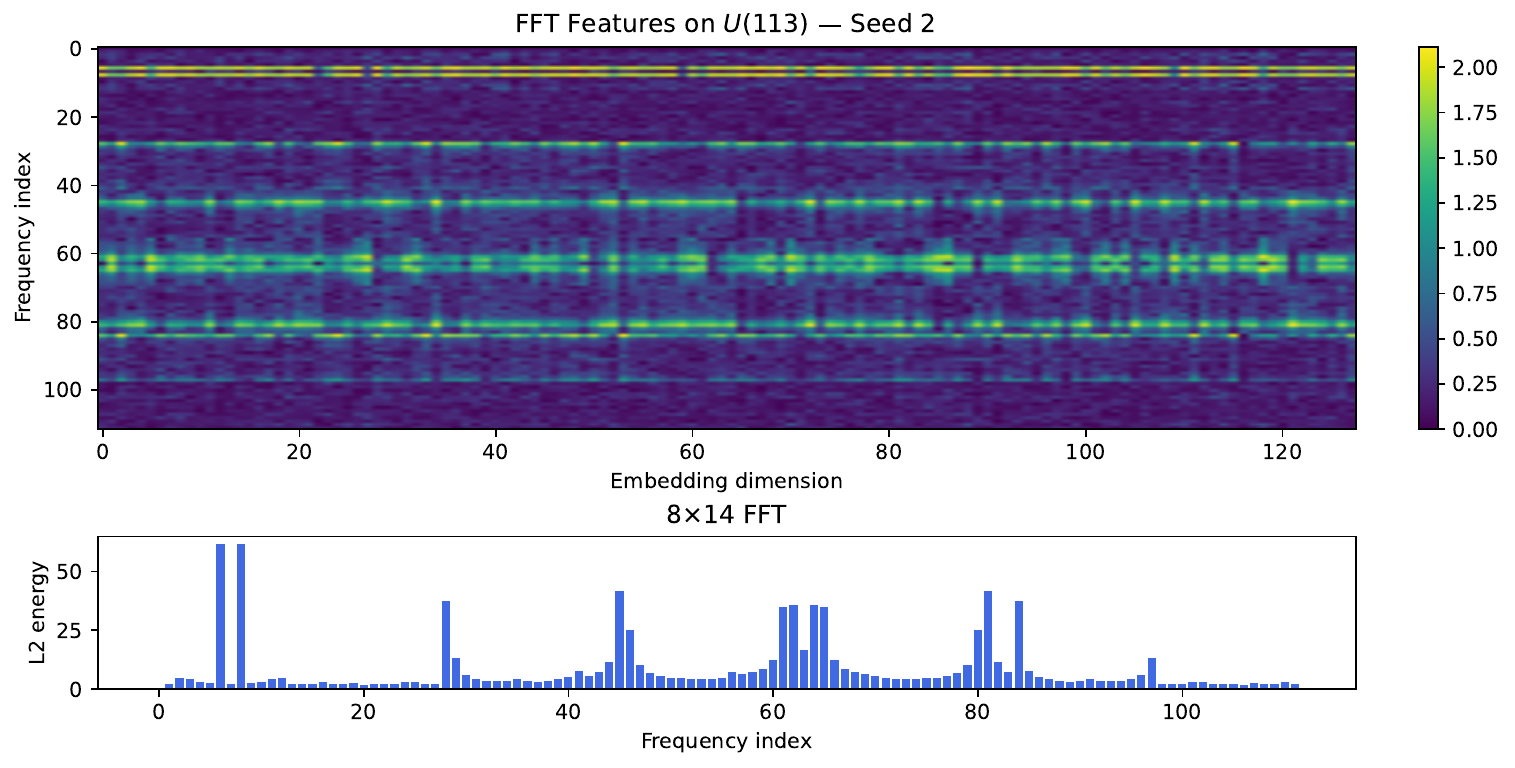}
        \caption{Seed 2}
    \end{subfigure}
    \hfill
    \begin{subfigure}[t]{0.24\textwidth}
        \centering
        \includegraphics[width=\linewidth]{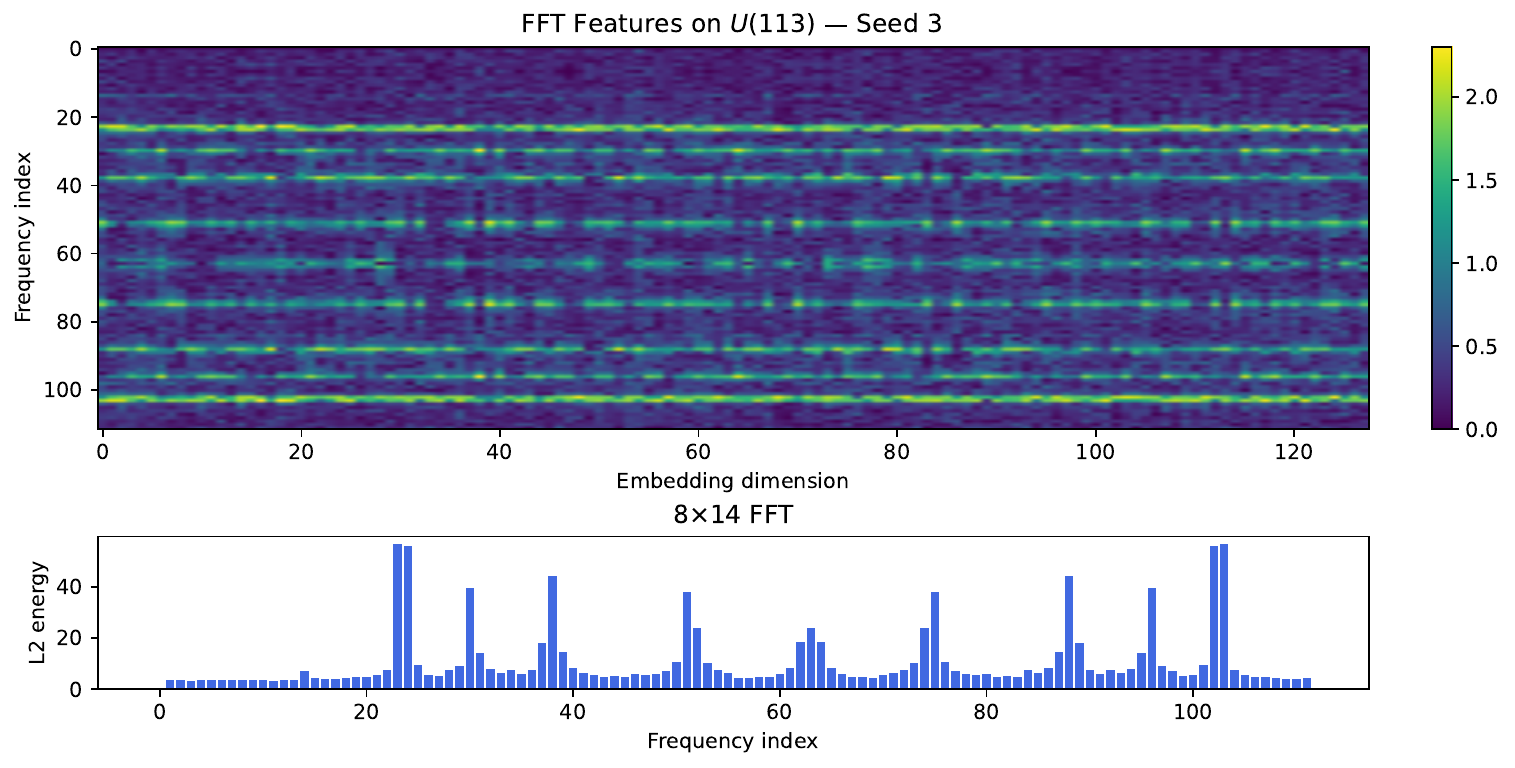}
        \caption{Seed 3}
    \end{subfigure}
    \hfill
    \begin{subfigure}[t]{0.24\textwidth}
        \centering
        \includegraphics[width=\linewidth]{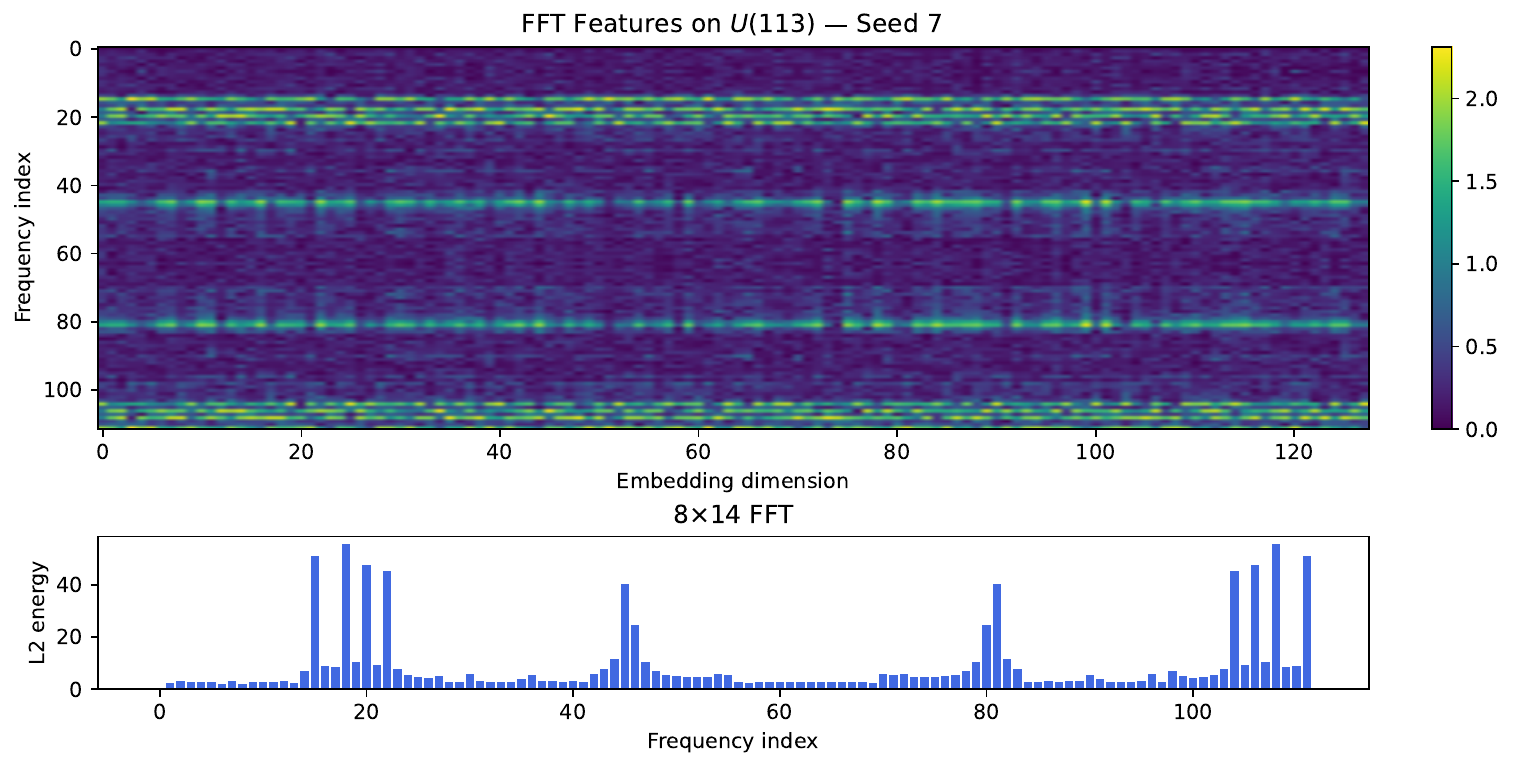}
        \caption{Seed 7}
    \end{subfigure}
    \hfill
    \begin{subfigure}[t]{0.24\textwidth}
        \centering
        \includegraphics[width=\linewidth]{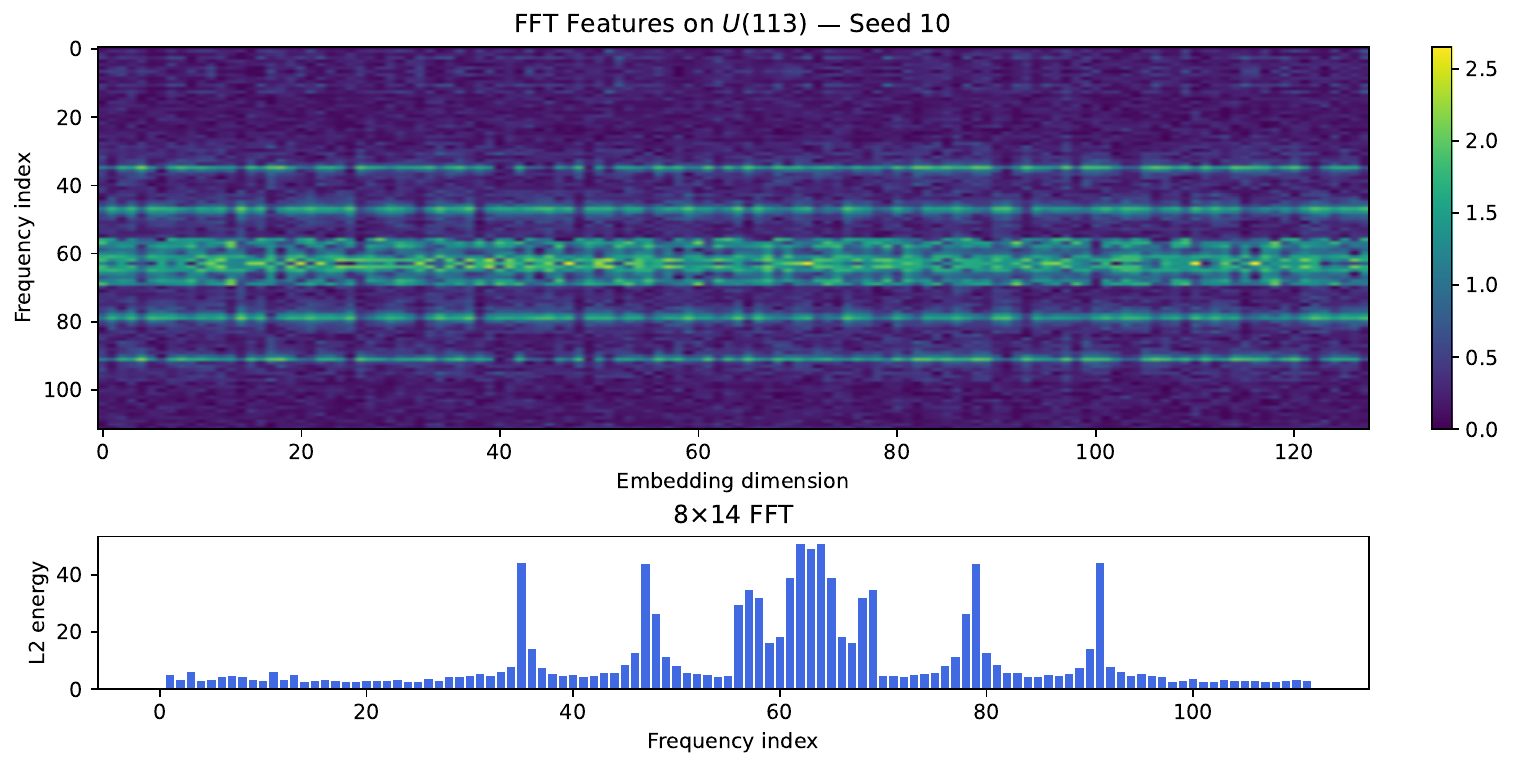}
        \caption{Seed 10}
    \end{subfigure}

    \caption{
    \textbf{Fourier stability across seeds and moduli.}
    We show $\mathcal{J}$-aligned Fourier spectra of embedding matrices across representative seeds for each modulus $n$. While the precise locations of dominant frequency peaks vary across random initializations, the overall block-structured spectral organization remains consistent, indicating that the learned representation is stable under stochastic optimization.
    }

    \label{fig:fft_stability_clean}
\end{figure*}

\subsection{Torus and CRT Structure of Embeddings}

We then explore the 3D projection of principal components of the embedding matrix. We see that the embedding organizes itself into clusters, grouped by $\mathcal{J}$-class. We observe a torus structure in the high dimensional space the the model operates in. Moreover, we see that across various random initializations and moduli, we consistently observe these toroidal patterns in the projection.

\newpage
\begin{figure*}[!t]
    \centering

    \textbf{$N = 165$}\\[0.2cm]

    \begin{subfigure}[t]{0.24\textwidth}
        \centering
        \includegraphics[width=\linewidth]{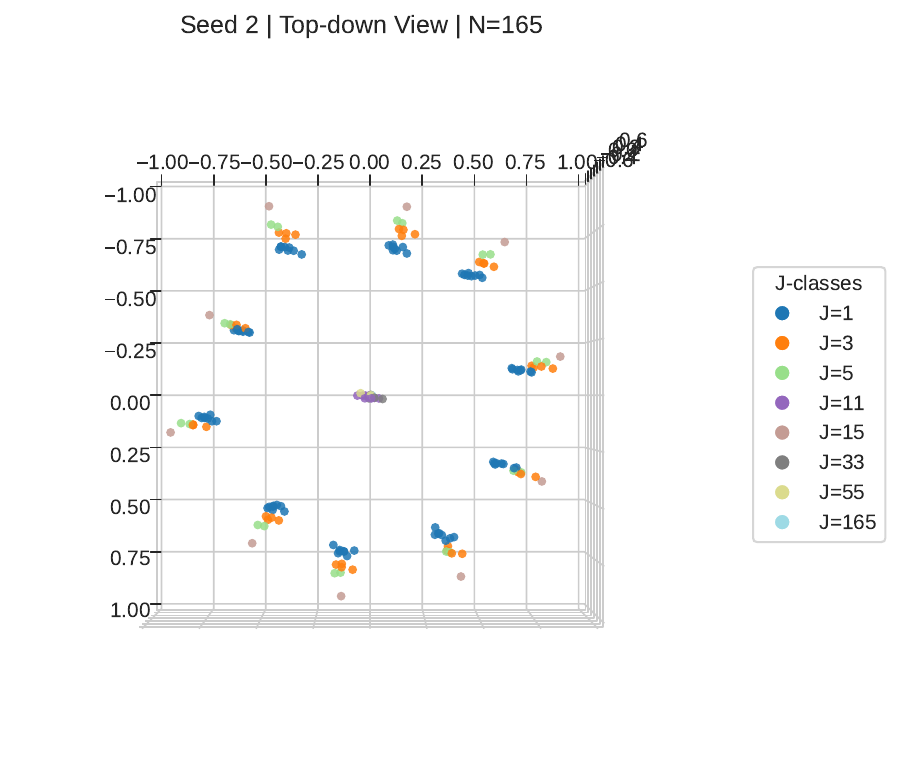}
        \caption{Seed 2}
    \end{subfigure}
    \hfill
    \begin{subfigure}[t]{0.24\textwidth}
        \centering
        \includegraphics[width=\linewidth]{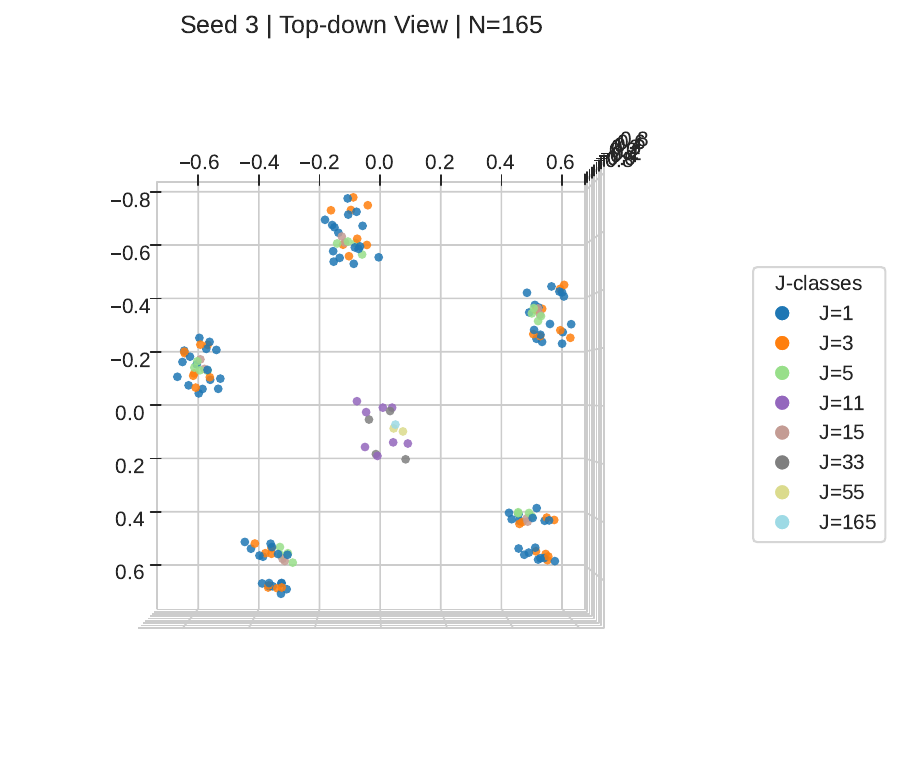}
        \caption{Seed 3}
    \end{subfigure}
    \hfill
    \begin{subfigure}[t]{0.24\textwidth}
        \centering
        \includegraphics[width=\linewidth]{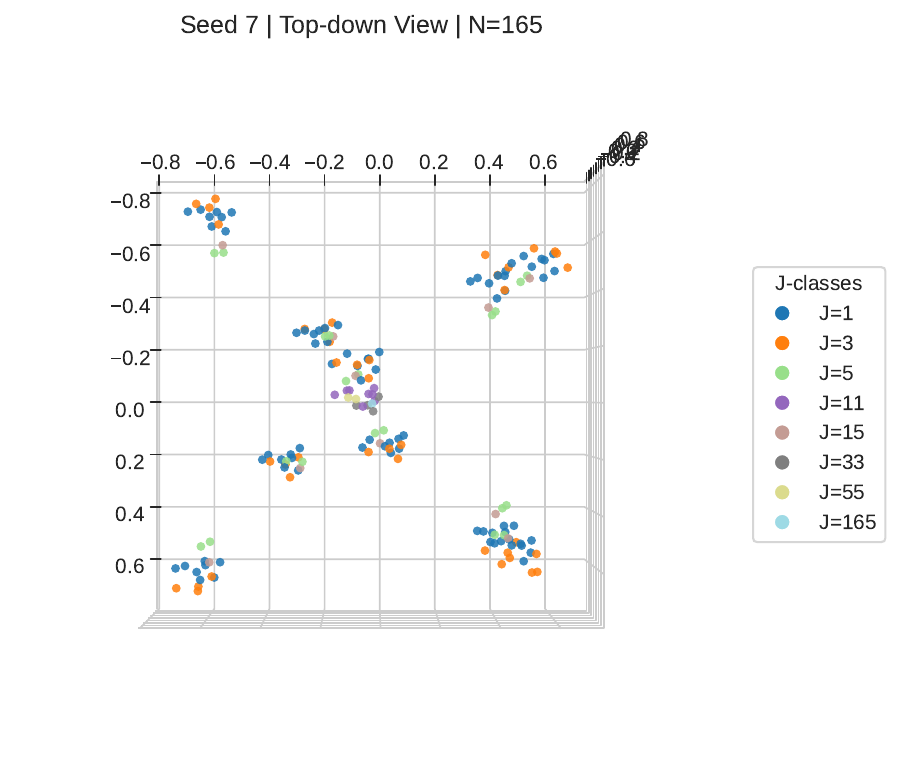}
        \caption{Seed 7}
    \end{subfigure}
    \hfill
    \begin{subfigure}[t]{0.24\textwidth}
        \centering
        \includegraphics[width=\linewidth]{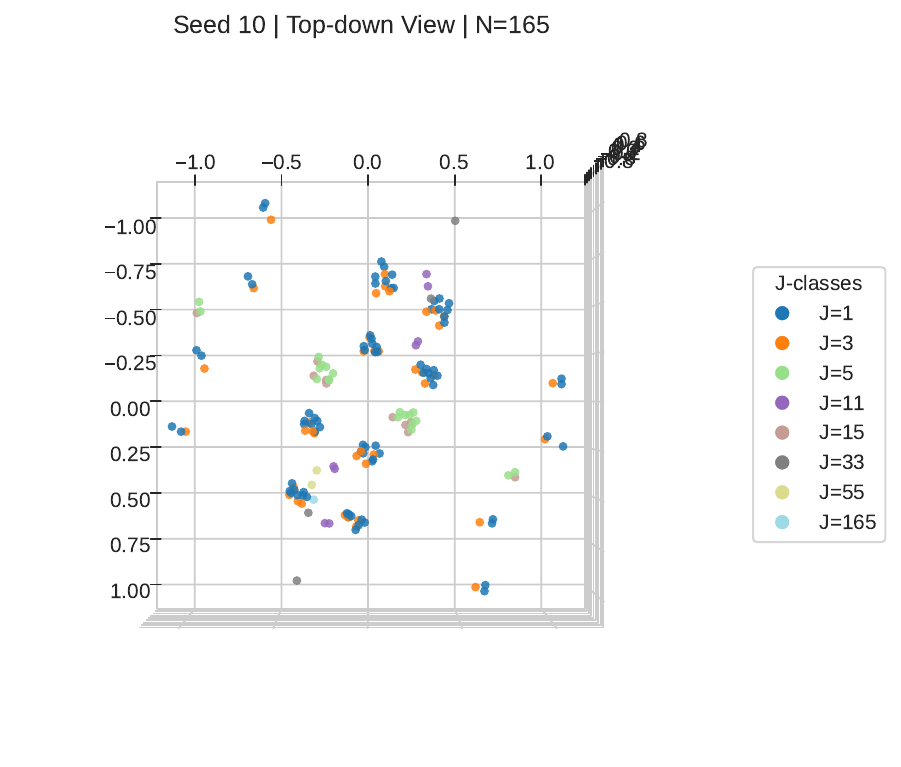}
        \caption{Seed 10}
    \end{subfigure}

    \vspace{0.4cm}

    \textbf{$N = 154$}\\[0.2cm]
    \begin{subfigure}[t]{0.24\textwidth}
        \centering
        \includegraphics[width=\linewidth]{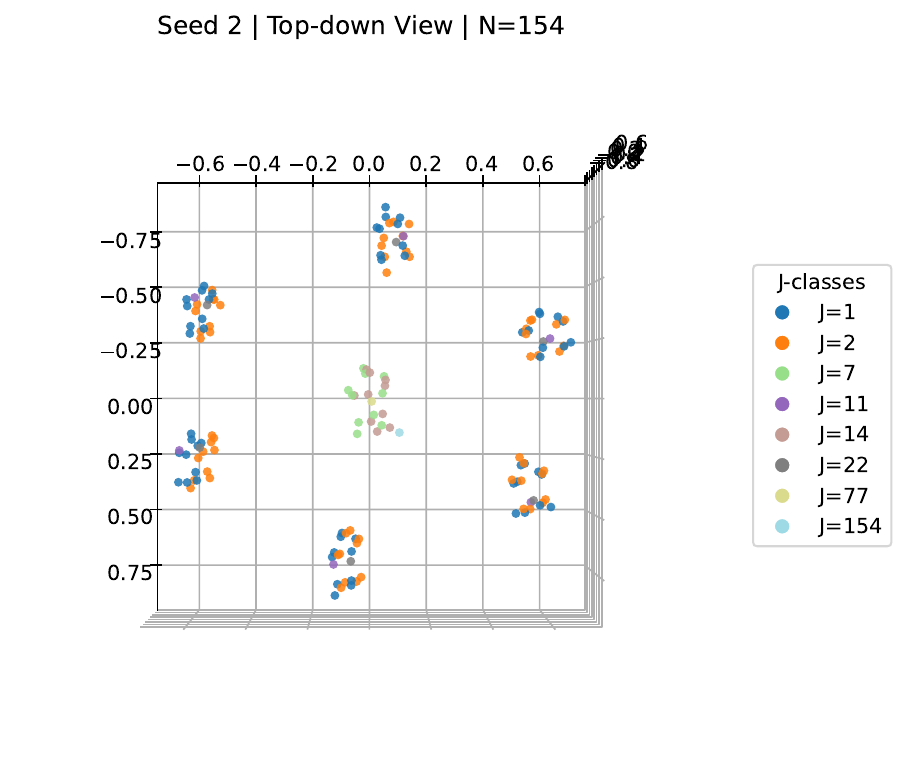}
        \caption{Seed 2}
    \end{subfigure}
    \hfill
    \begin{subfigure}[t]{0.24\textwidth}
        \centering
        \includegraphics[width=\linewidth]{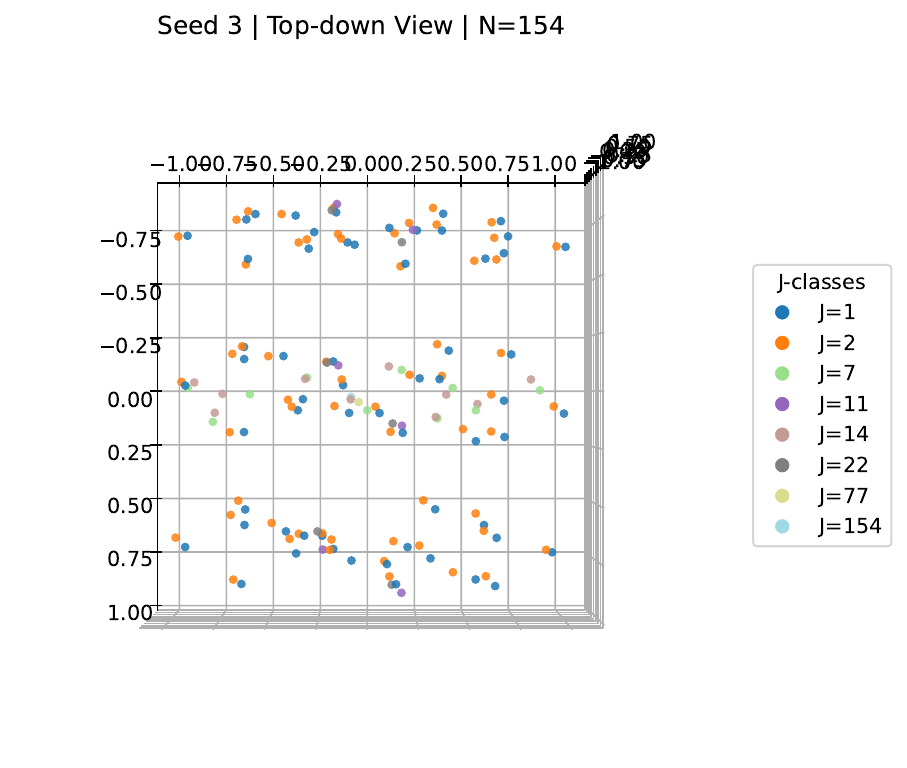}
        \caption{Seed 3}
    \end{subfigure}
    \hfill
    \begin{subfigure}[t]{0.24\textwidth}
        \centering
        \includegraphics[width=\linewidth]{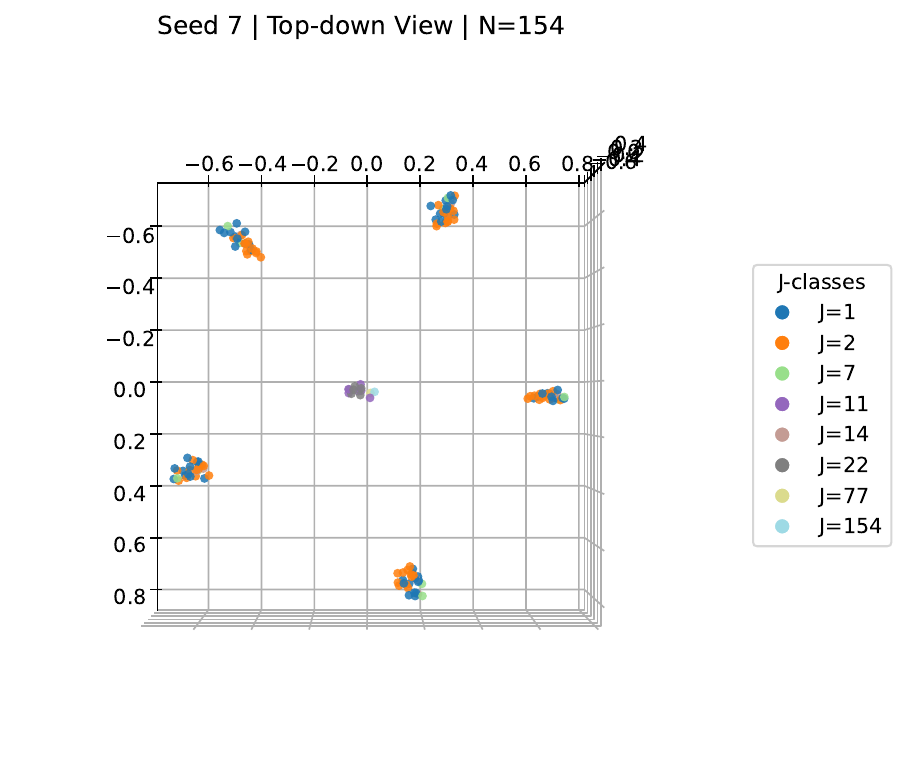}
        \caption{Seed 7}
    \end{subfigure}
    \hfill
    \begin{subfigure}[t]{0.24\textwidth}
        \centering
        \includegraphics[width=\linewidth]{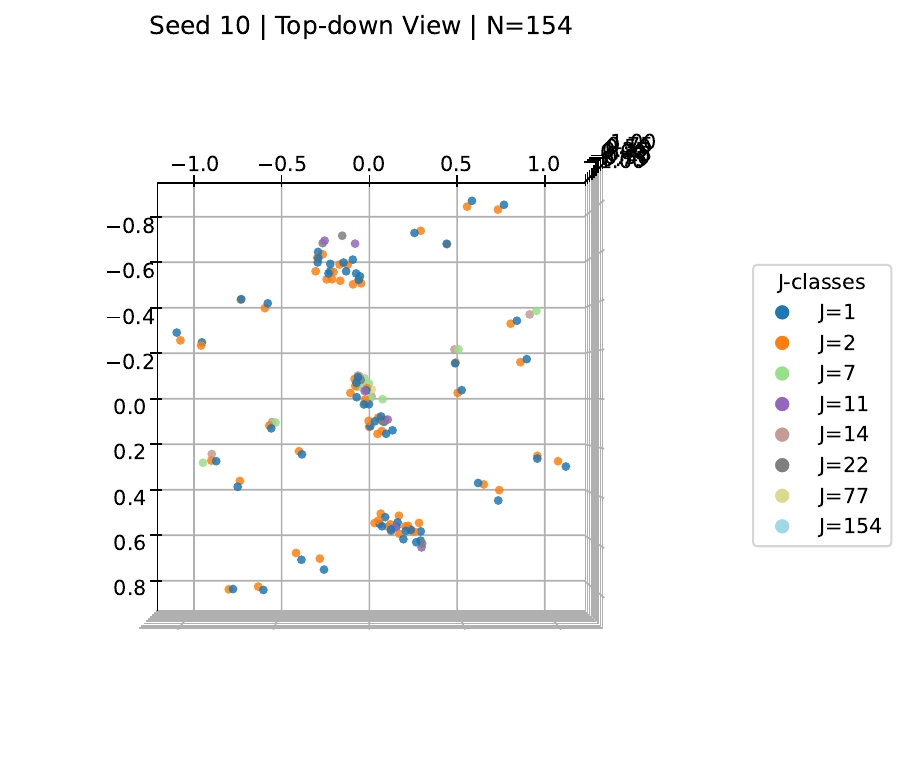}
        \caption{Seed 10}
    \end{subfigure}

    \vspace{0.4cm}

    \textbf{$n = 113, 143$}\\[0.2cm]

    \begin{subfigure}[t]{0.24\textwidth}
        \centering
        \includegraphics[width=\linewidth]{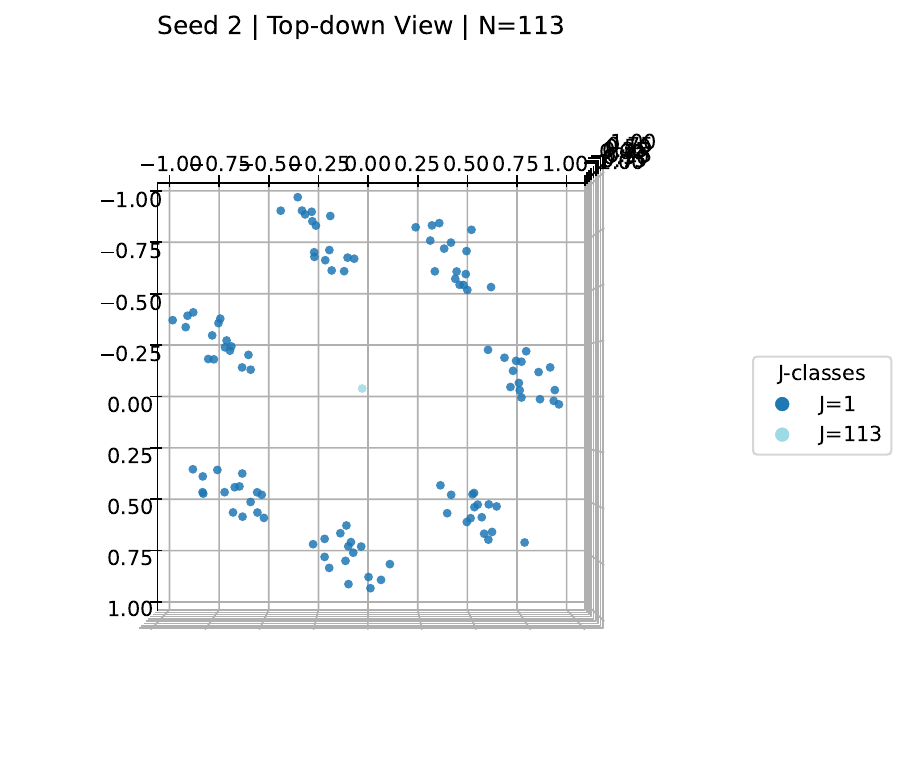}
        \caption{Seed 2}
    \end{subfigure}
    \hfill
    \begin{subfigure}[t]{0.24\textwidth}
        \centering
        \includegraphics[width=\linewidth]{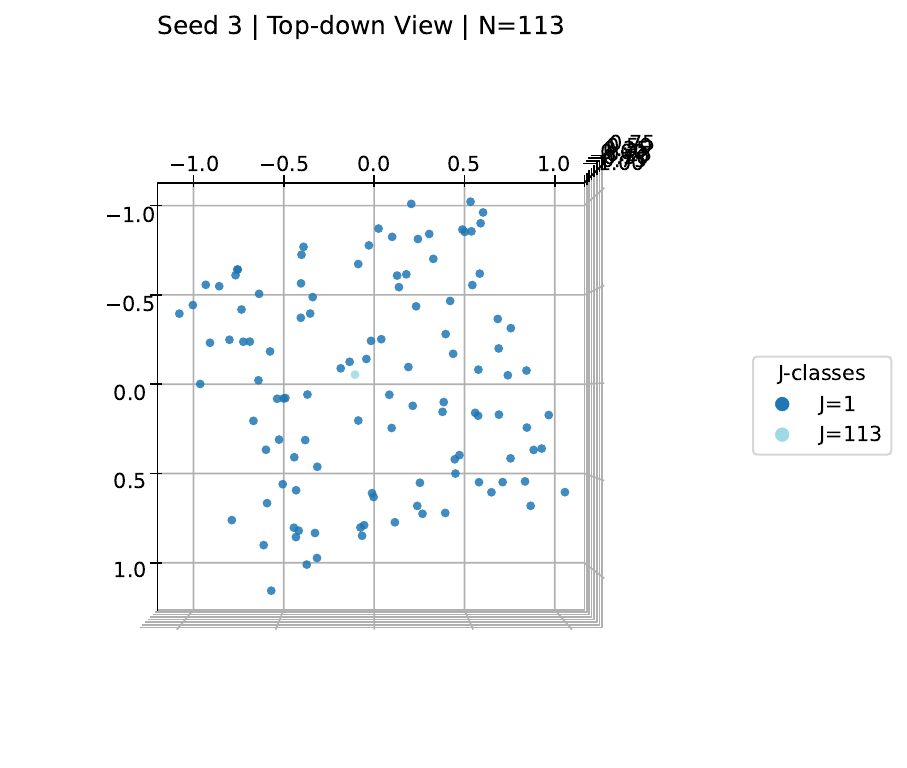}
        \caption{Seed 3}
    \end{subfigure}
    \hfill
    \begin{subfigure}[t]{0.24\textwidth}
        \centering
        \includegraphics[width=\linewidth]{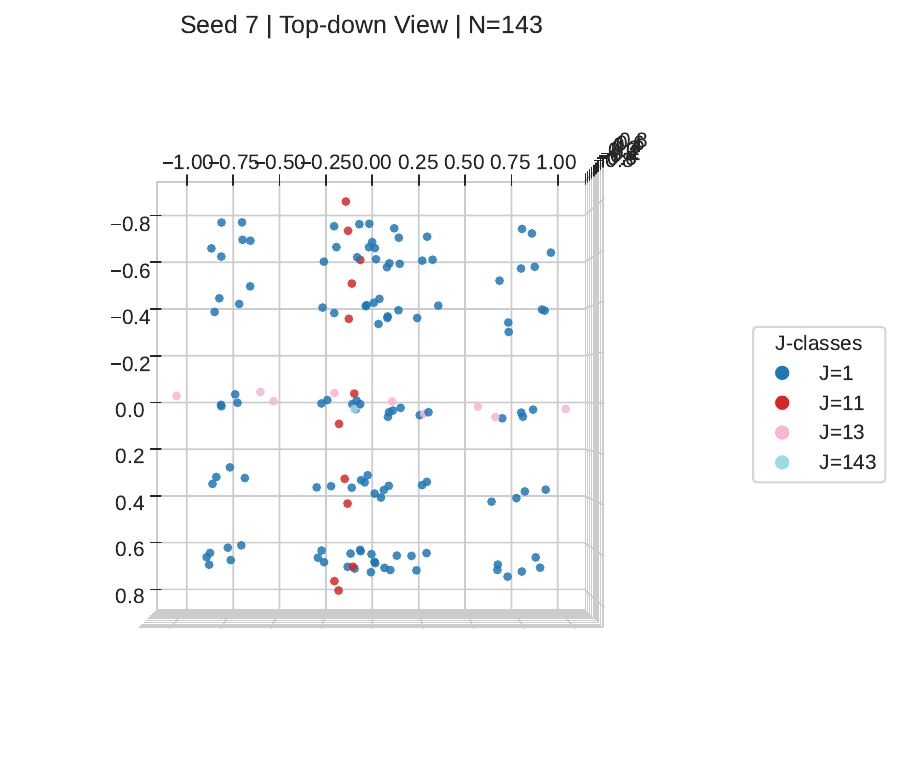}
        \caption{Seed 7}
    \end{subfigure}
    \hfill
    \begin{subfigure}[t]{0.24\textwidth}
        \centering
        \includegraphics[width=\linewidth]{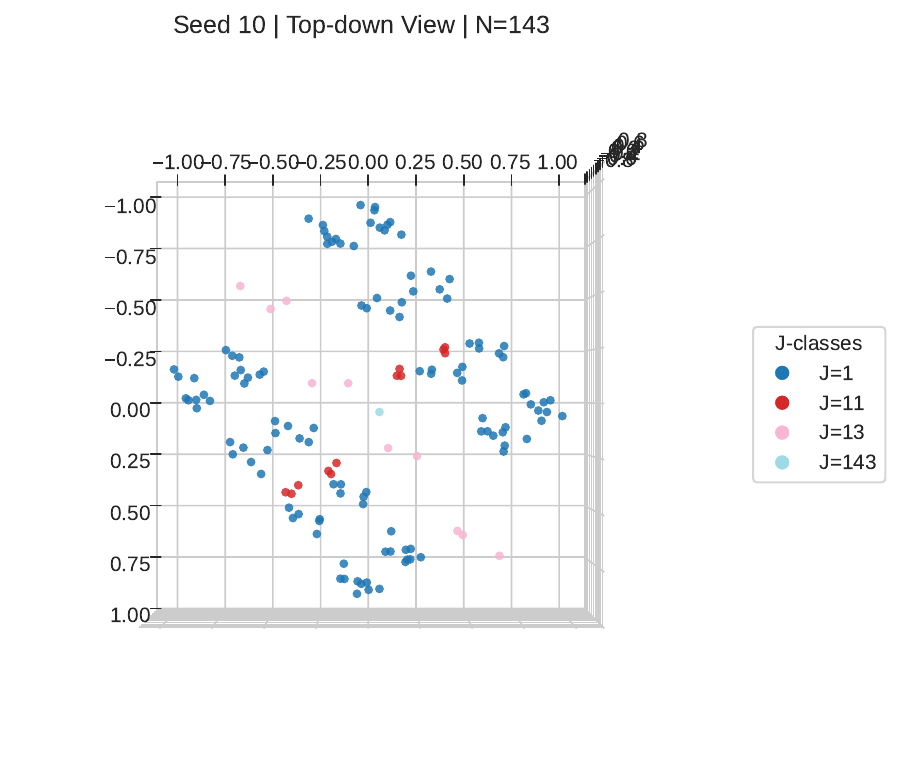}
        \caption{Seed 10}
    \end{subfigure}

    \caption{
    \textbf{PCA geometry stability across seeds and moduli.}
    Across all settings, embeddings consistently organize into structured low-dimensional manifolds aligned with $\mathcal{J}$-class structure.
    }

    \label{fig:pca_stability_clean}
\end{figure*}

\newpage

\subsection{Attention Routing Stability}

We further examine representational stability by visualizing attention maps across multiple random seeds and moduli $n \in \{113, 143, 154, 165\}$. For each setting, we analyze a single attention head and compute full attention matrices over input pairs $(a, b)$. For each input pair, we calculate and plot the attention score that the $=$ token pays to the $a$ token.

To expose algebraic structure, we reorder inputs according to their $\mathcal{J}$-class decomposition in $\mathbb{Z}_n$, as induced by the GCR partitioning framework \cite{chughtai2023toy}. This ordering reveals block structures corresponding to shared GCDs.

We observe that while individual attention values and patterns vary across seeds, the coarse block structure remains highly stable. This suggests that the attention mechanism consistently learns to route information according to the underlying $\mathcal{J}$-classes, independent of initialization.

\newpage

\begin{figure*}[!hb]
    \centering

\textbf{$n = 113$} \\[0.8em]

\begin{subfigure}[t]{0.95\textwidth}
    \centering
    \includegraphics[width=0.78\linewidth]{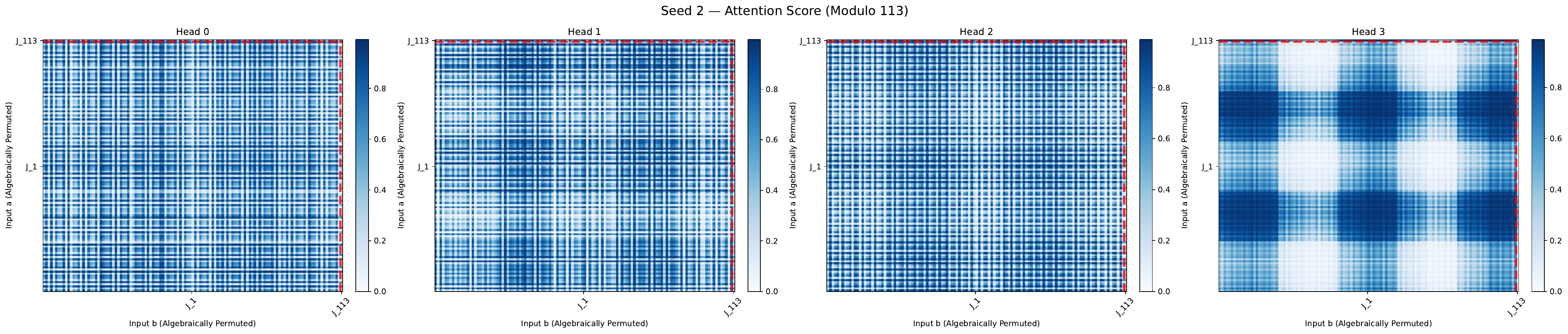}
    \caption{Seed 2}
\end{subfigure}

\vspace{0.35cm}

\begin{subfigure}[t]{0.95\textwidth}
    \centering
    \includegraphics[width=0.78\linewidth]{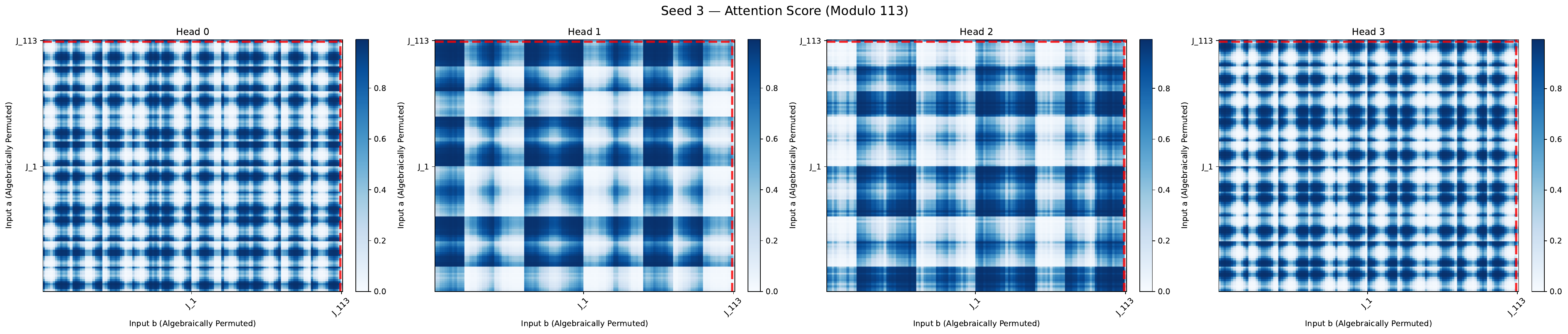}
    \caption{Seed 3}
\end{subfigure}

\vspace{0.35cm}

\begin{subfigure}[t]{0.95\textwidth}
    \centering
    \includegraphics[width=0.78\linewidth]{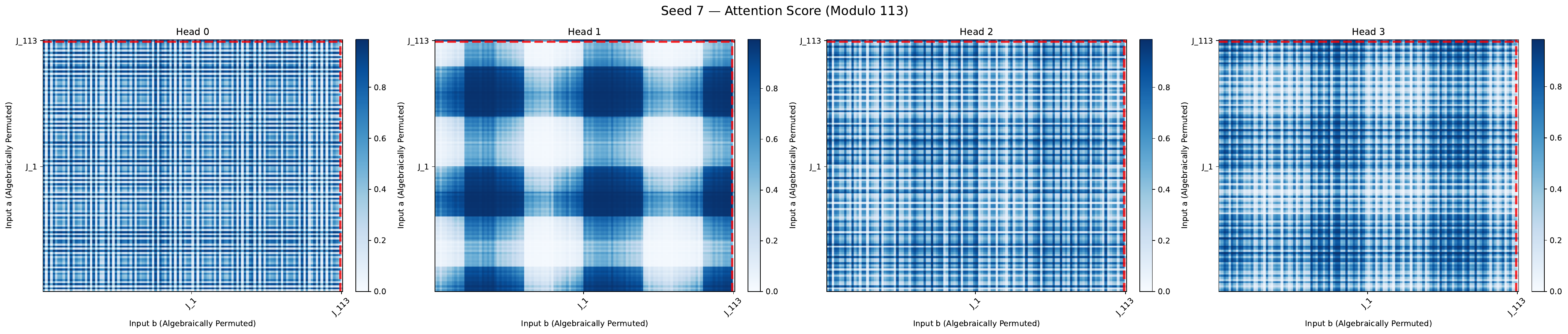}
    \caption{Seed 7}
\end{subfigure}

\vspace{0.35cm}

\begin{subfigure}[t]{0.95\textwidth}
    \centering
    \includegraphics[width=0.78\linewidth]{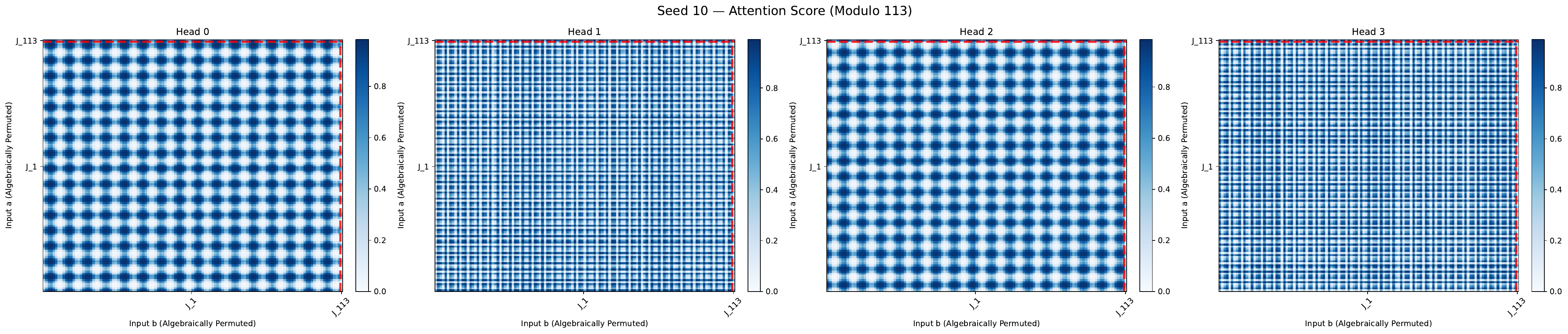}
    \caption{Seed 10}
\end{subfigure}

\caption{Attention stability for $\mathbb{Z}_{113}$ across seeds.}
\label{fig:attn_113}
\end{figure*}

\newpage

\begin{figure*}[!hb]
    \centering

\textbf{$n = 143$} \\[0.8em]

\begin{subfigure}[t]{0.95\textwidth}
    \centering
    \includegraphics[width=0.78\linewidth]{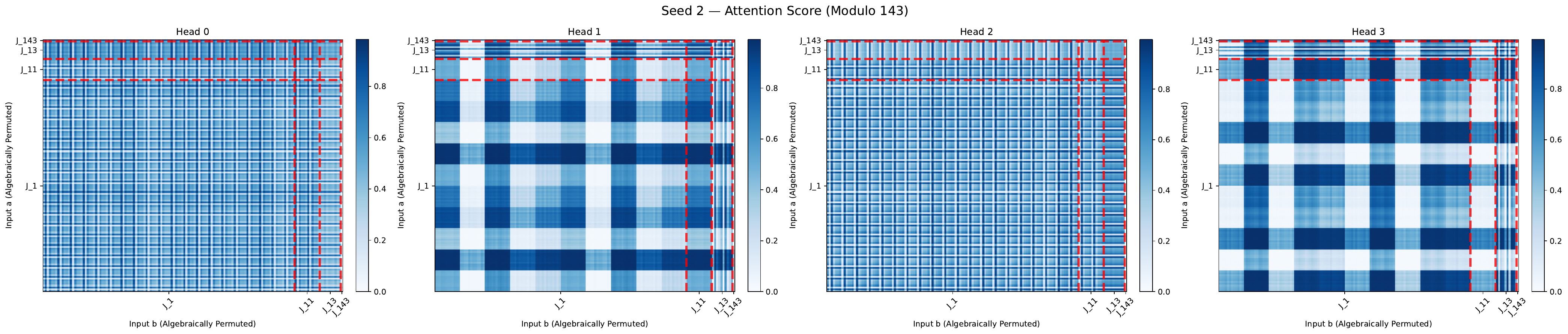}
    \caption{Seed 2}
\end{subfigure}

\vspace{0.35cm}

\begin{subfigure}[t]{0.95\textwidth}
    \centering
    \includegraphics[width=0.78\linewidth]{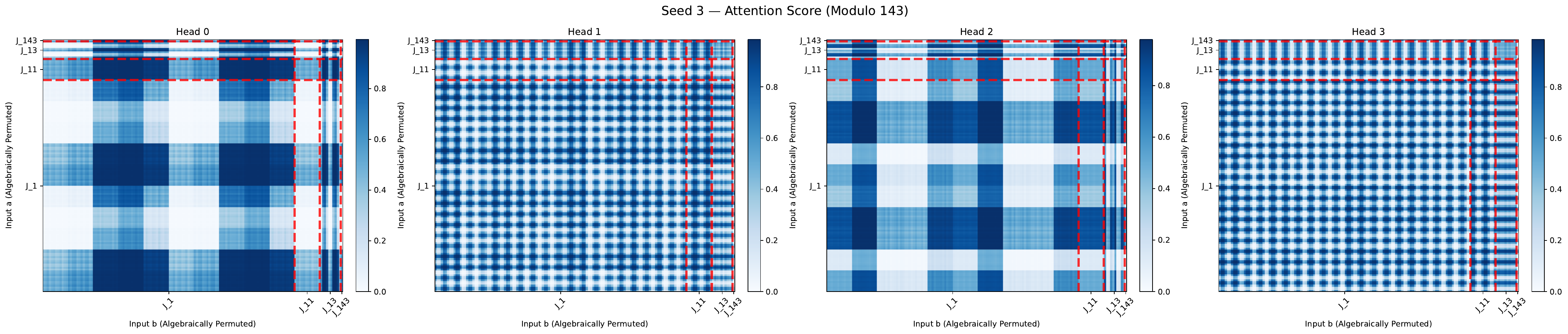}
    \caption{Seed 3}
\end{subfigure}

\vspace{0.35cm}

\begin{subfigure}[t]{0.95\textwidth}
    \centering
    \includegraphics[width=0.78\linewidth]{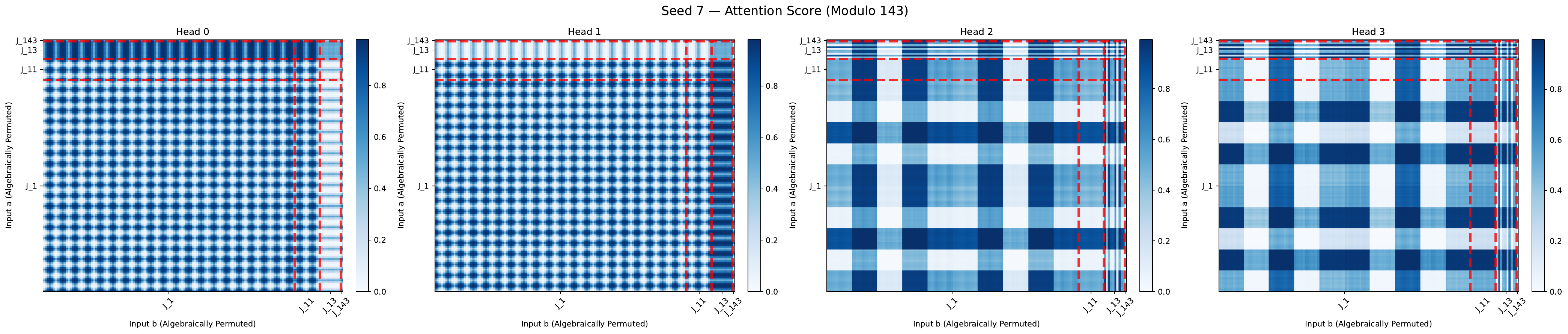}
    \caption{Seed 7}
\end{subfigure}

\vspace{0.35cm}

\begin{subfigure}[t]{0.95\textwidth}
    \centering
    \includegraphics[width=0.78\linewidth]{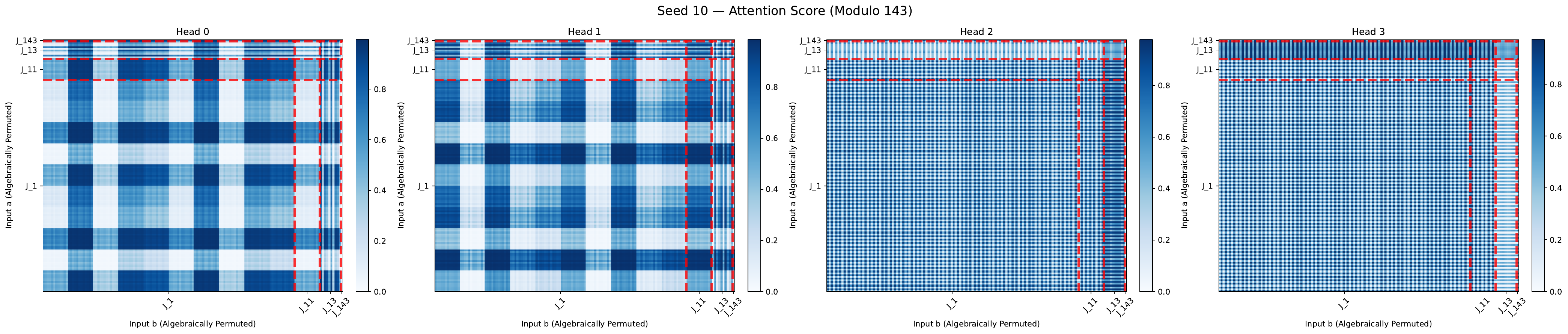}
    \caption{Seed 10}
\end{subfigure}

\caption{Attention stability for $\mathbb{Z}_{143}$ across seeds.}
\label{fig:attn_143}
\end{figure*}

\newpage

\begin{figure*}[!hb]
    \centering

\textbf{$n = 154$} \\[0.8em]

\begin{subfigure}[t]{0.95\textwidth}
    \centering
    \includegraphics[width=0.78\linewidth]{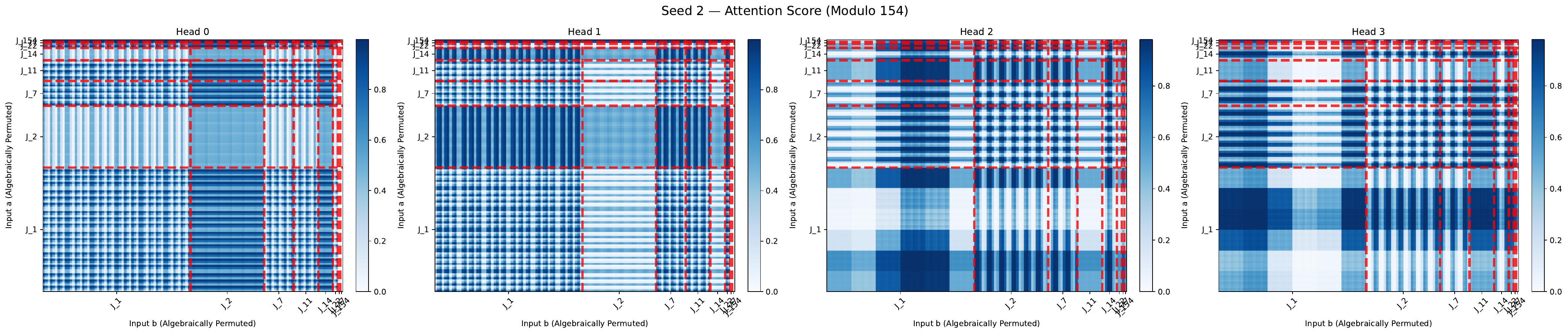}
    \caption{Seed 2}
\end{subfigure}

\vspace{0.35cm}

\begin{subfigure}[t]{0.95\textwidth}
    \centering
    \includegraphics[width=0.78\linewidth]{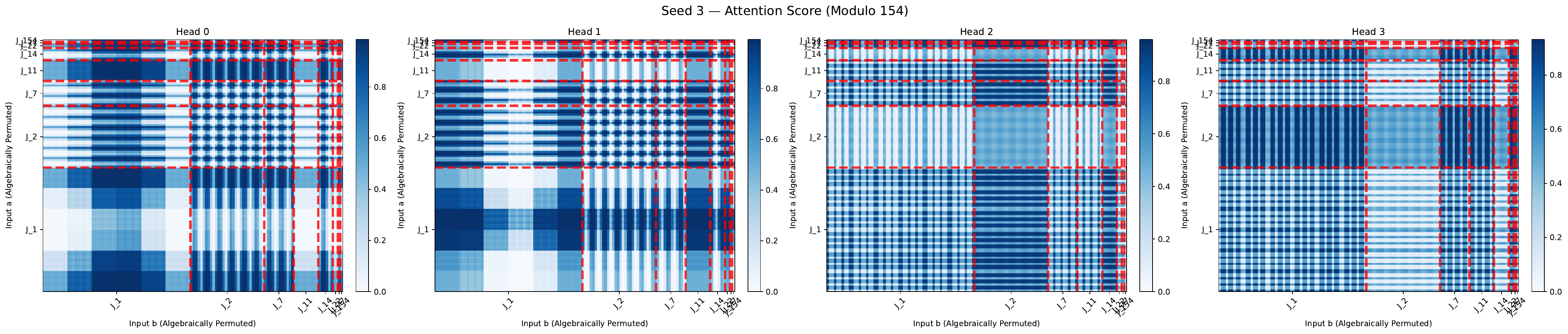}
    \caption{Seed 3}
\end{subfigure}

\vspace{0.35cm}

\begin{subfigure}[t]{0.95\textwidth}
    \centering
    \includegraphics[width=0.78\linewidth]{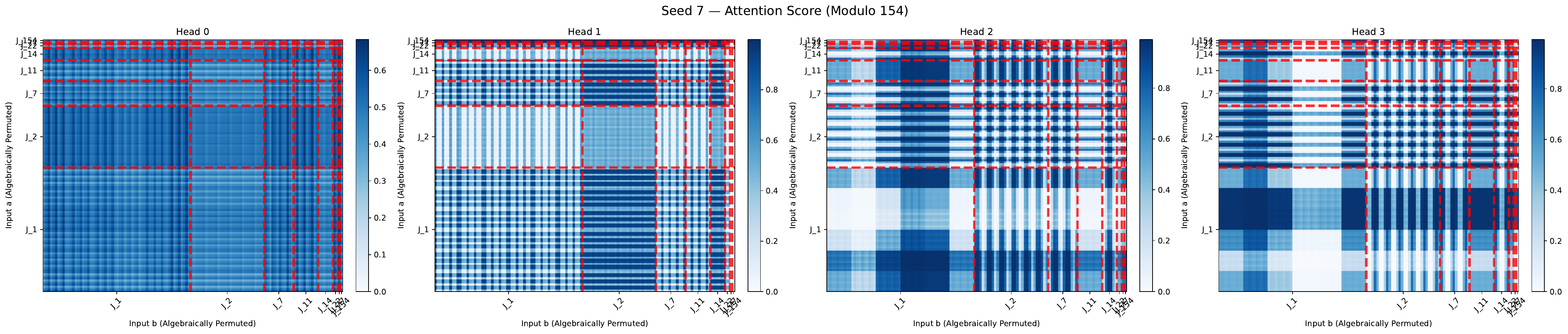}
    \caption{Seed 7}
\end{subfigure}

\vspace{0.35cm}

\begin{subfigure}[t]{0.95\textwidth}
    \centering
    \includegraphics[width=0.78\linewidth]{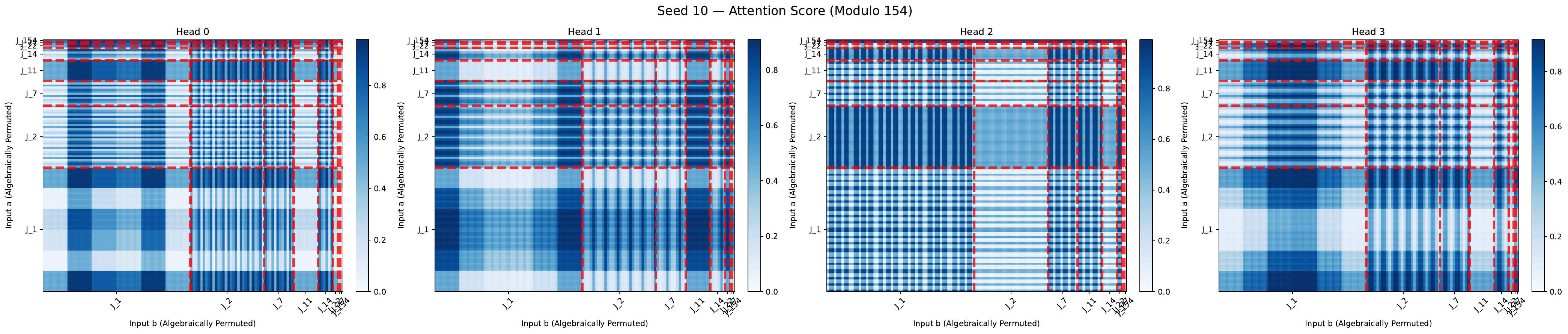}
    \caption{Seed 10}
\end{subfigure}

\caption{Attention stability for $\mathbb{Z}_{154}$ across seeds.}
\label{fig:attn_154}
\end{figure*}

\newpage

\begin{figure*}[!hb]
    \centering

\textbf{$n = 165$} \\[0.8em]

\begin{subfigure}[t]{0.95\textwidth}
    \centering
    \includegraphics[width=0.78\linewidth]{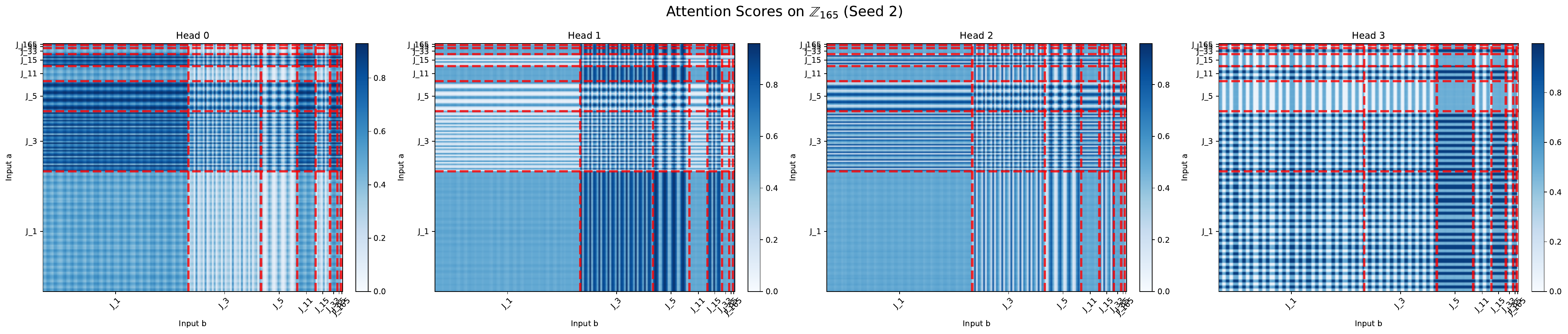}
    \caption{Seed 2}
\end{subfigure}

\vspace{0.35cm}

\begin{subfigure}[t]{0.95\textwidth}
    \centering
    \includegraphics[width=0.78\linewidth]{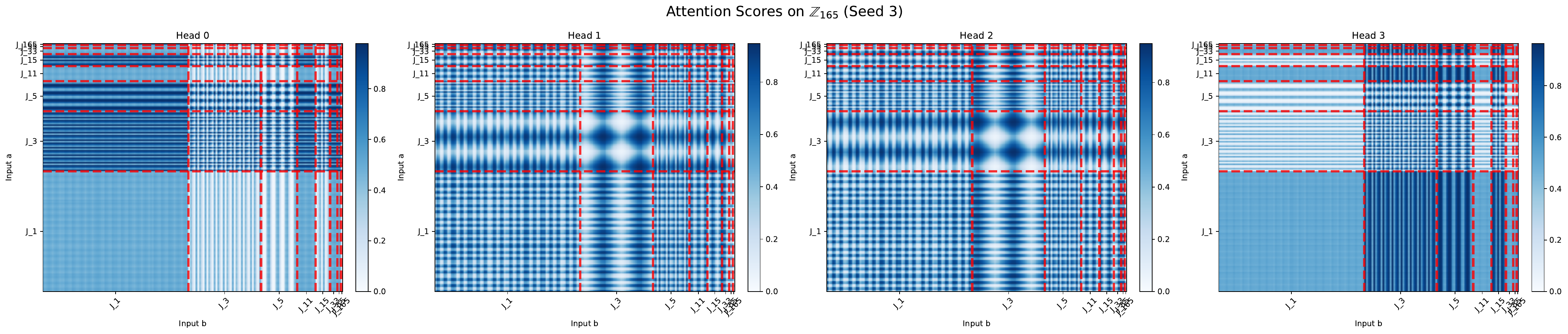}
    \caption{Seed 3}
\end{subfigure}

\vspace{0.35cm}

\begin{subfigure}[t]{0.95\textwidth}
    \centering
    \includegraphics[width=0.78\linewidth]{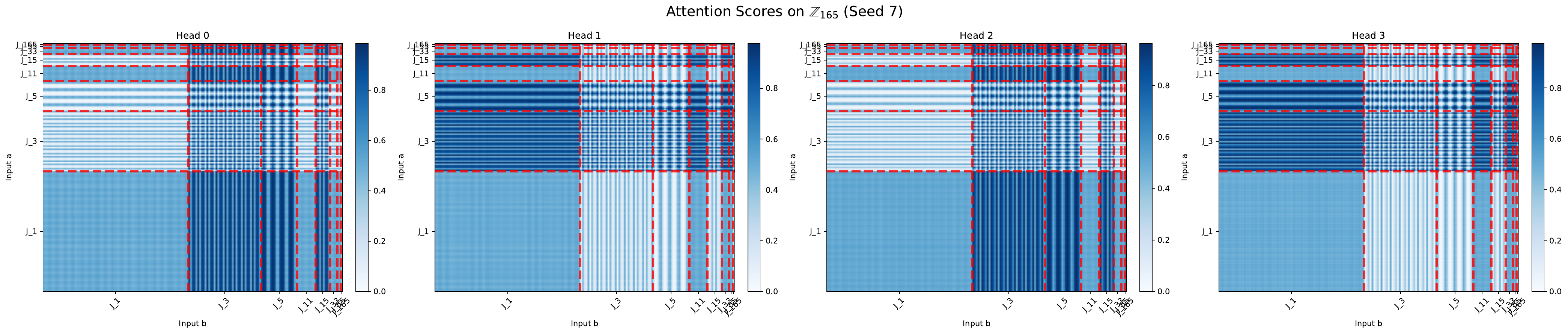}
    \caption{Seed 7}
\end{subfigure}

\vspace{0.35cm}

\begin{subfigure}[t]{0.95\textwidth}
    \centering
    \includegraphics[width=0.78\linewidth]{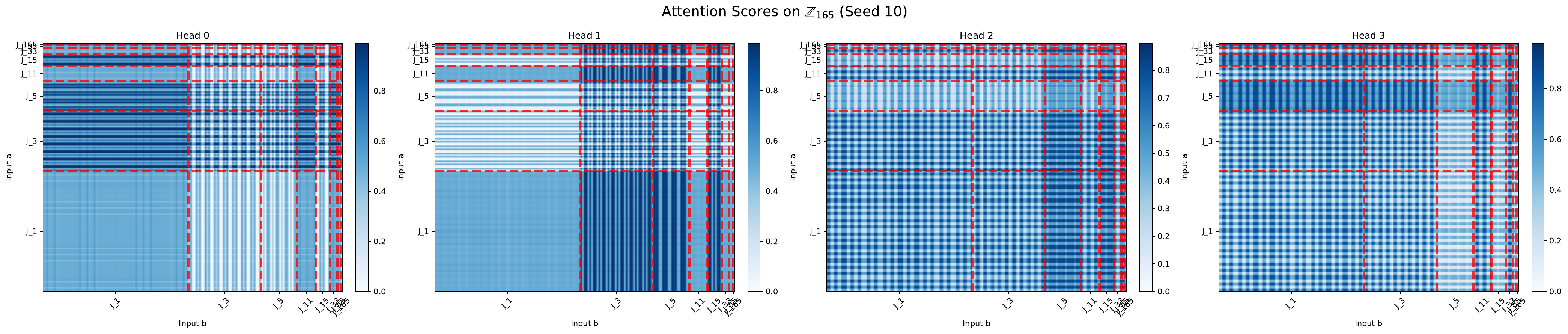}
    \caption{Seed 10}
\end{subfigure}

\caption{Attention stability for $\mathbb{Z}_{165}$ across seeds.}
\label{fig:attn_165}
\end{figure*}

\newpage

\subsection{MLP Feature Stability}

We further assess representational stability by visualizing the hidden layer activations of the MLP layer of our transformer architecture. We perform this analysis across multiple seeds and moduli $n \in \{113, 143, 154, 165\}$. 

For each setting, we randomly sample three neurons from the MLP hidden layer, and visualize their activations for each input pair $a, b$, permuted and ordered by $\mathcal{J}$-classes. This ordering induces a block structure corresponding to shared GCD structure in $\mathbb{Z}_n$ as explained by the GCR algorithm \cite{chughtai2023toy}.

While the exact features vary across seeds, the partitioned structure of the input space remains consistent across runs, suggesting that the MLP consistently decomposes the multiplicative structure of $\mathbb{Z}_n$ in a similar manner.

\begin{figure*}[!hb]
    \centering
    \textbf{$n = 165$} \\[0.5em]

\begin{subfigure}[t]{0.48\textwidth}
    \centering
    \includegraphics[width=\linewidth]{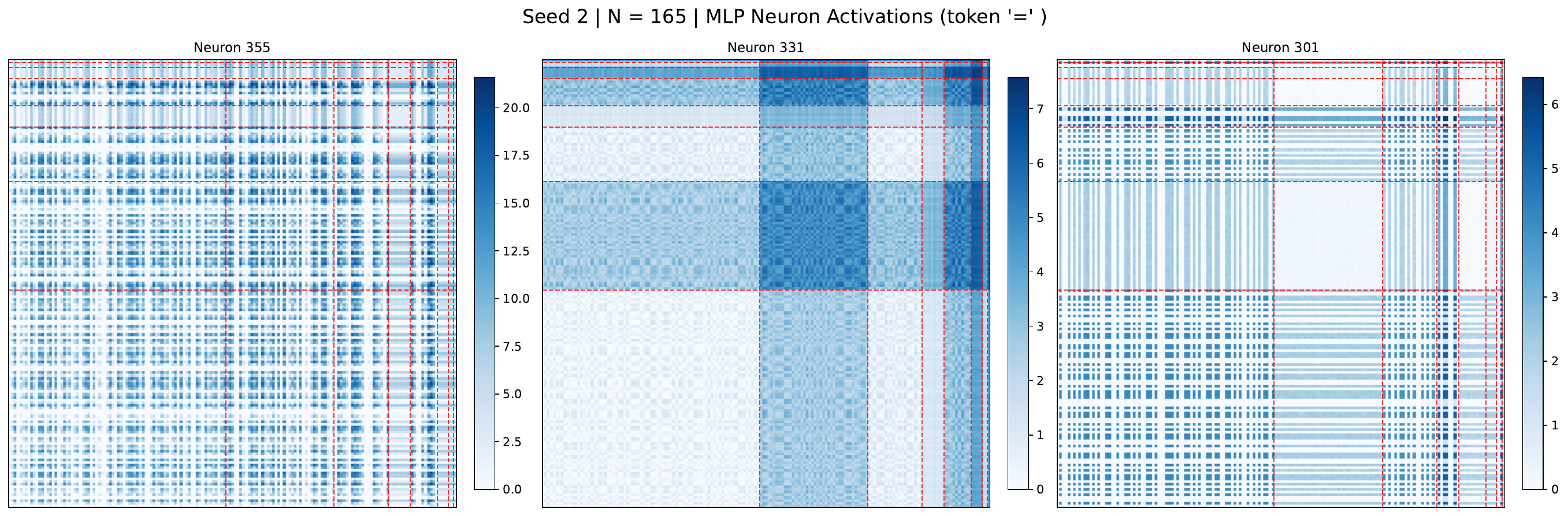}
    \caption{Seed 2}
\end{subfigure}
\hfill
\begin{subfigure}[t]{0.48\textwidth}
    \centering
    \includegraphics[width=\linewidth]{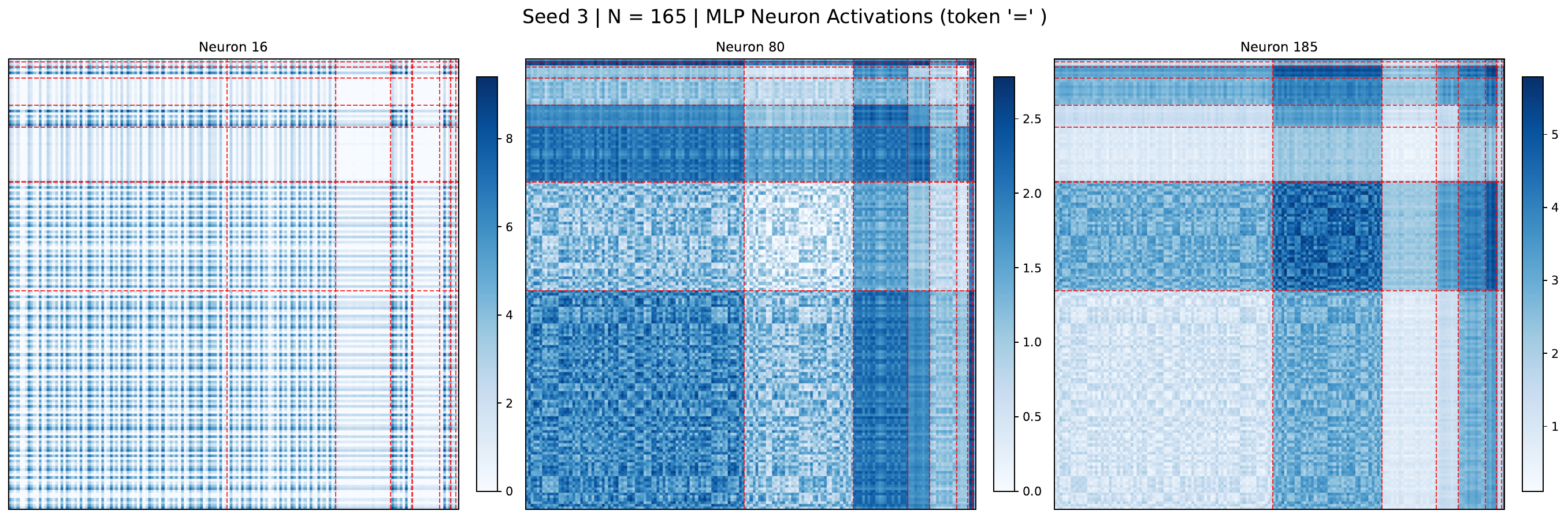}
    \caption{Seed 3}
\end{subfigure}

\vspace{0.4cm}

\begin{subfigure}[t]{0.48\textwidth}
    \centering
    \includegraphics[width=\linewidth]{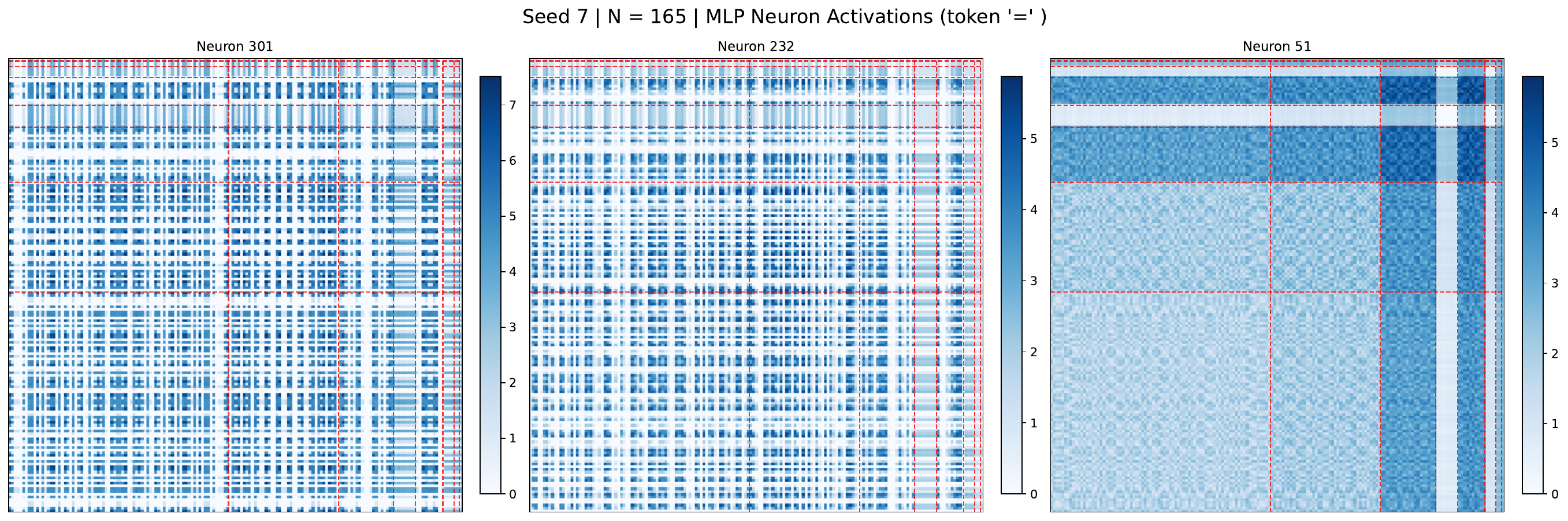}
    \caption{Seed 7}
\end{subfigure}
\hfill
\begin{subfigure}[t]{0.48\textwidth}
    \centering
    \includegraphics[width=\linewidth]{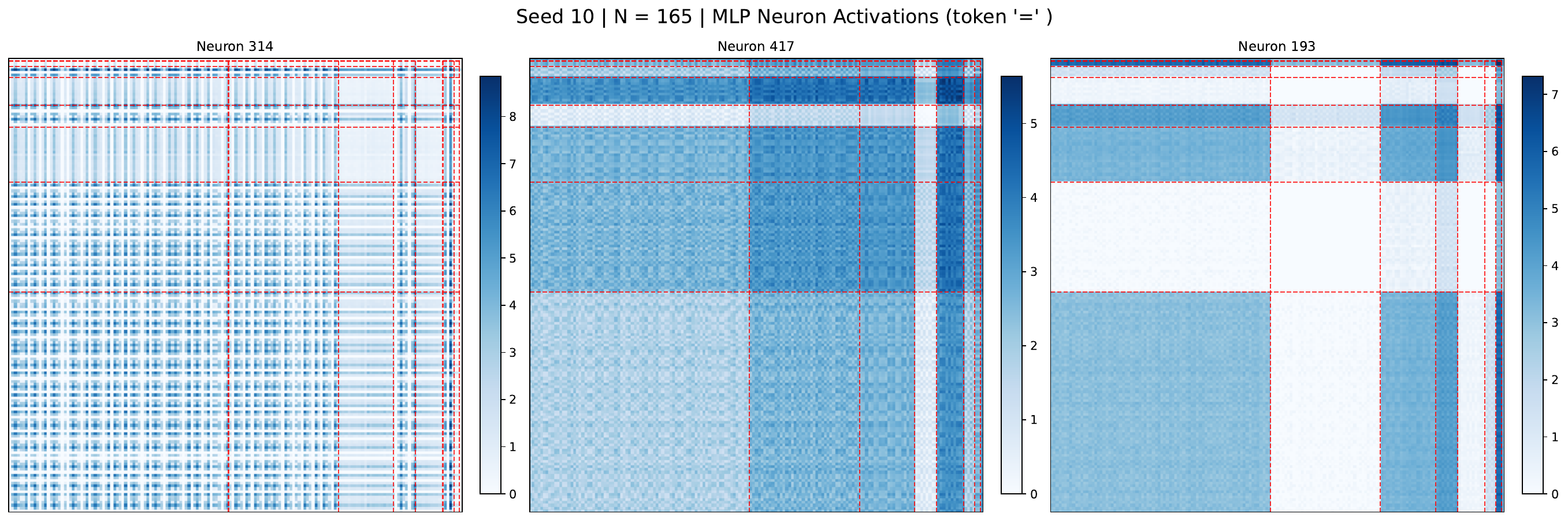}
    \caption{Seed 10}
\end{subfigure}

\caption{
\textbf{MLP feature stability for $n=165$.}
Activation maps of three randomly sampled MLP neurons, reordered by $\mathcal{J}$-class structure of $\mathbb{Z}_{165}$. Across seeds, we observe consistent emergence of block-structured representations aligned with classes, indicating robust recovery of algebraic structure in the learned embedding space.
}
\label{fig:mlp_stability_165}

\end{figure*}

\newpage

\begin{figure*}[!hb]
    \centering

\textbf{$n = 113$} \\[0.5em]

\begin{subfigure}[t]{0.48\textwidth}
    \centering
    \includegraphics[width=\linewidth]{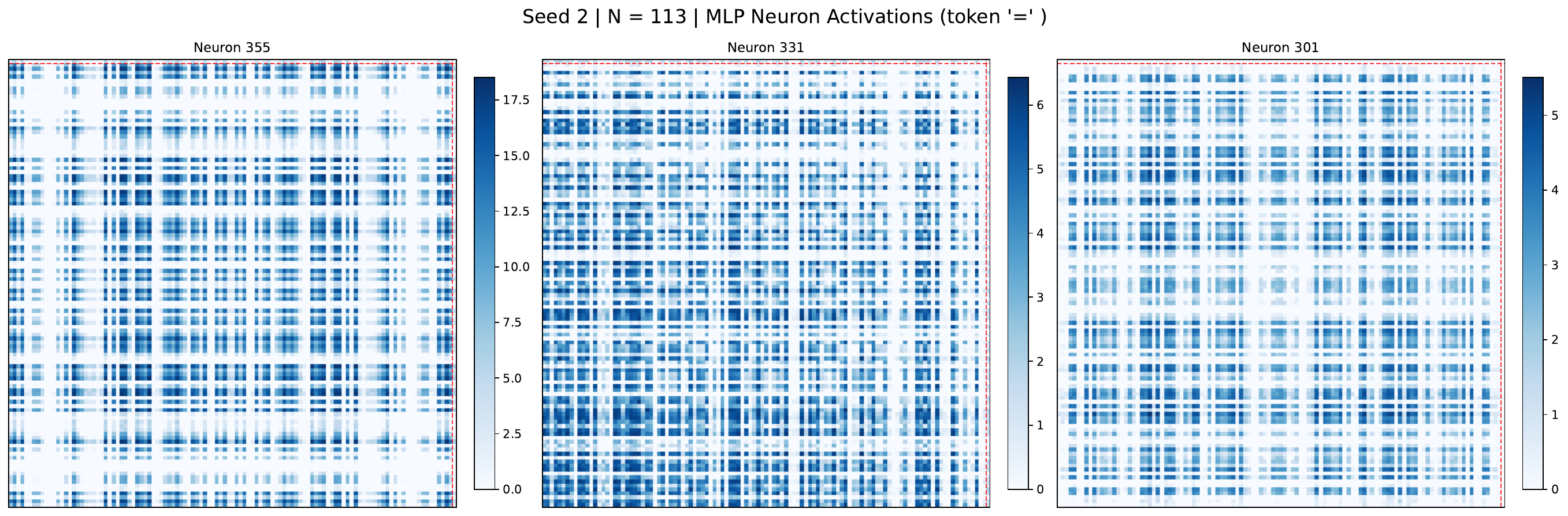}
    \caption{Seed 2}
\end{subfigure}
\hfill
\begin{subfigure}[t]{0.48\textwidth}
    \centering
    \includegraphics[width=\linewidth]{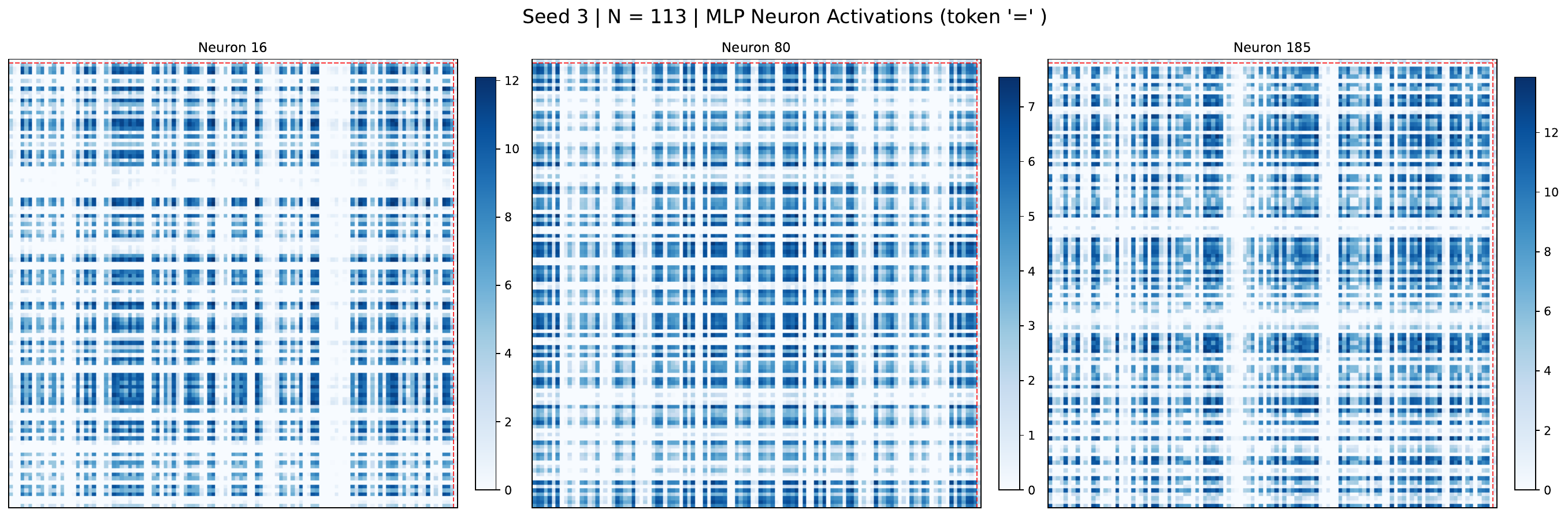}
    \caption{Seed 3}
\end{subfigure}

\vspace{0.1cm}

\begin{subfigure}[t]{0.48\textwidth}
    \centering
    \includegraphics[width=\linewidth]{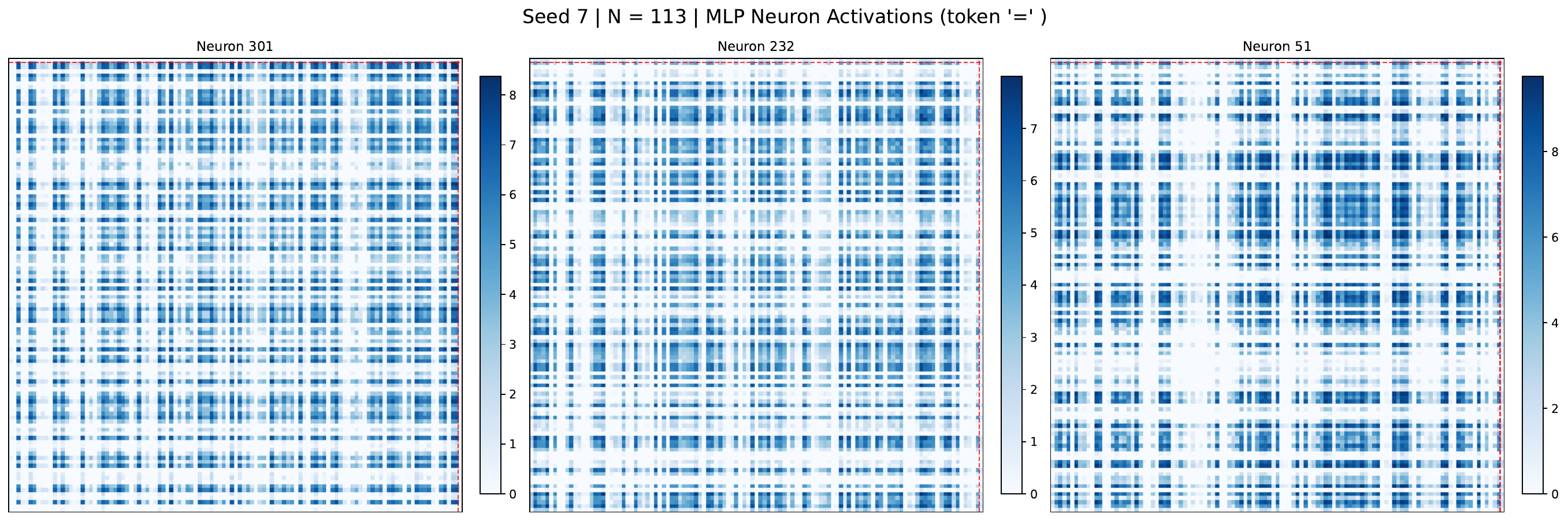}
    \caption{Seed 7}
\end{subfigure}
\hfill
\begin{subfigure}[t]{0.48\textwidth}
    \centering
    \includegraphics[width=\linewidth]{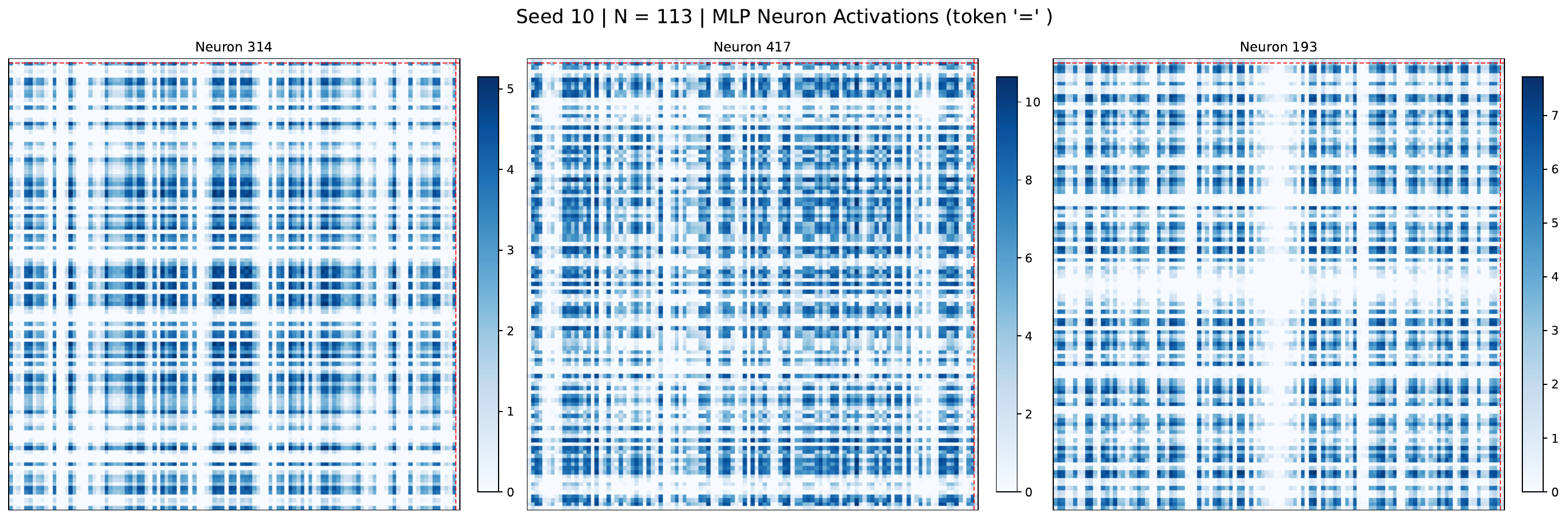}
    \caption{Seed 10}
\end{subfigure}

\vspace{0.1cm}

\textbf{$n = 143$} \\[0.5em]

\begin{subfigure}[t]{0.48\textwidth}
    \centering
    \includegraphics[width=\linewidth]{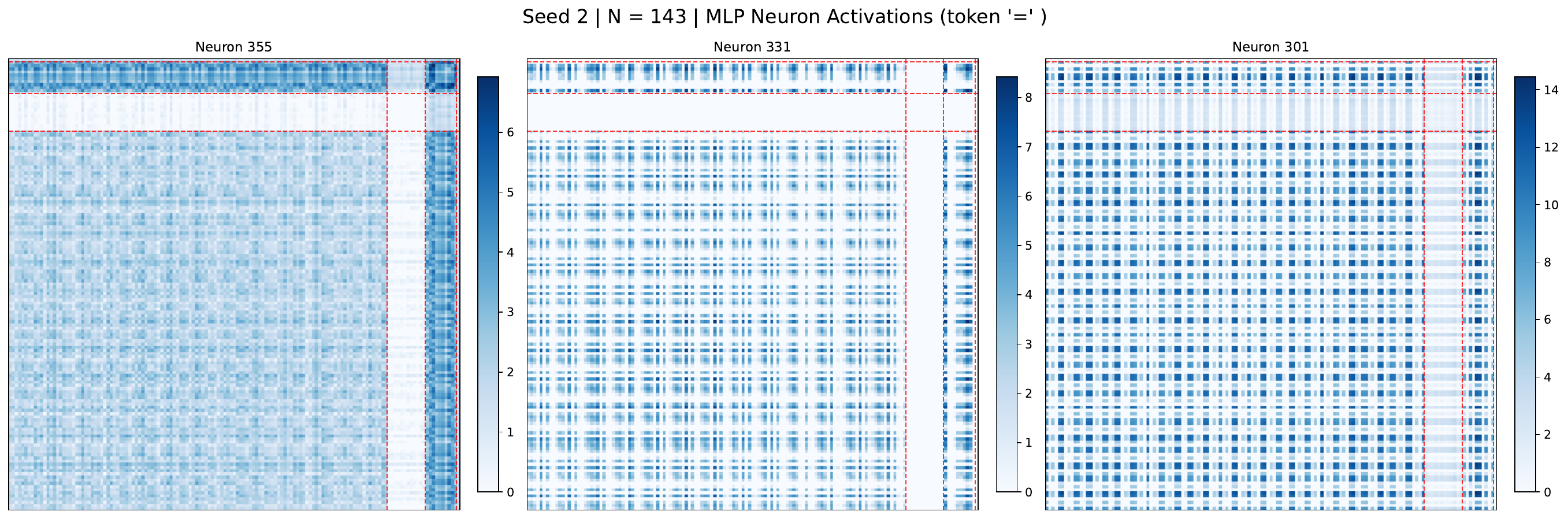}
    \caption{Seed 2}
\end{subfigure}
\hfill
\begin{subfigure}[t]{0.48\textwidth}
    \centering
    \includegraphics[width=\linewidth]{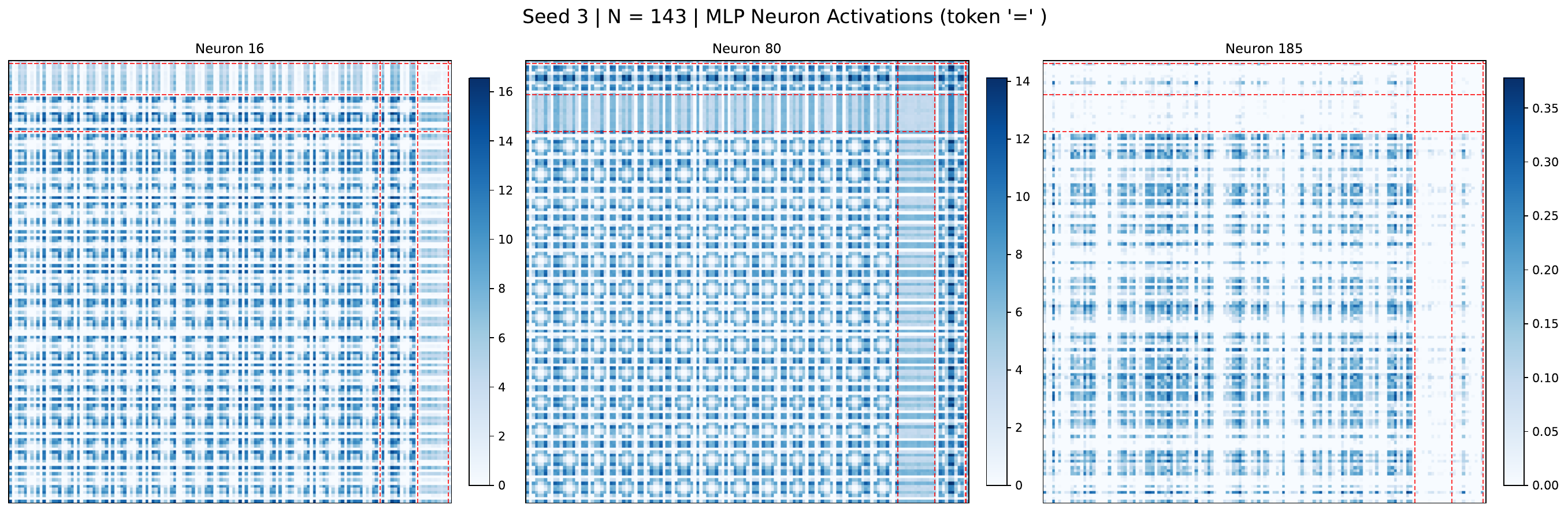}
    \caption{Seed 3}
\end{subfigure}

\vspace{0.1cm}

\begin{subfigure}[t]{0.48\textwidth}
    \centering
    \includegraphics[width=\linewidth]{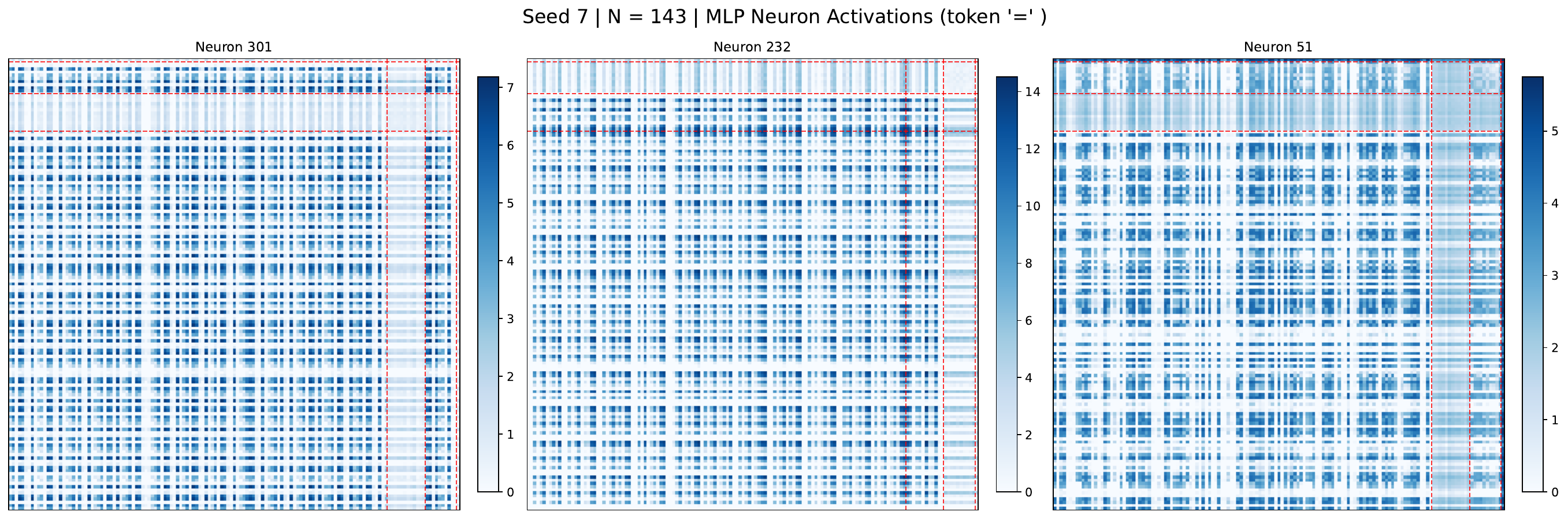}
    \caption{Seed 7}
\end{subfigure}
\hfill
\begin{subfigure}[t]{0.48\textwidth}
    \centering
    \includegraphics[width=\linewidth]{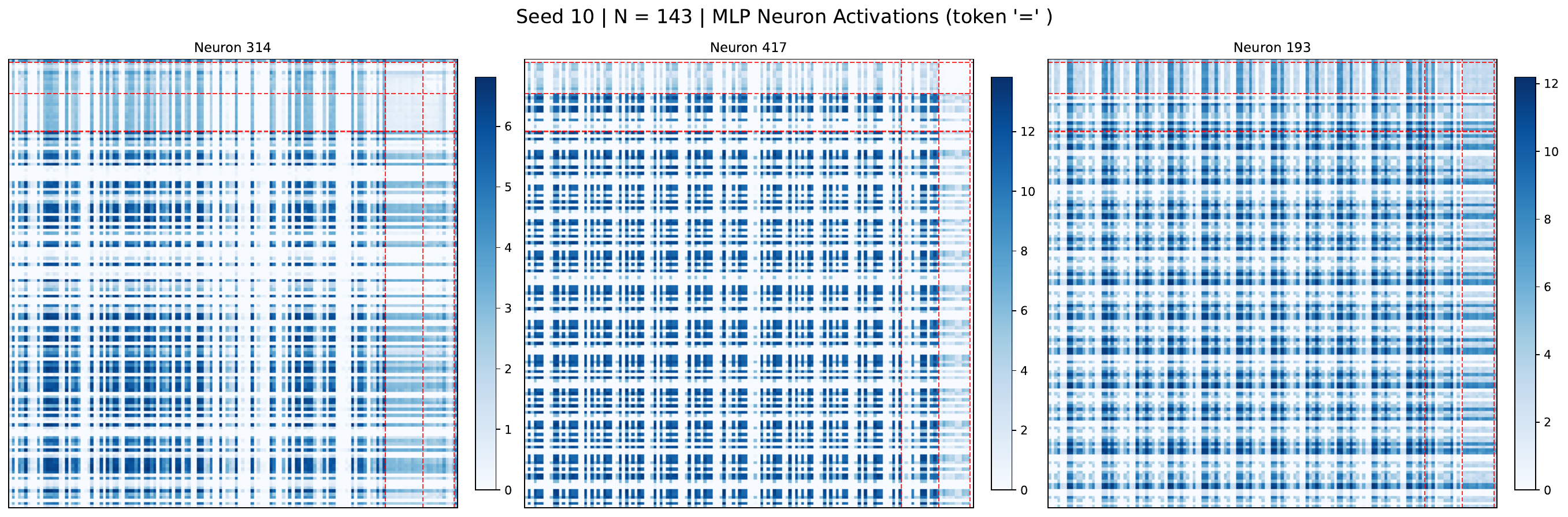}
    \caption{Seed 10}
\end{subfigure}

\vspace{0.1cm}

\textbf{$n = 154$} \\[0.5em]

\begin{subfigure}[H]{0.48\textwidth}
    \centering
    \includegraphics[width=\linewidth]{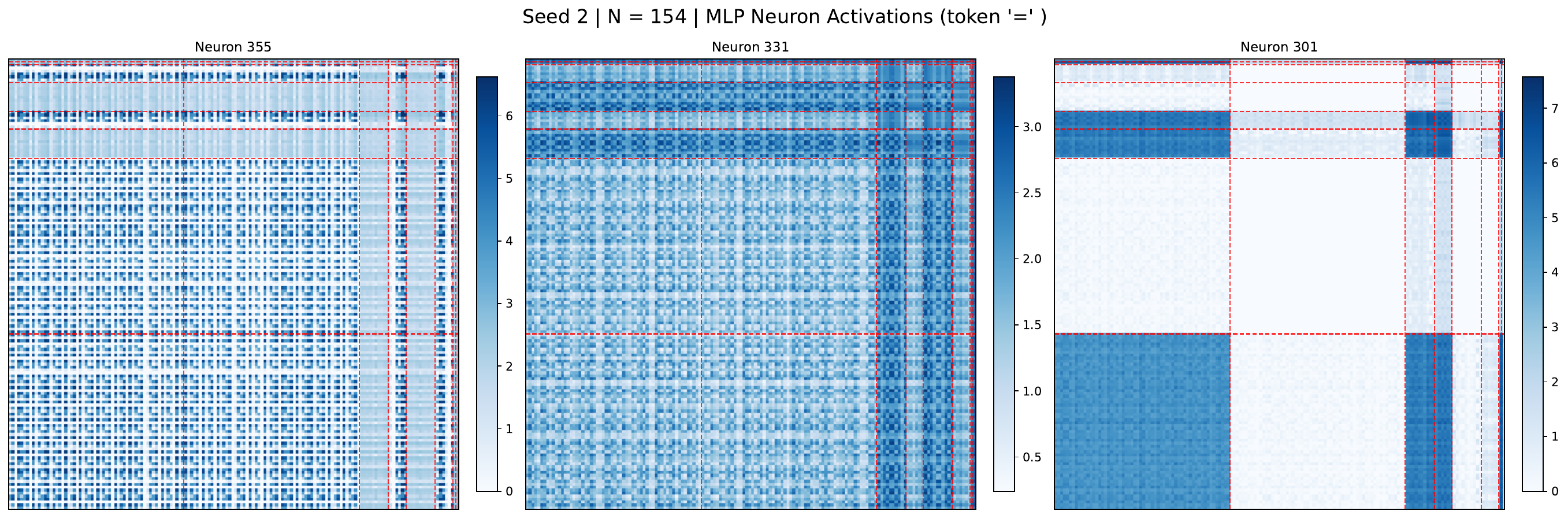}
    \caption{Seed 2}
\end{subfigure}
\hfill
\begin{subfigure}[H]{0.48\textwidth}
    \centering
    \includegraphics[width=\linewidth]{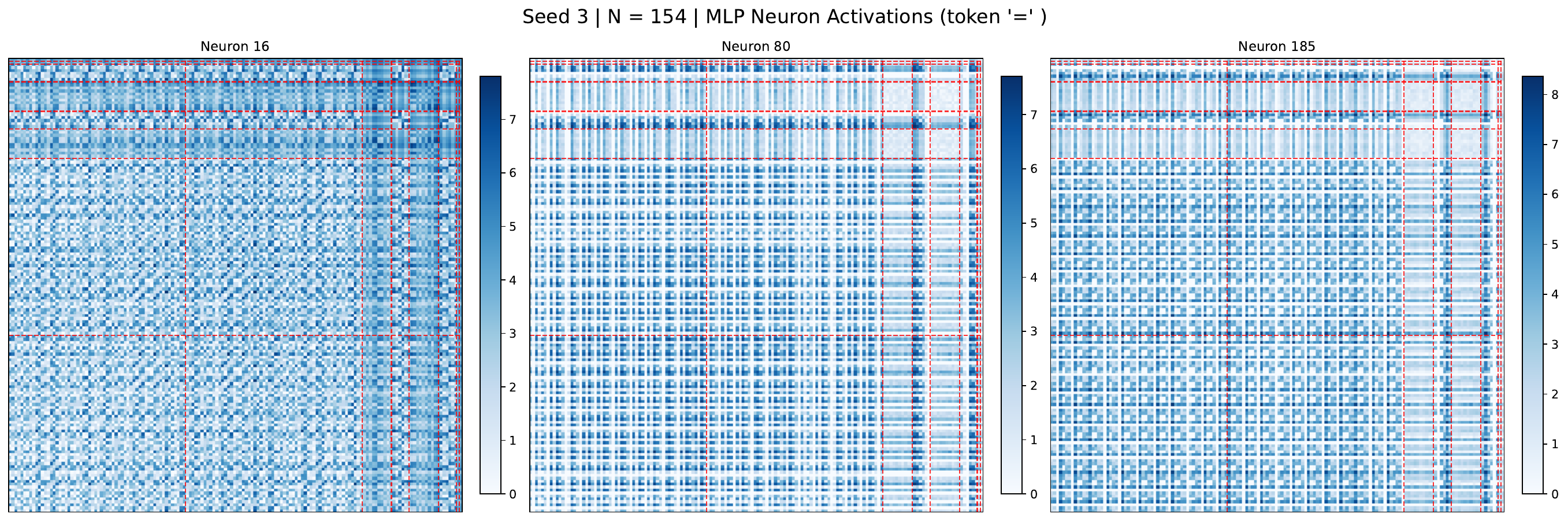}
    \caption{Seed 3}
\end{subfigure}

\vspace{0.1cm}

\begin{subfigure}[H]{0.48\textwidth}
    \centering
    \includegraphics[width=\linewidth]{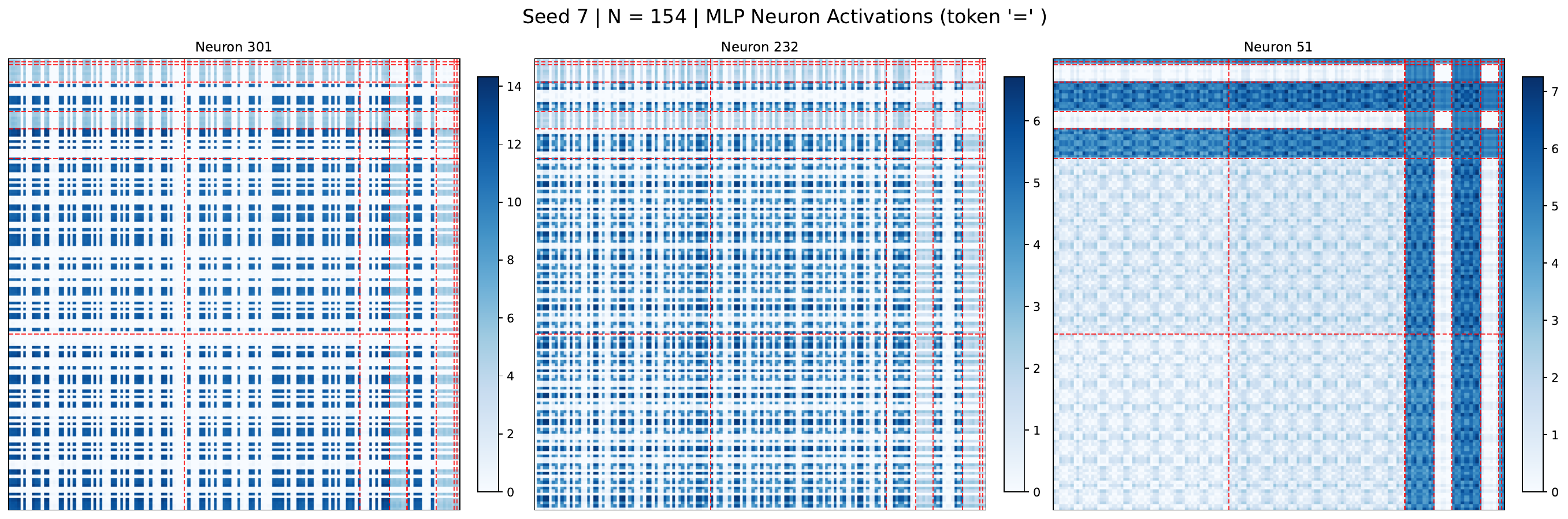}
    \caption{Seed 7}
\end{subfigure}
\hfill
\begin{subfigure}[H]{0.48\textwidth}
    \centering
    \includegraphics[width=\linewidth]{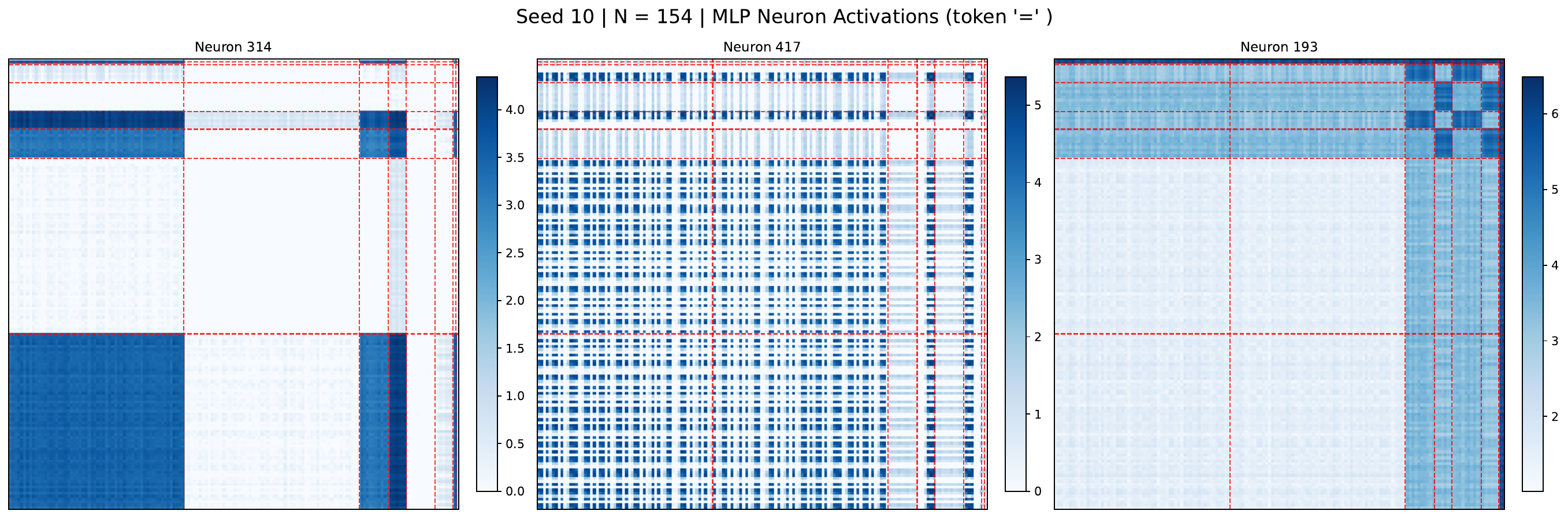}
    \caption{Seed 10}
\end{subfigure}

\vspace{0.1cm}

\caption{
\textbf{MLP feature stability across seeds and moduli.}
Each subplot shows activation maps of three randomly sampled MLP neurons, reordered according to the $\mathcal{J}$-class decomposition of $\mathbb{Z}_n$. Across seeds and moduli, we observe consistent emergence of block-structured activations aligned with multiplicative $\mathcal{J}$-classes, indicating that the learned representation is stable across random initializations.
}
\label{fig:mlp_stability_jclasses}

\end{figure*}

\newpage

\end{document}